\documentclass{article}

\PassOptionsToPackage{numbers, compress}{natbib}
\newif\ifpreprint
\preprintfalse
\ifpreprint
  \usepackage[preprint]{neurips_2026}
\else
  \usepackage[final,main]{neurips_2026}
\fi
\usepackage[utf8]{inputenc}
\usepackage[T1]{fontenc}
\usepackage{hyperref}
\usepackage{url}
\usepackage{arydshln}
\usepackage{booktabs}
\usepackage{amsfonts}
\usepackage{nicefrac}
\usepackage{microtype}
\usepackage[table]{xcolor}
\usepackage{colortbl}
\usepackage{graphicx}
\graphicspath{{figures/}}
\usepackage{subcaption}
\usepackage{multirow}
\usepackage{makecell}
\usepackage{tcolorbox}
\usepackage{pifont}
\usepackage{enumitem}
\usepackage{amsmath}
\usepackage{amssymb}
\usepackage{wrapfig}
\usepackage{amsthm}
\usepackage{mathtools}
\usepackage{float}
\usepackage{array,tabularx,ragged2e}

\newcolumntype{Y}{>{\RaggedRight\arraybackslash}X}
\newcolumntype{L}[1]{>{\hsize=#1\hsize\RaggedRight\arraybackslash}X}
\usepackage[capitalize,noabbrev]{cleveref}

\theoremstyle{plain}
\newtheorem{proposition}{Proposition}[section]
\newtheorem*{example*}{Example}

\newcommand{\ourmodel}{{\textsc{TabAlign}}}

\newcommand{\tabattn}{{\textsc{TabAttn}}}

\newcommand{\tick}{\ding{51}}
\newcommand{\cross}{\ding{55}}

\definecolor{Best}{RGB}{210,240,220}
\definecolor{Second}{RGB}{245,245,200}

\newcommand{\hi}[1]{\cellcolor{Best}#1}
\definecolor{Bad}{RGB}{255,220,220}

\newcommand{\badbox}[1]{\colorbox{Bad}{#1}}

\title{From Table to Cell: Attention for Better Reasoning with \ourmodel{}}

\author{\mdseries
Tung Sum Thomas Kwok$^{1,6*}$, 
Zeyong Zhang$^{2*}$, 
Xinyu Wang$^{3,6}$,\\
Chunhe Wang$^{1}$, 
Xiaofeng Lin$^{1}$, 
Hanwei Wu$^{4,6}$, \\
Lei Ding$^{5,6}$, Guang Cheng$^{1}$, Zhijiang Guo$^{7,8\dagger}$\\
$^{1}$University of California, Los Angeles, 
$^{2}$New Jersey Institute of Technology,\\
$^{3}$McGill University,
$^{4}$Université de Montréal,
$^{5}$University of Manitoba, 
$^{6}$SimpleWay, \\
$^{7}$HKUST (GZ), $^{8}$HKUST, \\
$^{*}$Major contributing authors, $^{\dagger}$Corresponding author \\
\texttt{\href{mailto:tk1018@ucla.edu}{tk1018@ucla.edu},
\href{mailto:zhijiangguo@hkust-gz.edu.cn}{zhijiangguo@hkust-gz.edu.cn}}
}

\begin{document}

\maketitle

\begin{abstract}
Multi-step LLM reasoning over structured tables fails because planning and execution share no explicit cell-grounding contract. Existing methods constrain the planner to a left-to-right factorization at odds with table permutation invariance, and score intermediate states by generated content alone, overlooking cell grounding. We conduct a pilot study showing that diffusion language models (DLMs) produce more human-aligned and permutation-stable cell attention on tables than autoregressive models, with a $40.2\%$ median reduction in attention-AUROC variability under row reordering. Motivated by this, we propose \ourmodel{}, a planned table reasoning framework that operationalizes the contract. \ourmodel{} pairs a masked DLM planner, whose bidirectional denoising emits plan steps as binary cell masks, with \tabattn{}, a lightweight verifier trained on 1{,}600 human-verified attention standards to score each step by its attention overlap with the plan-designated mask. Across eight benchmarks covering table question answering and fact verification, \ourmodel{} improves average accuracy by 15.76 percentage points over the strongest open-source baseline at comparable 8B-class scale, with a matched-backbone ablation attributing $2.87$ percentage points of this gain to the DLM planner over an AR planner on a fixed reasoner. Cleaner DLM plans also accelerate downstream reasoning execution by $44.64\%$.
Our code is publicly available at \url{https://github.com/ThomasK1018/TabAlign}.
\end{abstract}
\section{Introduction} \label{sec:intro}

Table reasoning has evolved from single-shot semantic parsing to \emph{interactive structured reasoning}, where agents select actions to transform table states and accumulate evidence to answer~\citep{pasupat2015compositional,chen2020tabfact,wu2025mmqa,xing2025mmtu}. Although tool-use agents~\citep{wang2024chainoftable,ji2025treeoftable} further extend this setting, they remain brittle on tables, where sparse evidence, permutation invariance of rows and columns, and complex schema dependencies cause early-step errors to compound across the trajectory. To mitigate this, \emph{planned reasoning} has emerged~\citep{nguyen2025interpretable,rawat2025preactmultistepplanningreasoning}, requiring agents to generate a natural-language plan before executing tool calls. Despite its appeal, this approach remains brittle because current pipelines lack \textbf{an explicit cell-grounding contract between planning and execution}. Because tabular evidence localizes to specific cells, this contract must operate at cell granularity. Without such a contract to supervise the planner's cell designations and verify the reasoner's cell-level attention, prior efforts often fall into a two-stage bottleneck where plausible reasoning steps target incorrect cells (\cref{fig:motivation}). 

\begin{figure}[t]
    \centering
    \includegraphics[width=0.9\linewidth]{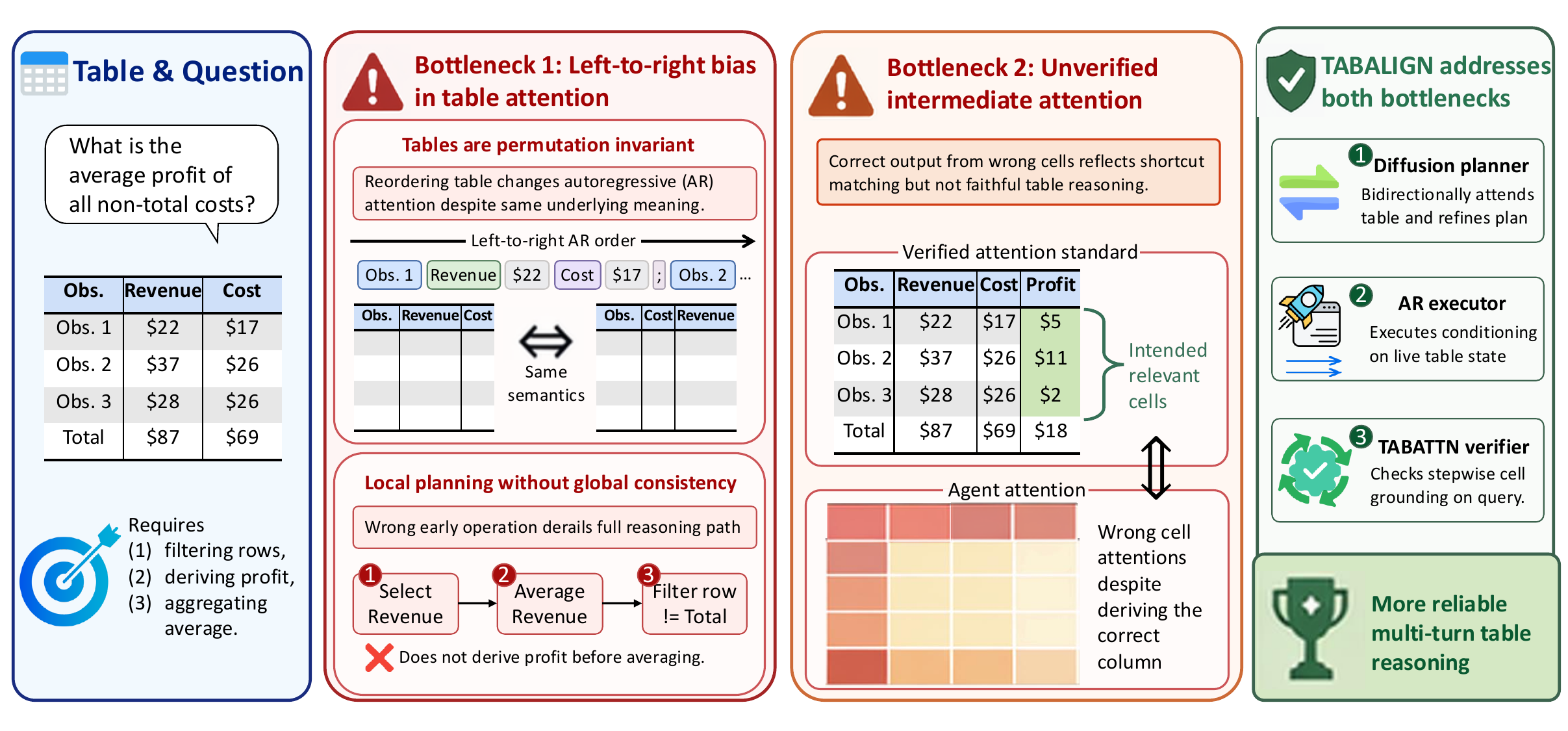}
    \vspace{-1mm}
    \caption{\textbf{Existing methods leave the cell-grounding contract implicit, producing failures at both planning and reasoning stages.} At planning time, AR models impose an order-sensitive prior that conflicts with table permutation invariance, producing a globally inconsistent plan. At reasoning time, content-based PRMs score the generated state alone, leaving cell grounding to chance. \ourmodel{} operationalizes the contract through a diffusion-based bidirectional planner and the \tabattn{} verifier.}
    \label{fig:motivation}
    \vspace{-5mm}
\end{figure}

This bottleneck first appears at planning time. Autoregressive (AR) generation produces plans under a left-to-right factorization, which carries an order-sensitive prior over which cells become relevant. Because tables are permutation-invariant, this prior couples plans to arbitrary serialization choices, so the same table under a different row order can yield a different plan. An early filtering step may only be valid given a later aggregation, but a causal planner commits to it before the relevant context is in scope. Our pilot study (\cref{sec:attn-pilot}) shows that diffusion-pretrained bidirectional attention produces more human-aligned and permutation-stable cell attention than state-of-the-art AR models.

The same bottleneck reappears at reasoning time, where existing process reward models \emph{do not check} whether each step consults plan-designated cells. Final-answer rewards~\citep{liu2024mixsc,nguyen2025interpretable} are too sparse to localize a step that read the wrong cell, and executor-based rewards~\citep{xing2025mmtu} require answer-checking environments unavailable for many tabular tasks. Existing process rewards either score the produced state's lexical or semantic quality~\citep{kwok2026retab,plrm2026}, or rely on LLM-as-a-judge step-correctness labels that face an upstream labelability ceiling~\citep{zou2026tattoo}, and neither family checks whether the model attends to the right cells. Our reward-component ablation (\cref{tab:planner_ablation}(b)) confirms that an attention-grounded signal closes this gap with consistent accuracy gains over content-based rewards.

Read together, the two failures trace to the same missing cell-grounding contract. A planned step commits to consult a specific set of cells, so the reasoner's cell-level attention forms an inspectable, labelable proxy for whether the step is grounded in the intended evidence rather than a complete causal explanation of the model's output. Our insight is that \emph{attention overlap with plan-designated cells yields a query-grounded, executor-free signal that exposes right-content-wrong-grounding failures invisible to content-based rewards and factorizes across planning and reasoning.}

Guided by this insight, we introduce \ourmodel{}, a planned table reasoning framework whose two components operationalize this contract. For \emph{understanding}, \ourmodel{} uses a masked diffusion language model (DLM) to denoise a globally consistent tool-use blueprint, with bidirectional denoising capturing mutual cell dependencies schema-wide. Each step terminates with a \texttt{[target: col]} tag that compiles into a binary cell mask $\mathbf{m}_t$, making the plan's cell commitment explicit. For \emph{reasoning}, we introduce \tabattn{}, a lightweight \emph{verifier} that scores each step by the reasoner's per-cell attention overlap with $\mathbf{m}_t$. Fit on 1{,}600 human-verified standards from eight benchmarks, \tabattn{} supplies in-loop feedback with a stagnation halt on non-improving attention.

Across eight benchmarks, \ourmodel{} outperforms the strongest open-source baseline at comparable 8B-class scale by 15.76\,pp on the six-benchmark accuracy average. A matched-backbone ablation (\cref{tab:backbone_ablation}) attributes 2.87\,pp of this gain to the DLM planner over an AR planner on a fixed Qwen3-VL-8B reasoner, with the remainder attributable to planning-vs-end-to-end execution and the attention-grounded reward. Cleaner DLM plans further accelerate downstream reasoning by 44.64\,\%. Our pilot study isolates two drivers behind these gains: bidirectional attention reduces attention-AUROC variability under row reordering by 40.2\,\%, and cell-relevance labels achieve 19.8\,pp higher LLM-vs-human agreement than step-correctness labels. Component-level analyses (\cref{sec:in-depth-analyses}) attribute the framework gain to planner choice, reasoner backbone, and the attention-grounded reward, ruling out backbone scaling. Our contributions are:

\clubpenalty=0
\widowpenalty=0
\displaywidowpenalty=0
(1) \textbf{Cell-level attention pilot study.} We conduct a pilot study showing that bidirectional attention is more human-aligned and permutation-stable than AR attention, and that cell relevance is more labelable than step correctness across eight benchmarks (\cref{sec:attn-pilot}).

(2) \textbf{Planned reasoning with attention grounding.} We propose \ourmodel{}, featuring a bidirectional DLM planner and an attention-overlap verifier (\tabattn{}). Our framework captures grounding failures missed by content-based rewards while utilizing optimized DLM plans to reduce column hallucinations and accelerate downstream reasoning (\cref{sec:methodology}).

(3) \textbf{Consistent gains across benchmarks.} \ourmodel{} consistently outperforms existing open-source baselines at comparable 8B-class scale
across eight table reasoning benchmarks and ablates cleanly across both planner and reasoner backbones (\cref{sec:experiments}).
 \vspace{-1mm}
\section{Related Work}
\vspace{-1mm}
\label{sec:related}
\textbf{Table reasoning agents and process supervision.}
Table reasoning has evolved from single-shot semantic parsing~\citep{pasupat2015compositional,herzig2020tapas,yin2020tabert} to interactive agent reasoning across fact verification~\citep{chen2020tabfact,gupta2020infotabs}, multi-hop QA~\citep{zhu2021tatqa,wu2025mmqa}, and unified evaluation benchmarks~\citep{cheng2022hitab,xing2025mmtu}. Prior work~\citep{wang2024chainoftable,ji2025treeoftable,liu2024mixsc,yang2025tide,yang2025cit} formulates this as a partially observable Markov Decision Process~\citep{ASTROM1965174} that transforms table states through tool actions. Tool-using LLM agents~\citep{yao2023react,wang2024chainoftable} predefine a closed action set but make local decisions without verifying intermediate states. PRMs~\citep{lightman2023verify,uesato2022processstep,wang2024math,lu2024autopsv,zhu2026codescalerscalingcodellm} have been introduced into table reasoning through lexical rewards~\citep{kwok2026retab}, learned tool-grounded supervision~\citep{zou2026tattoo}, and step-level execution rewards for agentic analysis~\citep{plrm2026}, while multimodal studies~\citep{kwok2026tabqaworld,xing2026tabledart} adopt table representation selection as a reasoning component. These approaches supervise generated content alone. We complement them by adding input attention as a process signal for structured table reasoning.

\textbf{Tabular data with diffusion models.}
Early diffusion work on tabular data focuses on data synthesis~\citep{10.5555/3618408.3619133,lin2025ctsyn}. TabNAT~\citep{tabnat2025} extends this with a continuous-discrete joint generative framework, noting that AR generation order introduces structural misalignment for heterogeneous tabular columns. We share this observation but shift the focus to table reasoning and process supervision. With the rise of discrete diffusion~\citep{lou2024discrete}, masked DLMs have evolved from early controllable generators~\citep{li2022diffusionlm,austin2021structured} to scaled models competitive with AR systems on structured generation~\citep{nie2025llada,zheng2025dream,sahoo2026mdlm}, enabling planned reasoning~\citep{nie2025planner} that pairs diffusion planners with AR executors. Existing diffusion-based reasoning has primarily focused on text and vision, leaving table reasoning underexplored. We address this by using masked bidirectional attention to better match the permutation invariance of tables, producing more stable and human-aligned cell attention~\citep{liu2024mixsc} (\cref{sec:app-perm-stability}).

 \vspace{-1mm}
\section{Pilot Analysis: Revisiting Attention Mechanism on Tables}
\label{sec:attn-pilot}
\vspace{-1mm}
We conduct a pilot study to identify what makes table reasoning hard to supervise, guided by the following research question:
\begin{tcolorbox}[width=0.99\textwidth, colback=blue!5!white, colframe=blue!75!black, boxsep=1pt, top=3pt, bottom=3pt, left=4pt, right=4pt]
    \resizebox{\linewidth}{!}{\textbf{RQ}: Which attention mechanism best aligns cell-level information accessibility with table semantics?
}
\end{tcolorbox}

\begin{wraptable}{r}{0.52\textwidth}
\centering
\small
\vspace{-3mm}
\caption{Mean AUROC and permutation stability $\sigma_{\text{AUROC}}$ ($K{=}5$ row permutations). Per-dataset breakdown in \cref{tab:auroc-combined-full}.}
\label{tab:auroc-combined}
\vspace{-1mm}
\setlength{\tabcolsep}{4pt}
\resizebox{\linewidth}{!}{\begin{tabular}{lcccc}
\toprule
Model & Bidir. & Diff. & Overall $\uparrow$ & $\sigma_{\text{AUROC}}\downarrow$ \\
\midrule
ModernBERT (149M) & \tick & \cross & 0.569 & \textbf{0.046} \\
Qwen3-VL-8B  & \cross & \cross & 0.628 & 0.123 \\
Qwen3-8B     & \cross & \cross & 0.623 & 0.138 \\
DiffuLLaMA-7B & \tick & \tick & 0.639 & 0.123 \\
LLaDA-8B     & \tick & \tick & \textbf{0.677} & 0.074 \\
\bottomrule
\end{tabular}}
\vspace{-5mm}
\end{wraptable}
 \textbf{Table attention for understanding.} To investigate the research question, we evaluate models with different attention mechanisms, including AR Qwen3-8B~\citep{yang2025qwen3technicalreport} and its multimodal counterpart Qwen3-VL-8B~\citep{bai2025qwen3vltechnicalreport}, and bidirectional attention models. We begin with ModernBERT-149M~\citep{warner-etal-2025-smarter}, trained via random masking. We then extend to diffusion-based models with scheduled masking, DiffuLLaMA-7B~\citep{gong2025diffugpt}, which adapts AR models through diffusion post-training, and LLaDA-8B~\citep{nie2025llada}, a fully diffusion-pretrained model. All models operate on serialized table text tokens for modality control. DiffuLLaMA specifically serves as a within-text causal control to isolate the effect of bidirectional \emph{pre-training} from multimodal architecture differences. We curate 1{,}600 human-verified table attention standards from eight datasets (details in \cref{sec:app-attention-curation}) Each map contains the table with each cell assigned a binary label to indicate its relevance to the question. 

\Cref{tab:auroc-combined} shows that bidirectional attention improves cell-level AUROC against human annotations by 4.9\,pp (0.628 $\to$ 0.677), with DiffuLLaMA and LLaDA both outperforming Qwen3-VL-8B despite older backbones (e.g., LLaMA-2-7B~\citep{touvron2023llama2openfoundation}). The text-only Qwen3-8B and multimodal Qwen3-VL-8B score nearly identically, and \cref{tab:llm-revis-controls} confirms that image-pad tokens receive zero attention mass in Qwen3-VL-8B, so the AR-DLM gap reflects attention mechanism rather than multimodal pretraining. While ModernBERT's smaller scale prevents a clean attribution of its AUROC, its permutation stability ($\sigma_{\text{AUROC}}=0.046$) exceeds even LLaDA-8B, evidence that bidirectional pre-training improves stability independently of model capacity. DiffuLLaMA's failure to inherit ModernBERT/LLaDA's stability further indicates that bidirectional \emph{pre-training} is the load-bearing property, motivating our use of a masked DLM as the planner.

\begin{figure}[t!]
\centering
\includegraphics[width=0.9\linewidth]{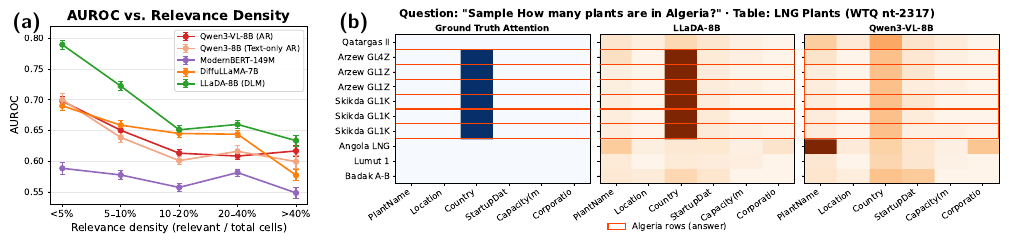}
\caption{\textbf{Pilot analysis: bidirectional attention is more human-aligned than AR attention, with the largest gap on tables where relevant cells are sparse.} \textbf{(a)} LLaDA's alignment advantage is largest with sparse relevant cells (density $<$10\%) and narrows above 40\% density, leading across all bins. \textbf{(b)} Question-cell attention for ``How many plants are in Algeria?'': LLaDA concentrates on relevant cells while Qwen3-VL disperses across irrelevant rows.}
\label{fig:pilot}
\vspace{-4mm}
\end{figure}

\Cref{fig:pilot}(a) shows that LLaDA's advantage is largest when relevant cells are sparse (density $<$10\%) and narrows above 40\%, while leading across all bins. This supports the intuition that bidirectional attention helps most when reasoning depends on a few decisive cells, the dominant table reasoning regime. \Cref{fig:pilot}(b) shows a WTQ example where LLaDA aligns with ground-truth cells while Qwen3-VL disperses.
\begin{tcolorbox}[width=0.98\textwidth, colback=yellow!15!white, colframe=yellow!55!black, boxsep=1pt, top=1pt, bottom=1pt, left=4pt, right=4pt]
\textbf{Takeaway 1} (Attention Mechanism on Tables): Existing AR LLMs produce structurally misaligned cell attention on tables, whereas diffusion-pretrained bidirectional attention consistently improves alignment with human annotations and is more stable under row reordering.

\end{tcolorbox}
\textbf{From attention alignment to attention-grounded supervision.} Building on Takeaway 1, we ask whether this attention can directly serve as process-level supervision, which hinges on labelability of its natural ground truth, cell-relevance. The two dominant families, \emph{content-similarity rewards}~\citep{reimers2019sbert,zhu2021tatqa,kwok2026retab} and \emph{step-correctness PRMs}~\citep{zou2026tattoo}, both score outputs of reasoning. We pivot to the input side, asking which cells the question depends on, a per-cell binary view supported by recent evidence that attention statistics predict downstream correctness~\citep{li2025attentionilluminatesllmreasoning}. To compare under matched conditions, we follow~\citet{krishna-etal-2023-longeval} and~\citet{bologna2026cqaevaldesigningreliableevaluations} in evaluating each side at its natural decision unit (one binary per step, one per cell) and run an identical LLM-as-a-judge pipeline on both, with full setup details in \cref{sec:app-labelability}.

\begin{table}[t]
\centering
\small
\setlength{\tabcolsep}{3pt}
\renewcommand{\arraystretch}{1.08}
\vspace{-4mm}
\caption{Representative failures of existing reward signals. Content-based rewards can miss numeric, semantic, or structural errors in table reasoning.}
\label{tab:metric-failures}
\begin{tabularx}{\columnwidth}{p{2.4cm} X}
\toprule
\textbf{Metric} & \textbf{Failure case} \\
\midrule

\textbf{Cosine Sim}~\citep{reimers2019sbert} &
\textbf{Q:} \textit{``How long did Cap Anson play?''}
\quad \textbf{GT:} 26 years
\quad \textbf{Ans:} \badbox{126} years
\quad \textbf{Score:} \badbox{0.749}
\quad \textbf{Issue:} Same semantic frame (``$N$ years'') hides the numeric mismatch~\citep{10.1145/3589335.3651526}. \\

\midrule
\textbf{Numerical F1}~\citep{zhu2021tatqa} &
\textbf{Q:} \textit{``Which player scored the most?''}
\quad \textbf{GT:} Lionel Messi, 25
\quad \textbf{Ans:} \badbox{Karim Benzema}, 25
\quad \textbf{Score:} \badbox{1.000}
\quad \textbf{Issue:} Ignores units and entity semantics, giving full credit despite the qualitative mismatch. \\

\midrule
\textbf{TABROUGE}~\citep{kwok2026retab} &
\textbf{Q:} \textit{``What is the element before Cu?''}
\quad \textbf{GT:} Fe
\quad \textbf{Step:} \texttt{sort(`atom\_num')}
\quad \textbf{Score:} $0.0167 \rightarrow \badbox{0.0167}$
\quad \textbf{Issue:} LCS rewards query-token presence rather than row order, leaving sorting steps unrewarded. \\

\bottomrule
\end{tabularx}
\vspace{-3mm}
\end{table} \Cref{tab:metric-failures} shows that content-similarity rewards miss content-preserving table-state changes where correct cell selection is indistinguishable from spurious overlap. \citet{zou2026tattoo} report the same pattern in their error analysis, with table-retrieval and schema-interaction \begin{wraptable}{r}{0.4\textwidth}
\centering
\small
\vspace{-3mm}
\caption{Pooled labelability under identical LLM-as-a-judge pipelines. Per-dataset breakdown in \cref{tab:labelability-full}.}
\label{tab:labelability}
\setlength{\tabcolsep}{4pt}
\begin{tabular}{lcc}
\toprule
Method & Agree (\%) & $\kappa$ \\
\midrule
Step-correctness & 73.2 & 0.48 \\
Cell-relevance   & \textbf{93.0} & \textbf{0.68} \\
$\Delta$ & +19.8 & +0.20 \\
\bottomrule
\end{tabular}
\vspace{-5mm}
\end{wraptable}
 errors as the dominant failure categories.  TaTToo's authors further prepend the correct sub-table to schema-interaction steps and manually correct retrievals to make their LLM-judge labels tractable, a context-engineering workaround that itself signals step-correctness as a hard target.

\Cref{tab:labelability} quantifies this gap directly. Cell-relevance reaches 93.0\% per-decision agreement and Cohen's $\kappa = 0.68$ pooled across 76{,}157 cell decisions, while step-correctness hits a structural ceiling. 46.9\% of human step labels are marked ``unsure,'' and on the human-confident subset LLM-vs-human agreement is only 73.2\% with $\kappa = 0.48$ (n=190). A step-correctness classifier trained on 1{,}548 LLM-judge step labels reaches held-out F1=0.77 / AUROC=0.88, capped by the noisy supervision signal it inherits. The pooled advantage holds on six of eight benchmarks. On InfoTabs and TabMWP, where step labels reduce to short-form classification or template-aligned arithmetic, step-correctness is easier to label than cell-relevance ($-10.6$ and $-5.9$\,pp respectively, \cref{tab:labelability-full}), so the claim is that cell-relevance is more labelable on average and on most regimes rather than uniformly. Cell decisions are local and extractive, while step decisions are multi-conditional and evaluative. Prior faithfulness annotation works~\citep{krishna-etal-2023-longeval,bologna2026cqaevaldesigningreliableevaluations} identify the same asymmetry as the source of finer-unit labeling reliability.
\begin{tcolorbox}[width=0.98\textwidth, colback=yellow!15!white, colframe=yellow!55!black, boxsep=1pt, top=3pt, bottom=3pt, left=4pt, right=4pt]
\textbf{Takeaway 2} (Labelability of Supervision Targets): Step-correctness hits a labelability ceiling under LLM-as-a-judge whereas cell-relevance admits reliable per-decision labeling, making cell-grounded attention a labelable proxy for step quality on most table reasoning regimes.
\end{tcolorbox}

 \vspace{-1mm}
\section{Methodology} \label{sec:methodology} \vspace{-1mm}

\Cref{sec:attn-pilot} motivates \ourmodel{} with two findings: bidirectional table attention is more human-aligned and permutation-stable than AR attention, and cell relevance is more labelable than step correctness. \ourmodel{} operationalizes this cell-grounding contract with two components: a bidirectional denoising DLM planner that assigns cell-level targets (\cref{sec:plan-design}), and a \tabattn{}-based execution module that verifies the reasoner's attention against those targets using a reward fit on cell-relevance labels (\cref{sec:stepwise-exec}).

\vspace{-1mm}
\subsection{Reasoning Plan Design} \label{sec:plan-design} \vspace{-1mm}

\textbf{Table understanding through bidirectional masking.}
To align with \cref{sec:attn-pilot}'s conclusion, we adopt bidirectional attention through a masked DLM. Given an input sequence $x_0$ of table tokens (headers, cell values) and the question, the forward process masks each token independently and symmetrically~\citep{feng2026theoretical}. This treats all positions identically and reduces order-sensitivity:
\begin{equation*}
q_{t|0}(x_t \mid x_0) = \prod_{i=1}^{L} q_{t|0}(x_t^i \mid x_0^i), \quad
q_{t|0}(x_t^i \mid x_0^i) =
\begin{cases}
\alpha_t & x_t^i = x_0^i \\
1 - \alpha_t & x_t^i = [\mathtt{m}]
\end{cases}
\end{equation*}
\textbf{Plan generation via joint denoising.}
The model denoises a fully masked plan canvas into a natural-language step sequence, with each token attending bidirectionally over the full canvas~\citep{feng2026theoretical}. Joint denoising lets requirements in later steps constrain earlier ones, enforcing inter-step dependencies that AR generation cannot. Canvas length, block schedule, and padding conventions follow~\citep{liu2025longllada,he2026ultrallada}, with full details in \cref{sec:app-planner}. \vspace{-1mm}
\begin{equation*}
p_\theta(x_0 \mid x_t) = \prod_{i=1}^{L} p_\theta(x_0^i \mid x_t).
\end{equation*}
\textbf{Tool-integrated planning.}
The planning prompt presents the DLM with a semi-closed tool vocabulary alongside the table metadata, requiring each step to conclude with a \texttt{[target: col]} tag designating which cells it operates on. These tags are parsed after generation into the binary cell mask $\mathbf{m}_t \in \{0,1\}^C$, later processed to evaluate stepwise attention alignments by \tabattn{} (\cref{sec:stepwise-exec}). The full tool vocabulary and prompt format are in \cref{sec:app-planner}.
\begin{figure*}[t]
    \centering
    \includegraphics[width=0.9\linewidth]{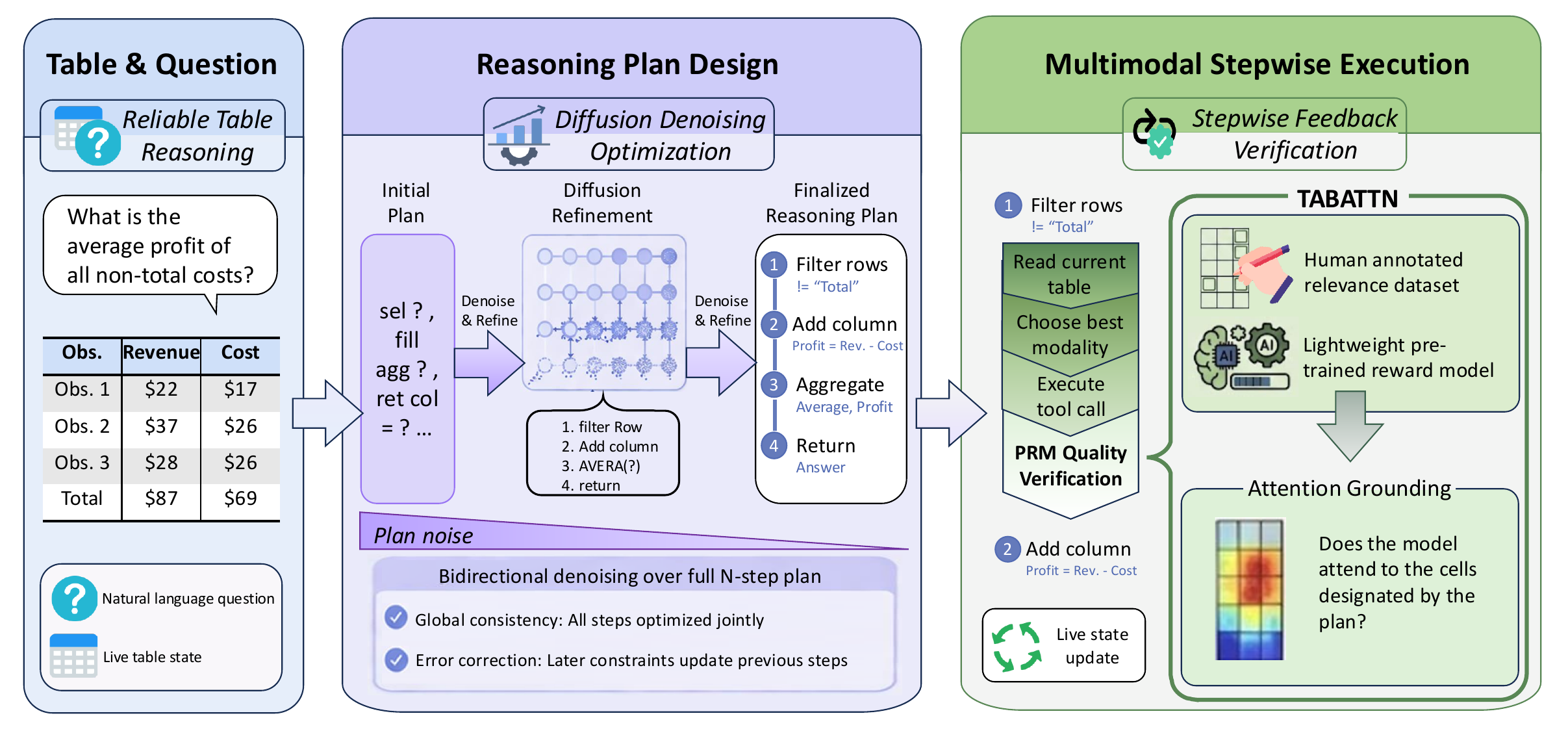}
    \caption{\textbf{\ourmodel{} pairs a masked DLM planner with the \tabattn{} verifier to enforce the cell-grounding contract.} The DLM denoises a tool-use plan over compact metadata into a step-by-step blueprint, with each step terminating in a \texttt{[target: col]} tag that compiles into a cell mask $\mathbf{m}_t$. \tabattn{} then scores each reasoning step by the overlap between the reasoner's per-cell attention and $\mathbf{m}_t$, supplying in-loop process feedback.}
    \label{fig:workflow}
    \vspace{-3mm}
\end{figure*}

\vspace{-1mm}
\subsection{Multimodal Stepwise Execution} \label{sec:stepwise-exec} \vspace{-1mm}

\textbf{Human-verified attention standard curation.}
The 1{,}600 human-verified table attention standards from \cref{sec:attn-pilot} double as the training set for the stepwise attention reward. We seed with 100 manual annotations of which cells should be attended to per question, bootstrap with strong LLMs, and verify all generated examples by human annotators. This human-in-the-loop generation pipeline is feasible because cell relevance is a labelable target (\cref{tab:labelability}), reaching 93.0\% LLM-vs-human per-cell agreement under our identical pipeline protocol, compared to step-correctness pipeline which inherits the upstream labelability ceiling measured in \cref{sec:attn-pilot}. Full annotation protocols are in \cref{sec:app-attention-curation}.

\textbf{\tabattn{}: lightweight learned attention verifier.}
\tabattn{} monitors whether each reasoning step attends to the cells designated by the DLM plan. For step $t$, the planner emits a binary cell mask $\mathbf{m}_t \in \{0,1\}^C$, while the multimodal reasoner provides token-level attention $\alpha_t^{(\tau)}$ from the next-token-emit position over the serialized table. We aggregate to per-cell scores by summing attention mass over the tokens of each cell, $a_t^{(c)} = \sum_{\tau \in \mathcal{T}(c)} \alpha_t^{(\tau)}$, where $\mathcal{T}(c)$ is the contiguous token span of cell $c$ in the serialization (header tokens are pooled into their column, value tokens into the cell at the corresponding row and column), and we then compute the overlap:
\begin{equation}\label{eq:rattn}
R_{\text{attn}}(a_t; \mathbf{m}_t) =
\frac{\sum_{c=1}^{C} a_t^{(c)}\, \mathbf{m}_t^{(c)}}
{\max\!\left(\sum_{c=1}^{C} a_t^{(c)},\; \epsilon\right)},
\end{equation}
with $\epsilon=10^{-6}$ for zero table-attention mass. A density-preserving falsification (\cref{sec:app-mask-falsification}) confirms that $R_{\text{attn}}$ tracks the correct cell set rather than attention sharpness ($1.5$ to $3.1\times$ over cell-shuffled nulls). Such uninformative steps are excluded from training. If a target tag cannot be parsed ($3.6$\,pp DLM-vs-AR parse-error increase, \cref{sec:app-exec-errors}), $\mathbf{m}_t$ defaults to a uniform mask. Finally, \tabattn{} applies a two-parameter logistic calibration to $R_{\text{attn}}$, fitted by binary cross-entropy against terminal trajectory correctness on the 1{,}600 standards with an 80/20 early-stopping split, preserving an interpretable calibrated cell-overlap reward. We use text-side cell attention because Qwen3-VL-8B assigns negligible attention to image-pad tokens (\cref{sec:app-llm-revis}).

\textbf{Grounded AR execution.}
At each step $t$, the multimodal AR reasoner receives the current table state $s_t$ and plan step $p_t$, selects an observation modality, and issues a tool call to produce $s_{t+1}$. Following prior per-step modality routing~\citep{xing2026tabledart,kwok2026tabqaworld}, it chooses between text serialization (cheap, exact strings) and rendered images (layout and merged-cell structure): lookups and string comparisons use text by default, while layout-sensitive operations use images, with escalation to images after text-branch failure. \tabattn{} scores each transition via \cref{eq:rattn} and returns the score as in-loop feedback. To cap unproductive iteration, a stagnation halt emits the highest-scoring answer when $R_{\text{attn}}$ improves by less than $0.02$ for two consecutive unchanged states. We follow \citet{xing2026tabledart} and report Pass@1. Full pseudocode and parse-error/stagnation accounting are in \cref{sec:app-exec-errors}.

 \vspace{-1mm}
\section{Empirical Evaluations} \label{sec:experiments} \vspace{-1mm}

All experiments follow the same evaluation protocol and baseline set adopted by recently published table reasoning works~\citep{xing2026tabledart,zou2026tattoo}.

\textbf{Baselines.}
We compare \ourmodel{} against five categories of prior work: table-as-text LLMs~\citep{touvron2023llama2openfoundation,zhang2024tablellm,yang2025tableggpttablegptt}, table-as-image VLMs~\citep{zheng2024tablellava,zhou2025syntab,yao2024minicpmv,bai2025qwen25vl,bai2025qwen3vltechnicalreport}, multimodal frameworks~\citep{xing2026tabledart}, trajectory-optimization agents~\citep{zou2026tattoo,guo2026tabtrim}, and training-free GPT-3.5 agents~\citep{ye2023dater,zhang2024reactable,liu2024mixsc,yang2025tide,yang2025cit}. We additionally evaluate the contemporary Qwen3-VL-8B~\citep{bai2025qwen3vltechnicalreport} as both stronger baseline and reasoner backbone of \ourmodel{}. Proprietary models appear for context only, as \tabattn{} requires open attention weights, so the direct comparisons are 8B-class systems (TableDART, TATTOO). Full baseline details are in \cref{sec:app-exp-details}.

\textbf{Datasets.}
We evaluate on eight benchmarks covering two task types: TQA on WTQ~\citep{pasupat2015compositional}, MMQA~\citep{wu2025mmqa}, TabMWP~\citep{lu2023tabmwp}, TAT-QA~\citep{zhu2021tatqa}, HiTab~\citep{cheng2022hitab}, FeTaQA~\citep{nan2022fetaqa}, and TFV on TabFact~\citep{chen2020tabfact} and InfoTabs~\citep{gupta2020infotabs}. We use accuracy for all benchmarks except FeTaQA, which is scored by BLEU~\citep{10.3115/1073083.1073135} after the LLM-rewriting post-processing of~\citet{xing2026tabledart}. The headline \emph{Avg} column in \cref{tab:main-results} averages the six accuracy benchmarks, with MMQA and FeTaQA excluded since not all baselines report them. Per-dataset statistics are in \cref{sec:app-exp-details}.

\textbf{Implementation.}
\ourmodel{} uses LLaDA-8B-Instruct~\citep{nie2025llada} as planner and Qwen3-VL-8B~\citep{bai2025qwen3vltechnicalreport} as reasoner. The 1{,}600 \tabattn{} attention standards come from train/dev splits and do not overlap with the test sets in \cref{tab:main-results}. We report all \ourmodel{} accuracies in \cref{tab:main-results,tab:planner_ablation,tab:backbone_ablation} with a fixed decoding seed: greedy denoising for the planner ($T{=}0$) and temperature 0.7 for the reasoner. Pilot AUROC comparisons (\cref{tab:auroc-combined}) include paired Wilcoxon $p$-values over $n{=}319$ matched records (\cref{sec:app-perm-stability}). Each accuracy entry in \cref{tab:main-results,tab:planner_ablation} is computed over $n{=}200$ evaluation questions per dataset. Per-cell volatility analysis is in \cref{tab:appendix-volatility}. Hyperparameters, ablation backbones, and computational environment are in \cref{sec:app-exp-details}.

\begin{table}[t!]
\centering
\small
\setlength{\tabcolsep}{4pt}
\caption{Results on eight table reasoning benchmarks. All numbers are accuracy except FeTaQA (BLEU). \emph{Avg} measures accuracy mean of WTQ, TabMWP, TAT-QA, HiTab, TabFact, InfoTabs. Baselines that do not report all six are marked ``---''. \textbf{Bold}/\underline{Underline}: best/second per column. \ourmodel{} uses LLaDA-8B-Instruct as the planner and Qwen3-VL-8B as the reasoner. 
}
\label{tab:main-results}
\resizebox{0.95\textwidth}{!}{\begin{tabular}{lccccccccc}
\toprule
 & \multicolumn{6}{c}{\textbf{TQA}} & \multicolumn{2}{c}{\textbf{TFV}} & \\
\cmidrule(lr){2-7} \cmidrule(lr){8-9}
\textbf{Method} & WTQ & MMQA & TabMWP & TAT-QA & HiTab & FeTaQA & TabFact & InfoTabs & \textbf{Avg} \\
\midrule
\rowcolor{gray!11} \multicolumn{10}{l}{\textit{Table-as-Text Baselines}} \\
Llama-2-7B & 16.39 & --- & 22.82 & 13.73 & 10.72 & 10.93 & 9.20 & 38.92 & 18.63 \\
Llama3-Instruct-8B & 21.24 & --- & 42.01 & 13.08 & 6.97 & 12.66 & 73.89 & 54.00 & 35.20 \\
TableLlama-7B & 24.97 & --- & 10.10 & 19.04 & 46.57 & \textbf{38.38} & 79.37 & 46.57 & 37.77 \\
\midrule
\rowcolor{gray!11} \multicolumn{10}{l}{\textit{Table-as-Image Baselines}} \\
Table-LLaVA-7B & 18.43 & --- & 57.78 & 12.82 & 10.09 & 25.60 & 59.85 & 65.26 & 37.37 \\
SynTab-LLaVA-7B & 39.59 & --- & \underline{88.30} & 51.94 & 35.66 & 35.45 & 70.78 & 69.42 & 59.28 \\
MiniCPM-V-2.6-8B & 47.97 & --- & 83.68 & 51.55 & 56.53 & 32.68 & 78.48 & 73.03 & 65.21 \\
Qwen2.5-VL-7B & 54.37 & --- & 63.69 & 51.94 & 62.69 & 10.99 & 75.81 & 70.13 & 63.11 \\
Qwen3-VL-8B & \underline{81.82} & --- & 79.90 & 59.80 & 65.83 & 13.89 & 82.41 & 75.88 & 74.27 \\
\midrule
\rowcolor{gray!11} \multicolumn{10}{l}{\textit{Joint Multimodal Baselines}} \\
HIPPO-8B & 55.77 & --- & 87.50 & 60.75 & 63.00 & 33.18 & 82.27 & 75.74 & 70.84 \\
Gemini 2.0 Flash & 63.56 & --- & 46.29 & 35.62 & 60.41 & 10.57 & 81.33 & 54.31 & 56.92 \\
\midrule
\rowcolor{gray!11} \multicolumn{10}{l}{\textit{Training-free GPT-3.5 Agents}} \\
DATER & 65.9 & --- & --- & --- & --- & 30.92 & 85.60 & --- & --- \\
ReAcTable & 68.0 & --- & --- & --- & --- & 30.43 & 86.10 & --- & --- \\
Mix-SC & 73.7 & --- & --- & --- & --- & --- & 88.50 & --- & --- \\
TIDE & 75.0 & --- & --- & --- & --- & --- & 89.82 & --- & --- \\
CIT-DP & 76.4 & --- & --- & --- & --- & \underline{36.34} & \textbf{91.30} & --- & --- \\
\midrule
\rowcolor{gray!11} \multicolumn{10}{l}{\textit{Dynamic Modality Routing}} \\
TableDART (on Q2.5VL) & 69.29 & --- & 72.61 & 59.07 & 71.13 & 29.87 & 77.94 & 71.46 & 70.25 \\
TableDART (on Ovis2)  & 70.58 & --- & 84.54 & \underline{62.05} & \underline{74.37} & 36.11 & 81.37 & \underline{76.22} & \underline{74.86} \\
\midrule
\rowcolor{gray!11} \multicolumn{10}{l}{\textit{Trajectory-Optimization Agents}} \\
TATTOO       & 69.80 & \underline{25.10} & --- & --- & --- & --- & 58.79 & --- & --- \\
TabTrim-4B   & 76.80 & --- & --- & --- & --- & --- & 89.40 & --- & --- \\
TabTrim-8B   & 79.40 & --- & --- & --- & --- & --- & \underline{91.20} & --- & --- \\
\midrule
\rowcolor{blue!8} \ourmodel{} (Ours) & \textbf{94.12} & \textbf{71.07} & \textbf{95.13} & \textbf{91.41} & \textbf{82.48} & 32.63 & 90.09 & \textbf{90.45} & \textbf{90.62} \\
\bottomrule
\end{tabular}}
\vspace{-6mm}
\end{table}
 {
\clubpenalty=0        \widowpenalty=0       \brokenpenalty=0
\interlinepenalty=0
\subsection{Main Results} 
\label{sec:main-results} \vspace{-1mm}
\textbf{Headline Accuracy across Eight Benchmarks.} \ourmodel{} reaches 90.62\% on the six-benchmark accuracy average, which is 15.76\,pp above the strongest dynamic-modality baseline TableDART (on Ovis2, 74.86\%) and 16.35\,pp above the strongest single-modality baseline Qwen3-VL-8B (74.27\%) at comparable 8B-class scale (\cref{tab:main-results}). At the per-benchmark level, \ourmodel{} ranks first on six of the eight columns (WTQ, MMQA, TabMWP, TAT-QA, HiTab, InfoTabs) and remains within 1.21\,pp of the best on TabFact. FeTaQA is the only column where the table-specialized TableLlama-7B retains a meaningful lead under BLEU. Consistent with recent work framing modality as the key bottleneck~\citep{xu2026visual}, backbone ablations (\cref{tab:backbone_ablation}) confirm this is not pure backbone scaling.} 

\textbf{Generalization across Component-specific Baselines.} 
\ourmodel{} also outperforms baselines targeting specific parts of the table-reasoning pipeline. Against dynamic modality routing, \ourmodel{} beats TableDART (Ovis2) on six of seven accuracy benchmarks and stays within 3.48\,BLEU on FeTaQA. Against trajectory-optimization agents with auxiliary verifiers or pruners over a frozen reasoner, \ourmodel{} exceeds TabTrim-8B by 14.72\,pp on WTQ and stays within 1.11\,pp on TabFact, while outperforming TATTOO by 24.32\,pp on WTQ and 31.30\,pp on TabFact. These gains suggest that grounding cell-level attention recovers more headroom than auxiliary trajectory supervision under fixed-attention reasoners.
\vspace{-1mm}
\subsection{In-Depth Analyses on \ourmodel{}} 
\label{sec:in-depth-analyses} \vspace{-1mm}
This section reports five analyses: planner $\times$ reasoner backbone (\cref{tab:backbone_ablation}), reasoning configuration with DLM error taxonomy (\cref{tab:planner_ablation}(a), \cref{tab:vdlm_failures}), reward signal (\cref{tab:planner_ablation}(b)), inference latency (\cref{tab:planner_ablation}(a)), and plan quality (\cref{tab:schema_adherence}). Permutation stability is established earlier in \cref{sec:attn-pilot}.

\textbf{Backbone Ablation: Disentangling Planner and Reasoner Gains.} To rule out backbone scaling as the driver of \ourmodel{}'s gain, we ablate planner choice (AR vs DLM) crossed with reasoner backbone (Q2.5-VL vs Q3-VL) on the six-benchmark average (\cref{tab:backbone_ablation}). Holding the reasoner fixed, the DLM planner contributes +2.07 to +2.87\,pp across both backbones, with the strongest configuration (DLM + Q3-VL) reaching 90.62\%. This consistency across reasoners rules out backbone scaling alone as the source of \ourmodel{}'s gain over the AR baseline.

\begin{wraptable}{r}{0.49\textwidth}
\vspace{0mm}
\centering
\small
\caption{Backbone ablation: planner (AR Qwen3-VL-8B (Q3-VL) versus DLM (LLaDA-8B) $\times$ reasoner (Qwen2.5-VL-7B vs Q3-VL) combinations on the six-benchmark accuracy average. }
\label{tab:backbone_ablation}
\setlength{\tabcolsep}{4pt}
\renewcommand{\arraystretch}{1.05}
\resizebox{\linewidth}{!}{\begin{tabular}{lccc}
\toprule
\textbf{Planner\,/\,Reasoner} & \textbf{Q2.5-VL} & \textbf{Q3-VL} & $\Delta$ \\
\midrule
AR (Q3-VL)       & 78.22 & 87.75          & +9.53 \\
DLM (LLaDA-8B)   & 80.29 & \textbf{90.62} & +10.33 \\
\midrule
$\Delta$         & +2.07 & +2.87          & --- \\
\bottomrule
\end{tabular}}
\vspace{-4mm}
\end{wraptable} \textbf{Reasoning Ablation and Error-Mode Taxonomy.}
\Cref{tab:planner_ablation}(a) compares end-to-end DLM/LLM reasoning with their planned variants under a shared AR reasoner. End-to-end DLM reasoning collapses, mainly due to duplicate actions (56\%) and format errors (51\%) (\cref{tab:vdlm_failures}). Using the DLM only for planning improves accuracy by 4.69~pp over AR planning, suggesting current DLMs are stronger at table-structure understanding than tool execution. This likely reflects instruction-tuning gaps in existing DLMs rather than diffusion-specific limits~\citep{yang2026darediffusionlargelanguage,bie2025llada20scalingdiffusionlanguage}. The main remaining failure is \emph{plan bypass}: in 67\% of error episodes, the AR reasoner emits \texttt{f\_final\_answer} at step~1 and ignores the DLM plan (\cref{tab:difftab_failures}). Plan-following episodes have lower schema hallucination (\cref{tab:schema_adherence}) and shorter trajectories (\cref{tab:planner_ablation}(a)), motivating future plan-adherence reward shaping (\cref{sec:app-exec-errors}).

\begin{table}[t]
\centering
\small
\caption{Component ablations on six benchmarks. \textbf{(a) Backbone ablation:} accuracy (\%) with avg.\ reasoning turns in parentheses, and latency (s/q) shown in {\scriptsize small} as understand\,+\,reason\,=\,total. \textbf{(b) Reward-signal ablation:} accuracy (\%) with the AR planner (Qwen3-VL-8B) and Qwen3-VL-8B executor fixed, varying only the reward signal. \tabattn{} outperforms all content-based verifiers, with the largest gains on numerically and structurally dense tasks. The bottom block swaps the cell mask for the human Oracle and a 20\% noised Oracle (\cref{sec:app-mask-source-ablation}) as upper bounds.}
\label{tab:planner_ablation}
\setlength{\tabcolsep}{4pt}
\resizebox{\textwidth}{!}{\begin{tabular}{lccccccc}
\toprule
\textbf{Configuration} & \textbf{WTQ} & \textbf{MMQA} & \textbf{InfoT} & \textbf{TFact} & \textbf{TAT} & \textbf{HiTab} & \textbf{Avg.} \\
\midrule
\rowcolor{gray!11}  \multicolumn{8}{l}{\textit{(a) Backbone ablation}} \\
\midrule
DLM Reasoning &
  \makecell{30.61 \\ \scriptsize ---} &
  \makecell{33.33 \\ \scriptsize ---\,+\,17.7} &
  \makecell{40.28 \\ \scriptsize ---\,+\,22.8} &
  \makecell{41.94 \\ \scriptsize ---\,+\,23.3} &
  \makecell{46.91 \\ \scriptsize ---\,+\,21.1} &
  \makecell{33.33 \\ \scriptsize ---\,+\,21.0} &
  \makecell{37.73 \\ \scriptsize ---\,+\,21.2$^\dagger$} \\
LLM Reasoning &
  \makecell{81.54 \\ \scriptsize ---\,+\,54.7} &
  \makecell{60.10 \\ \scriptsize ---\,+\,36.6} &
  \makecell{79.78 \\ \scriptsize ---\,+\,21.5} &
  \makecell{87.05 \\ \scriptsize ---\,+\,21.6} &
  \makecell{79.80 \\ \scriptsize ---\,+\,49.9} &
  \makecell{73.87 \\ \scriptsize ---\,+\,58.6} &
  \makecell{77.02 \\ \scriptsize ---\,+\,40.5} \\
\cdashline{1-8}[2pt/1pt]
LLM Planned Reasoning &
  \makecell{90.32 (2.13) \\ \scriptsize 3.5\,+\,124.7\,=\,128.2} &
  \makecell{61.34 (1.60) \\ \scriptsize 8.5\,+\,142.6\,=\,151.1} &
  \makecell{83.41 (1.50) \\ \scriptsize 3.4\,+\,70.3\,=\,73.7} &
  \makecell{\textbf{92.39} (2.86) \\ \scriptsize 3.5\,+\,58.6\,=\,62.1} &
  \makecell{85.50 (2.08) \\ \scriptsize 13.4\,+\,127.6\,=\,141.0} &
  \makecell{78.50 (2.38) \\ \scriptsize 3.6\,+\,92.0\,=\,95.6} &
  \makecell{81.91 (2.09) \\ \scriptsize 6.0\,+\,102.6\,=\,108.6} \\
\rowcolor{blue!8} DLM Planned Reasoning &
  \makecell{\textbf{94.12} (1.89) \\ \scriptsize 85.6\,+\,71.8\,=\,157.4} &
  \makecell{\textbf{71.07} (1.46) \\ \scriptsize 95.8\,+\,85.8\,=\,181.6} &
  \makecell{\textbf{90.45} (1.49) \\ \scriptsize 41.4\,+\,24.2\,=\,65.6} &
  \makecell{90.09 (1.99) \\ \scriptsize 86.4\,+\,64.2\,=\,150.6} &
  \makecell{\textbf{91.41} (1.58) \\ \scriptsize 78.6\,+\,47.0\,=\,125.6} &
  \makecell{\textbf{82.48} (1.38) \\ \scriptsize 80.2\,+\,47.5\,=\,127.7} &
  \makecell{\textbf{86.60} (1.63) \\ \scriptsize 78.0\,+\,56.8\,=\,134.8} \\
\midrule[\heavyrulewidth]
\rowcolor{gray!11}  \multicolumn{8}{l}{\textit{(b) Reward-signal ablation}} \\
\midrule
No Reward  (Baseline) & 81.8 & 54.3 & 75.9 & 82.4 & 59.8 & 65.8 & 70.0 \\
\cdashline{1-8}[2pt/1pt]
Cosine Sim~\citep{reimers2019sbert} & 77.8 & 56.6 & 78.2 & 83.8 & 81.5 & 67.0 & 74.2 \\
Numerical F1~\citep{zhu2021tatqa} & 76.5 & 56.3 & 78.4 & 84.3 & 79.4 & 66.5 & 73.6 \\
Cosine + Numerical & 75.9 & 60.1 & 77.3 & 85.3 & 79.0 & 65.5 & 73.9 \\
TABROUGE~\citep{kwok2026retab} & 81.2 & 57.7 & 80.2 & \textbf{92.4} & 85.9 & 65.5 & 77.1 \\
\cdashline{1-8}[2pt/1pt]
\rowcolor{blue!8} \tabattn{} \textbf{(ours)} & \textbf{90.3} & \textbf{61.3} & \textbf{83.4} & \textbf{92.4} & \textbf{85.5} & \textbf{78.5} & \textbf{81.9} \\
\cdashline{1-8}[2pt/1pt]
\tabattn{} + 20\% Noise & {91.2} & {76.8} & {85.0} & {93.3} & {87.9} & {83.3} & {86.3} \\
\tabattn{} + Oracle & 93.0 & 78.6 & 88.3 & 95.0 & 87.9 & 86.7 & 88.3  \\
\bottomrule
\end{tabular}}
\vspace{-4mm}
\end{table}
  \vspace{-2mm}
\begin{table}[t]
\centering
\small
\caption{Failure mode taxonomy for LLaDA-8B end-to-end reasoning.}
\label{tab:vdlm_failures}
\resizebox{\textwidth}{!}{\begin{tabular}{llcl}
\toprule
\textbf{Failure mode} & \textbf{Description} & \textbf{\% wrong} & \textbf{Root cause} \\
\midrule
Duplicate action generation & Same tool call emitted 2--4$\times$ per step & 56\% & Fixed canvas force denoising on unused positions \\
Output format and parsing error & Wrong label set or argument syntax & 51\% & Weaker instruction alignment vs.\ AR \\
Canvas overflow & $\ge$5 actions packed into one step & 17\% & Fixed length forces action list compression \\
Empty answer & \texttt{f\_final\_answer} emits empty string & 9\% & Tool-call chain fails with no recovery path \\
Over-enumeration & Full column returned instead of single value & 7\% & No intermediate filter/select step executed \\
\bottomrule
\end{tabular}} \vspace{-5mm}
\end{table}

\textbf{Inference Latency and Accuracy--Cost Tradeoff.} The latency rows of \cref{tab:planner_ablation}(a) split wall-clock time into understanding and reasoning phases. Consistent with prior planned-reasoning work~\citep{nguyen2025interpretable,rawat2025preactmultistepplanningreasoning,deng2026planu}, planning improves accuracy at added inference cost. DLM planning is slower than AR planning because multi-step denoising is expensive, especially on long, multi-column tables such as WTQ, MMQA, and TabFact. However, better DLM plans reduce downstream execution time from 102.6s to 56.8s on average, a 44.64\% drop. This makes DLM acceleration, e.g., batched denoising~\citep{song2025seeddiffusionlargescalediffusion,wang2026diffusion} or flow-matching distillation~\citep{sahoo2025the}, a promising future direction.

\textbf{Reward-Signal Ablation: Cell-Grounded vs Content-Based Supervision.} \cref{sec:attn-pilot} motivated cell-grounded attention as a stronger supervision target than content-similarity rewards. We confirm this empirically in \cref{tab:planner_ablation}(b) by holding the AR planner and executor as Qwen3-VL-8B, varying only the reward signal. \tabattn{} reaches 81.9\% average accuracy, 4.8~pp above TABROUGE (77.1\%). Per-dataset, the gap widens on numerically and structurally dense tasks, with 9.1~pp and 13~pp over TABROUGE on WTQ and HiTab, where attention-grounded supervision distinguishes correct cell selection from spurious content overlap. Additional human-verified and 20\% noised oracles in \cref{tab:planner_ablation}(b) indicates that \tabattn{} is within 6.4~pp of the ceiling (full setup in \cref{sec:app-mask-source-ablation}).

\begin{wraptable}{r}{0.48\textwidth}
\vspace{-3mm}
\centering
\small
\caption{Plan quality: AR vs.\ DLM column hallucination rates, measured per plan and per step.}\label{tab:schema_adherence}
\vspace{-2mm}
\setlength{\tabcolsep}{4pt}
\resizebox{0.83\linewidth}{!}{\begin{tabular}{lcccc}
\toprule
\multirow{2}{*}{\textbf{Dataset}} &
\multicolumn{2}{c}{\textbf{Avg.\ halluc.} $\downarrow$} &
\multicolumn{2}{c}{\textbf{\% steps w/ halluc.} $\downarrow$} \\
\cmidrule(lr){2-3}\cmidrule(lr){4-5}
 & AR & DLM & AR & DLM \\
\midrule
WTQ   & 64.5 & \hi{\textbf{41.5}} & 74.0 & \hi{\textbf{50.0}} \\
MMQA     & 82.0 & \hi{\textbf{40.5}} & 87.4 & \hi{\textbf{50.3}} \\
TabFact  & 86.9 & \hi{\textbf{46.8}} & 94.9 & \hi{\textbf{57.5}} \\
HiTab    & 79.2 & \hi{\textbf{44.1}} & 91.5 & \hi{\textbf{48.1}} \\
TabMWP   & 59.4 & \hi{\textbf{23.8}} & 67.4 & \hi{\textbf{27.6}} \\
TAT-QA   & \textbf{78.2} & \hi{78.7} & 87.1 & \hi{\textbf{83.5}} \\
Avg.     & 75.0 & \hi{\textbf{45.9}} & 83.7 & \hi{\textbf{52.8}} \\
\bottomrule
\end{tabular}}
\vspace{-4mm}
\end{wraptable} \textbf{Plan Quality Rubric: Column Hallucination Rate.} To test whether bidirectional attention improves table understanding, we compare AR (Qwen3-VL-8B) and DLM (LLaDA-8B-Instruct) planners by column hallucination rate. \Cref{tab:schema_adherence} shows that across six datasets, the DLM reduces average hallucination from 75.0\% to 45.9\%, and hallucinatory steps from 83.7\% to 52.8\%. Gains are largest on structurally rich datasets (MMQA, TabFact, HiTab), while shallow-schema TAT-QA shows similar rates. This aligns with \cref{sec:attn-pilot} as bidirectional attention improves adherence on non-trivial schemas, supporting DLM as planner. %
 \vspace{-1mm}
\section{Conclusion} \label{sec:discussion-conclusion} \vspace{-1mm}
We presented \ourmodel{}, a table-reasoning framework that enforces cell grounding through a masked diffusion planner and \tabattn{} stepwise verifier. Across eight benchmarks, \ourmodel{} outperforms strongest 8B-class baseline by 15.76~pp on the six-benchmark accuracy average, with diagnostics attributing gains to attention-grounded understanding (\cref{sec:in-depth-analyses}). Our future work includes extending to complex tabular settings, and DLM acceleration in table reasoning (\cref{sec:app-outlook}). 

\textbf{Limitations.} Two limitations still persist. 67\% of \ourmodel{} error trajectories bypass the plan by answering at step~1 (\cref{tab:difftab_failures}), mainly on single-cell lookups, suggesting plan-adherence reward shaping. On FeTaQA, \ourmodel{} trails TableLlama-7B in BLEU largely due to BLEU penalizing correct paraphrases. Our LLM-judge analysis finds 78.6\% factual correctness, indicating the BLEU gap overstates the capability gap (\cref{sec:fetaqa-judge}). 
 
\begin{ack}
Lei Ding is supported by the Natural Sciences and Engineering Research Council of Canada (NSERC) Discovery Grant (RGPIN-2026-06465) and a University of Manitoba Start-Up Grant. 
Zhijiang Guo is supported by the Guangdong Provincial Key Lab of Integrated Communication, Sensing and Computation for Ubiquitous Internet of Things (No. 2023B1212010007) 
\end{ack}

\bibliographystyle{plainnat}

\appendix
\newpage
\section*{Impact Statement}
\ourmodel{} contributes a general framework for grounded reasoning over structured tabular data, with downstream applications in financial analysis, scientific data interpretation, and enterprise question answering. By aligning model attention with human-verified cell relevance, the framework supports more interpretable and auditable reasoning trajectories than end-to-end black-box approaches. The curated attention standards and code released with this work are derived from publicly available academic benchmarks and contain no personally identifiable information, and we do not anticipate direct negative societal impacts beyond those broadly attributable to large language models. By explicitly modeling cell-level attention at both the understanding and reasoning stages, \ourmodel{} reduces trajectory drift and hallucination compounding that currently limit the reliability of multi-step table agents. The proposed \tabattn{} reward provides a verifiable, query-grounded supervision signal that makes intermediate reasoning steps inspectable, contributing to more transparent and auditable AI reasoning over structured data. \ourmodel{} also reduces downstream execution trajectory length by 22\% and downstream reasoning time by 44.64\%, enabling more focused and efficient agent execution. We anticipate that attention-grounded process rewards will generalize beyond table reasoning to other structured data modalities where cell- or token-level grounding can be formally specified.

\newpage
\section{Empirical Setup: Datasets, Baselines, and Implementation Details}
\label{sec:app-exp-details}

\paragraph{Dataset statistics.}
WTQ contains 2,108 test questions over heterogeneous web tables.
MMQA adds cross-modal multi-hop questions over tables paired with images and text.
TabMWP covers grade-school math word problems with tabular context.
TAT-QA pairs tabular and textual evidence for financial reasoning.
HiTab targets hierarchical tables requiring multi-level lookups.
FeTaQA requires fluent free-form answers graded by BLEU.
TabFact and InfoTabs are binary entailment benchmarks over Wikipedia and infobox tables, respectively.

\paragraph{Baseline details.}
We organize the baselines into six categories that span the dominant table reasoning paradigms in the recent literature. For each category we first define the paradigm, then describe how each individual baseline within the category operates.

\textit{Table-as-text LLMs} serialize the input table into a textual format (typically Markdown or HTML) and feed it to a text-only LLM that produces the answer in a single forward pass. This is the dominant pre-multimodal deployment pattern for tabular QA. Llama-2-7B~\citep{touvron2023llama2openfoundation} is a generalist instruction-tuned chat model with no table-specific adaptation, included as a baseline for the floor of generalist LLM capability at the 7B scale. Llama3-Instruct-8B~\citep{grattafiori2024llama3herdmodels} is the analogous third-generation generalist chat backbone. TableLlama-7B~\citep{zhang2024tablellamaopenlargegeneralist} is an LLM further post-trained on large-scale table instruction data, and TableGPT-R1~\citep{yang2025tableggpttablegptt} extends this line by adding reinforcement learning on tabular reasoning trajectories. TableLLM~\citep{zhang2024tablellm} additionally targets real-office spreadsheet manipulation tasks. Together this group quantifies how much table-specific post-training matters when the table is presented purely as text.

\textit{Table-as-image VLMs} render the table as an image and feed it to a vision language model that perceives the cell grid visually rather than through text serialization. The motivation is to bypass the brittleness of long Markdown serialization on wide tables and to retain visual cues such as merged cells and hierarchical headers~\citep{zheng2024tablellava}. Table-LLaVA-7B~\citep{zheng2024tablellava} and SynTab-LLaVA-7B~\citep{zhou2025syntab} are LLaVA-style VLMs further trained on real and synthetic table images. MiniCPM-V-2.6-8B~\citep{yao2024minicpmv}, Qwen2.5-VL-7B~\citep{bai2025qwen25vl}, and Qwen3-VL-8B~\citep{bai2025qwen3vltechnicalreport} are general-purpose VLMs evaluated on table-as-image inputs without table-specific tuning. We include Qwen3-VL-8B explicitly because it postdates the VLMs used by competing baselines and serves as the strongest contemporary open-source VLM at our parameter scale.

\textit{Joint multimodal frameworks} process both the textual serialization and the rendered image of the table jointly inside a single model, on the assumption that the two modalities expose complementary information (text preserves exact strings while image preserves layout). HIPPO-8B~\citep{wang2026hippo} is a recent open-source multimodal LLM that aligns the two modalities for table understanding through hybrid-modal preference optimization, conditioning the answer on a fused text-and-image representation for every query. Gemini 2.0 Flash~\citep{google2025gemini3flash} is a proprietary multimodal LLM included for research context, but it cannot serve as a primary comparison point for \ourmodel{} because \tabattn{} requires access to internal attention weights that proprietary endpoints do not expose.

\textit{Dynamic modality-routing frameworks} train a lightweight router that, per query, selects either the text branch, the image branch, or a fusion of both, rather than committing to a single modality for every input. TableDART~\citep{xing2026tabledart} is the representative method in this category. It keeps two frozen single-modality experts (a text expert built on TableGPT2-7B~\citep{yang2025tableggpttablegptt} and an image expert built on either Qwen2.5-VL-7B~\citep{bai2025qwen25vl} or Ovis2-8B~\citep{lu2025ovis25technicalreport}) and trains a 2-layer MLP gating network to pick the cheapest inference path that still answers correctly. We report both published TableDART variants. TableDART is the closest competing baseline to \ourmodel{} in spirit, since both methods explicitly refuse to commit to a static modality choice.

\textit{Trained trajectory-optimization agents} train an auxiliary model that supervises or prunes the multi-step reasoning trajectory produced by a base reasoner, with the goal of suppressing wrong intermediate steps before they propagate into the final answer. TATTOO~\citep{zou2026tattoo} is a tool-grounded process reward model (PRM) that scores each reasoning step of a DeepSeek-R1-Distill-Qwen-14B policy reasoner and drives best-of-$N$ selection at inference, using external tool execution to ground the per-step verification signal. TabTrim~\citep{guo2026tabtrim} is a parallel-search table pruner that trims irrelevant rows and columns from intermediate states under gold-trajectory supervision. These two methods isolate the contribution of trajectory-level supervision, and they form the most direct comparison point for \tabattn{}, which also operates as an in-the-loop step-level verifier but grounds its scores in attention rather than in an extra reward model or pruning network.

\textit{Training-free GPT-3.5 agents} prompt a closed-source GPT-3.5 backbone with handcrafted reasoning scaffolds and external tools, without any fine-tuning of the underlying LLM. We include this category for historical completeness, since the published numbers from these methods on WTQ, TabFact, and TabMWP set the prior open-domain accuracy bar. DATER~\citep{ye2023dater} decomposes a complex question into a sequence of sub-table operations. ReAcTable~\citep{zhang2024reactable} adapts the ReAct paradigm to tabular QA by interleaving SQL and Python tool calls with chain-of-thought reasoning. Mix-SC~\citep{liu2024mixsc} aggregates self-consistency across mixed text and symbolic reasoning paths. TIDE~\citep{yang2025tide} structures the table into triples to make question decomposition and step verification more tractable, and CIT-DP~\citep{yang2025cit} debiases the LLM through front-door causal intervention. Because these baselines are GPT-3.5 era, they are not the most direct comparison point for \ourmodel{}, but they remain useful for situating absolute accuracy on the long-running benchmarks.

Baselines are reproduced from published checkpoints under the same evaluation protocol where possible. Numbers taken directly from the original papers are marked with $\dagger$.

\paragraph{Baseline selection rationale.}
Our baseline pool is intentionally aligned with the most recently published peer-reviewed table reasoning works, in particular TableDART~\citep{xing2026tabledart} and TATTOO~\citep{zou2026tattoo}, so that the numbers in \cref{tab:main-results} are directly comparable to the latest published table reasoning results without re-implementation drift. We acknowledge that several baselines in this pool, such as Llama-2-7B, Table-LLaVA-7B, and Gemini 2.0 Flash, use backbones that are now one to two generations behind the frontier. To address the concern that this could understate current open-source capability and bias the comparison in our favor, we additionally evaluate Qwen3-VL-8B~\citep{bai2025qwen3vltechnicalreport} both as a standalone table-as-image VLM baseline and as the reasoner backbone in \ourmodel{}. Qwen3-VL-8B postdates the Qwen2.5-VL-7B and Ovis2-8B variants used by TableDART and the text-only Qwen3-8B backbone used by TATTOO, so the gains we report are measured against the strongest publicly available open-source VLM at our parameter scale at the time of submission, not only against the legacy backbones inherited from prior baseline pools.

\paragraph{Implementation details.}
The DLM understanding model (LLaDA-8B-Instruct) conditions on the question, column names, data types, and up to 15 sampled cell values. It runs 128 denoising steps with block length 32 and greedy decoding.
The AR reasoner (Qwen3-VL-8B) runs at most 6 tool steps with a maximum of 16,384 output tokens per step and temperature 0.7. The AR understanding model ablation uses the same Qwen3-VL-8B backbone to isolate the DLM contribution.
\ourmodel{} commits \textbf{one trajectory per query} ($N{=}1$): the DLM emits a plan via greedy denoising, the reasoner executes it step by step, and \tabattn{} acts as an in-the-loop step-level verifier whose score and rationale are passed to the reasoner before the next step. A stagnation halt emits the highest-scoring step's answer when $R_{\text{attn}}$ fails to improve by more than 0.02 for two consecutive steps and the table state is unchanged. This design redirects misaligned steps through in-loop feedback rather than paying the $N\times$ cost of an outer best-of-$N$ re-ranker over independently sampled trajectories.
\tabattn{}'s learned attention verifier (\cref{eq:rattn}) is fit by binary cross-entropy on terminal trajectory correctness using the curated 1{,}600 attention standards (\cref{sec:app-attention-curation}), split 80/20 with the inner 20\% reserved solely for early stopping. This is the single training procedure used in all reported results. Separately, to test cross-dataset generalisation, we run a transfer evaluation in which \tabattn{} is trained on WTQ examples only and measured on held-out performance on the remaining seven benchmarks (see \cref{sec:app-components} and the transfer results below).

\paragraph{\tabattn{} cross-dataset transfer.}
\tabattn{} trained on WTQ attention standards transfers to all seven held-out benchmarks without degradation ($\Delta\rho \leq 0$ in all cases), with AUROC of 0.95 to 1.00 on six of seven benchmarks.
TabMWP is the outlier (AUROC~$=0.779$), consistent with its near-ceiling accuracy (95.1\%) limiting binary discrimination.
MMQA shows the sharpest transfer benefit ($\rho{=}0.888$ vs.\ $0.834$ when retrained on MMQA), suggesting WTQ's trajectory diversity provides a stronger training signal than the target domain's own data.

\paragraph{Volatility of reported accuracies.}
We adopt $n{=}200$ evaluation questions per dataset under a single deterministic decoding pass (\cref{sec:experiments}), matching the per-dataset budget used by recently accepted table reasoning works such as TaTToo~\citep{zou2026tattoo} and TableDART~\citep{xing2026tabledart}. To surface volatility for every reported accuracy under this protocol, \cref{tab:appendix-volatility} attaches a per-cell binomial standard deviation $\sigma{=}\sqrt{p(1-p)/n}$ to every accuracy in \cref{tab:main-results} and \cref{tab:planner_ablation}, with the Avg.\ column reporting the corresponding macro-mean alongside its propagated standard error under independent binomials. Per-cell SDs sit between $1.3$ and $3.5$\,pp at $n{=}200$, narrowing on the macro-mean. We report this volatility analysis as a transparency aid since prior table reasoning works do not disclose run-to-run variance under matched evaluation budgets, and we use the binomial form rather than seed-level $\sigma$ because the dominant source of variance under deterministic decoding is finite-sample evaluation noise rather than stochastic generation.
\begin{table}[t!]
\centering
\caption{\textbf{Volatility study via binomial-approximation standard deviations.} For every accuracy value reported in \cref{tab:main-results} (Part a) and \cref{tab:planner_ablation} (Parts b1, b2), we report the binomial standard deviation $\sigma=\sqrt{p(1-p)/n}$ at $n{=}200$ evaluations per dataset, the unified evaluation budget under our protocol. Cells are written \emph{accuracy\,$\pm$\,$\sigma$} in percentage points (pp). The Avg.\ column reports the macro-average mean and the propagated standard error of that mean under independent binomials, $\mathrm{SE}(\bar{p})=\frac{1}{k}\sqrt{\sum_i p_i(1-p_i)/n}$, with $k$ the number of datasets averaged. FeTaQA is omitted because BLEU is not a binomial accuracy. For baselines re-evaluated under our pipeline $n{=}200$ matches the actual sample size, for numbers cited from prior work $n{=}200$ is used as a consistency-level estimate (original protocols may differ). Avg.\ matches each source table's convention (Part a excludes MMQA, Parts b1, b2 exclude TABMWP).}
\label{tab:appendix-volatility}
\setlength{\tabcolsep}{3pt}
\resizebox{\textwidth}{!}{\begin{tabular}{lcccccccc}
\toprule
\textbf{Method} & \textbf{WTQ} & \textbf{MMQA} & \textbf{TABMWP} & \textbf{TAT-QA} & \textbf{HiTab} & \textbf{TabFact} & \textbf{InfoT} & \textbf{Avg.} \\
\midrule
\multicolumn{9}{l}{\textit{Part (a): Main benchmark results} (mirrors \cref{tab:main-results})} \\
\midrule
\rowcolor{gray!11} \multicolumn{9}{l}{\textit{Table-as-Text Baselines}} \\
Llama-2-7B & 16.39\,$\pm$\,2.62 & --- & 22.82\,$\pm$\,2.97 & 13.73\,$\pm$\,2.43 & 10.72\,$\pm$\,2.19 & 9.20\,$\pm$\,2.04 & 38.92\,$\pm$\,3.45 & 18.63\,$\pm$\,1.09 \\
Llama3-Instruct-8B & 21.24\,$\pm$\,2.89 & --- & 42.01\,$\pm$\,3.49 & 13.08\,$\pm$\,2.38 & 6.97\,$\pm$\,1.80 & 73.89\,$\pm$\,3.11 & 54.00\,$\pm$\,3.52 & 35.20\,$\pm$\,1.20 \\
TableLlama-7B & 24.97\,$\pm$\,3.06 & --- & 10.10\,$\pm$\,2.13 & 19.04\,$\pm$\,2.78 & 46.57\,$\pm$\,3.53 & 79.37\,$\pm$\,2.86 & 46.57\,$\pm$\,3.53 & 37.77\,$\pm$\,1.23 \\
\midrule
\rowcolor{gray!11} \multicolumn{9}{l}{\textit{Table-as-Image Baselines}} \\
Table-LLaVA-7B & 18.43\,$\pm$\,2.74 & --- & 57.78\,$\pm$\,3.49 & 12.82\,$\pm$\,2.36 & 10.09\,$\pm$\,2.13 & 59.85\,$\pm$\,3.47 & 65.26\,$\pm$\,3.37 & 37.37\,$\pm$\,1.22 \\
SynTab-LLaVA-7B & 39.59\,$\pm$\,3.46 & --- & 88.30\,$\pm$\,2.27 & 51.94\,$\pm$\,3.53 & 35.66\,$\pm$\,3.39 & 70.78\,$\pm$\,3.22 & 69.42\,$\pm$\,3.26 & 59.28\,$\pm$\,1.31 \\
MiniCPM-V-2.6-8B & 47.97\,$\pm$\,3.53 & --- & 83.68\,$\pm$\,2.61 & 51.55\,$\pm$\,3.53 & 56.53\,$\pm$\,3.51 & 78.48\,$\pm$\,2.91 & 73.03\,$\pm$\,3.14 & 65.21\,$\pm$\,1.32 \\
Qwen2.5-VL-7B & 54.37\,$\pm$\,3.52 & --- & 63.69\,$\pm$\,3.40 & 51.94\,$\pm$\,3.53 & 62.69\,$\pm$\,3.42 & 75.81\,$\pm$\,3.03 & 70.13\,$\pm$\,3.24 & 63.10\,$\pm$\,1.37 \\
Qwen3-VL-8B & 81.82\,$\pm$\,2.73 & --- & 79.90\,$\pm$\,2.83 & 59.80\,$\pm$\,3.47 & 65.83\,$\pm$\,3.35 & 82.41\,$\pm$\,2.69 & 75.88\,$\pm$\,3.03 & 74.27\,$\pm$\,1.24 \\
\midrule
\rowcolor{gray!11} \multicolumn{9}{l}{\textit{Table-as-Multimodality Baselines}} \\
HIPPO-8B & 55.77\,$\pm$\,3.51 & --- & 87.50\,$\pm$\,2.34 & 60.75\,$\pm$\,3.45 & 63.00\,$\pm$\,3.41 & 82.27\,$\pm$\,2.70 & 75.74\,$\pm$\,3.03 & 70.84\,$\pm$\,1.27 \\
Gemini 2.0 Flash & 63.56\,$\pm$\,3.40 & --- & 46.29\,$\pm$\,3.53 & 35.62\,$\pm$\,3.39 & 60.41\,$\pm$\,3.46 & 81.33\,$\pm$\,2.76 & 54.31\,$\pm$\,3.52 & 56.92\,$\pm$\,1.37 \\
\midrule
\rowcolor{gray!11} \multicolumn{9}{l}{\textit{Training-free TableQA Agents (Proprietary GPT3.5, cited)}} \\
DATER & 65.90\,$\pm$\,3.35 & --- & --- & --- & --- & 85.60\,$\pm$\,2.48 & --- & 75.75\,$\pm$\,2.09 \\
ReAcTable & 68.00\,$\pm$\,3.30 & --- & --- & --- & --- & 86.10\,$\pm$\,2.45 & --- & 77.05\,$\pm$\,2.05 \\
Mix-SC & 73.70\,$\pm$\,3.11 & --- & --- & --- & --- & 88.50\,$\pm$\,2.26 & --- & 81.10\,$\pm$\,1.92 \\
TIDE & 75.00\,$\pm$\,3.06 & --- & --- & --- & --- & 89.82\,$\pm$\,2.14 & --- & 82.41\,$\pm$\,1.87 \\
CIT-DP & 76.40\,$\pm$\,3.00 & --- & --- & --- & --- & 91.30\,$\pm$\,1.99 & --- & 83.85\,$\pm$\,1.80 \\
\midrule
\rowcolor{gray!11} \multicolumn{9}{l}{\textit{Routing on TableGPT2-7B / Ovis2-8B (cited)}} \\
TG2-7B (text) & 61.42\,$\pm$\,3.44 & --- & 83.87\,$\pm$\,2.60 & 50.39\,$\pm$\,3.54 & 70.27\,$\pm$\,3.23 & 77.80\,$\pm$\,2.94 & 71.07\,$\pm$\,3.21 & 69.14\,$\pm$\,1.30 \\
Ovis2-8B (image) & 58.76\,$\pm$\,3.48 & --- & 87.00\,$\pm$\,2.38 & 47.67\,$\pm$\,3.53 & 68.59\,$\pm$\,3.28 & 80.80\,$\pm$\,2.79 & 74.11\,$\pm$\,3.10 & 69.49\,$\pm$\,1.27 \\
TableDART (Q2.5VL) & 69.29\,$\pm$\,3.26 & --- & 72.61\,$\pm$\,3.15 & 59.07\,$\pm$\,3.48 & 71.13\,$\pm$\,3.20 & 77.94\,$\pm$\,2.93 & 71.46\,$\pm$\,3.19 & 70.25\,$\pm$\,1.31 \\
TableDART (Ovis2-8B) & 70.58\,$\pm$\,3.22 & --- & 84.54\,$\pm$\,2.56 & 62.05\,$\pm$\,3.43 & 74.37\,$\pm$\,3.09 & 81.37\,$\pm$\,2.75 & 76.22\,$\pm$\,3.01 & 74.86\,$\pm$\,1.23 \\
\midrule
\rowcolor{gray!11} \multicolumn{9}{l}{\textit{Trajectory-optimizing agents (cited)}} \\
DS + Qwen2.5-Math-PRM-72B & 69.20\,$\pm$\,3.26 & 24.40\,$\pm$\,3.04 & --- & --- & --- & 57.90\,$\pm$\,3.49 & --- & 63.55\,$\pm$\,2.39 \\
DS + TATTOO (Qwen3-8B) & 69.80\,$\pm$\,3.25 & 25.10\,$\pm$\,3.07 & --- & --- & --- & 58.79\,$\pm$\,3.48 & --- & 64.30\,$\pm$\,2.38 \\
TabTrim-4B & 76.80\,$\pm$\,2.98 & --- & --- & --- & --- & 89.40\,$\pm$\,2.18 & --- & 83.10\,$\pm$\,1.85 \\
TabTrim-8B & 79.40\,$\pm$\,2.86 & --- & --- & --- & --- & 91.20\,$\pm$\,2.00 & --- & 85.30\,$\pm$\,1.75 \\
\midrule
\rowcolor{gray!11} \multicolumn{9}{l}{\textit{\ourmodel{}}} \\
LLM + Q2.5VL & 69.89\,$\pm$\,3.24 & 47.22\,$\pm$\,3.53 & 87.94\,$\pm$\,2.30 & 74.87\,$\pm$\,3.07 & 75.38\,$\pm$\,3.05 & 83.33\,$\pm$\,2.64 & 77.89\,$\pm$\,2.93 & 78.22\,$\pm$\,1.18 \\
DLM + Q2.5VL & 73.76\,$\pm$\,3.11 & 52.53\,$\pm$\,3.53 & 88.50\,$\pm$\,2.26 & 78.37\,$\pm$\,2.91 & 79.00\,$\pm$\,2.88 & 82.41\,$\pm$\,2.69 & 79.71\,$\pm$\,2.84 & 80.29\,$\pm$\,1.14 \\
LLM + Q3VL & 90.32\,$\pm$\,2.09 & 61.34\,$\pm$\,3.44 & 96.35\,$\pm$\,1.33 & 85.50\,$\pm$\,2.49 & 78.50\,$\pm$\,2.90 & 92.39\,$\pm$\,1.87 & 83.41\,$\pm$\,2.63 & 87.75\,$\pm$\,0.93 \\
DLM + Q3VL & 94.12\,$\pm$\,1.66 & 71.07\,$\pm$\,3.21 & 95.13\,$\pm$\,1.52 & 91.41\,$\pm$\,1.98 & 82.48\,$\pm$\,2.69 & 90.09\,$\pm$\,2.11 & 90.45\,$\pm$\,2.08 & 90.61\,$\pm$\,0.83 \\
\midrule[\heavyrulewidth]
\multicolumn{9}{l}{\textit{Part (b1): Backbone ablation} (mirrors \cref{tab:planner_ablation}(a))} \\
\midrule
DLM Reasoning & 30.61\,$\pm$\,3.26 & 33.33\,$\pm$\,3.33 & --- & 46.91\,$\pm$\,3.53 & 33.33\,$\pm$\,3.33 & 41.94\,$\pm$\,3.49 & 40.28\,$\pm$\,3.47 & 37.73\,$\pm$\,1.39 \\
LLM Reasoning & 81.54\,$\pm$\,2.74 & 60.10\,$\pm$\,3.46 & --- & 79.80\,$\pm$\,2.84 & 73.87\,$\pm$\,3.11 & 87.05\,$\pm$\,2.37 & 79.78\,$\pm$\,2.84 & 77.02\,$\pm$\,1.19 \\
LLM Planned Reasoning & 90.32\,$\pm$\,2.09 & 61.34\,$\pm$\,3.44 & --- & 85.50\,$\pm$\,2.49 & 78.50\,$\pm$\,2.90 & 92.39\,$\pm$\,1.87 & 83.41\,$\pm$\,2.63 & 81.91\,$\pm$\,1.07 \\
DLM Planned Reasoning & 94.12\,$\pm$\,1.66 & 71.07\,$\pm$\,3.21 & --- & 91.41\,$\pm$\,1.98 & 82.48\,$\pm$\,2.69 & 90.09\,$\pm$\,2.11 & 90.45\,$\pm$\,2.08 & 86.60\,$\pm$\,0.96 \\
\midrule
\multicolumn{9}{l}{\textit{Part (b2): Reward-signal ablation} (mirrors \cref{tab:planner_ablation}(b))} \\
\midrule
Cosine Sim & 77.80\,$\pm$\,2.94 & 56.60\,$\pm$\,3.50 & --- & 81.50\,$\pm$\,2.75 & 67.00\,$\pm$\,3.32 & 83.80\,$\pm$\,2.61 & 78.20\,$\pm$\,2.92 & 74.15\,$\pm$\,1.23 \\
Numerical F1 & 76.50\,$\pm$\,3.00 & 56.30\,$\pm$\,3.51 & --- & 79.40\,$\pm$\,2.86 & 66.50\,$\pm$\,3.34 & 84.30\,$\pm$\,2.57 & 78.40\,$\pm$\,2.91 & 73.57\,$\pm$\,1.24 \\
Cosine + Numerical & 75.90\,$\pm$\,3.02 & 60.10\,$\pm$\,3.46 & --- & 79.00\,$\pm$\,2.88 & 65.50\,$\pm$\,3.36 & 85.30\,$\pm$\,2.50 & 77.30\,$\pm$\,2.96 & 73.85\,$\pm$\,1.24 \\
TABROUGE & 81.20\,$\pm$\,2.76 & 57.70\,$\pm$\,3.49 & --- & 85.90\,$\pm$\,2.46 & 65.50\,$\pm$\,3.36 & 92.40\,$\pm$\,1.87 & 80.20\,$\pm$\,2.82 & 77.15\,$\pm$\,1.16 \\
\tabattn{} (ours) & 90.30\,$\pm$\,2.09 & 61.30\,$\pm$\,3.44 & --- & 85.50\,$\pm$\,2.49 & 78.50\,$\pm$\,2.90 & 92.40\,$\pm$\,1.87 & 83.40\,$\pm$\,2.63 & 81.90\,$\pm$\,1.07 \\
\bottomrule
\end{tabular}}
\end{table}
 
\subsection{FeTaQA: LLM-Judge Supplementary Analysis}
\label{sec:fetaqa-judge}

FeTaQA is evaluated with BLEU throughout this paper, consistent with standard practice~\citep{nan2022fetaqa}.
The BLEU score of 32.63 reported in \cref{tab:main-results} is computed after a post-processing step that rewrites fragmentary fact-extraction outputs into complete descriptive sentences using a secondary LLM call with no gold reference (natural mode), following the dataset-specific answer generation approach used in prior work~\citep{xing2026tabledart}.
The diagnostic analysis below is conducted on raw agent outputs \emph{before} this refinement step, on a 185-example subset, to illustrate why unrefined outputs score poorly on BLEU despite being factually correct.
BLEU measures n-gram overlap against a single human reference phrasing and is insensitive to semantically equivalent answers expressed differently, so fact-extraction outputs systematically under-score even when the factual content is correct.

To quantify this gap, we ran a post-hoc LLM judge (GPT-5-nano) on these 185 pre-refinement FeTaQA trajectories, asking whether each predicted answer is \emph{factually correct} regardless of phrasing.
The judge was given the question, a table excerpt, the reference answer, and the predicted answer, and asked to respond with exactly one word (\textsc{correct} or \textsc{incorrect}).
Results are summarised below.

\begin{center}
\small
\begin{tabular}{lr}
\toprule
\textbf{Metric} & \textbf{Value} \\
\midrule
Trajectories judged            & 185 \\
Judge-correct                  & 125 \hspace{1pt}(67.6\%) \\
Judge-incorrect                & 34 \hspace{1pt}(18.4\%) \\
Unresolved (judge abstained)   & 26 \hspace{1pt}(14.1\%) \\
Mean BLEU (all)                & 0.089 $\pm$ 0.165 \\
Judge-correct but BLEU $<$ 0.05 & 81 \hspace{1pt}(43.8\%) \\
\bottomrule
\end{tabular}
\end{center}

The key figure is the last row: \textbf{81 predictions (43.8\%) are judged semantically correct yet score near-zero BLEU}, confirming that the low BLEU reflects a phrasing mismatch rather than a factual failure.
Among the 159 trajectories where the judge reached a verdict, 78.6\% were judged correct.

We emphasise that this analysis does \emph{not} replace the BLEU score reported in \cref{tab:main-results}. It supplements it.
The LLM-judge methodology does not follow the standard FeTaQA evaluation protocol, and direct comparison with baselines that report BLEU should be made only via the BLEU column.
The judge results are offered as diagnostic evidence that the capability gap is smaller than the metric gap suggests.

\newpage
\section{Method Details}
\label{sec:app-method-details}

This section consolidates implementation and design details for the three components of \ourmodel{}: the architectural choice that splits planning from execution (\cref{sec:app-why-split}), the diffusion-based understanding model that emits plans (\cref{sec:app-planner}), the curated attention standards that train the verifier (\cref{sec:app-attention-curation}), and the design justification for the \tabattn{} verifier itself (\cref{sec:app-components}).

\subsection{Why Planning and Execution Should Be Split}
\label{sec:app-why-split}

Table reasoning decomposes into a global \emph{planning phase} and a reactive \emph{execution phase} that place different computational demands on the model. Four reasons justify assigning these phases to different architectures.

\paragraph{Reason 1: execution is inherently tool-conditioned and sequential.}
During execution, the model must read the latest tool output (often a transformed dataframe printed as text) and generate the next tool call conditioned on that observed state. This creates an observe $\rightarrow$ respond $\rightarrow$ observe loop, not a one-shot structured prediction problem. LLaDA-style masked diffusion models instead generate by iteratively denoising a fully specified masked canvas, which is well suited to holistic sequence construction but not naturally suited to reactive conditioning on newly arriving external observations~\citep{nie2025llada}. Autoregressive models natively condition each next token on all previously revealed context, including newly returned tool outputs.

\paragraph{Reason 2: DLMs require fixed output length, which is unknown at execution time.}
Tool outputs vary substantially across steps: a selection operation may return a short list of columns, a filter may return many rows, and an aggregation may collapse the state to a scalar. The reasoner must consume these variable-length observations and decide when to terminate with a final answer. This is a poor match for masked diffusion inference, where the generation canvas length is typically fixed before decoding begins. Recent work makes this limitation explicit: DAEDAL frames statically specified generation length as a core inference bottleneck for diffusion language models, while FlexMDM starts from the same fixed-canvas assumption and requires substantial retraining to support flexible-length generation~\citep{li2026beyond,kim2026anyorder}.
We verify this directly: LLaDA-8B running end-to-end table reasoning with a 256-token canvas achieves 40.28\% exact match on InfoTabs and 41.94\% on TabFact, compared with 79.78\% and 87.05\% for AR Reasoning (\cref{tab:planner_ablation}), a gap of 39 to 45\,pp.
We note that 256 tokens is the largest canvas that reliably fits within LLaDA-8B's fixed-canvas constraint for multi-step table inputs. Larger canvases that exceed the model's pre-trained length distribution cause degenerate outputs.

\paragraph{Reason 3: the DLM attention advantage is a comprehension advantage, not a generation advantage.}
The correct tool call at step $t$ depends on the live tool output $o_{t-1}$ returned by step $t{-}1$. That output is not available before denoising begins, because it is produced by executing the prior action in the environment. A DLM cannot condition step $t$'s tool call on $o_{t-1}$ without first observing $o_{t-1}$, creating a fundamental circular dependency. Autoregressive decoding resolves this naturally: each token is conditioned on all previously revealed context, including freshly returned tool outputs interleaved between generation steps. Masked diffusion, which denoises a fixed canvas in parallel, has no mechanism for this interleaved conditioning without restarting denoising from scratch after every tool call.\footnote{Schema-scaffolding approaches such as S$^3$~\citep{xiong2026unveiling} cannot side-step this: the output schema at execution time is data-dependent (whether the next action is \texttt{f\_aggregate} or \texttt{f\_filter\_row} is determined by the prior tool output that has not yet been observed).}

\paragraph{Reason 4: AR reasoners are better aligned to multimodal grounding, instruction following, and error recovery.}
The reasoner must perceive rendered table images, and no current diffusion language model supports image inputs. LLaDA, the DLM used as the understanding model, is text-only. Only an autoregressive VLM can handle the multimodal perception that action-conditioned modality routing requires~\citep{kwok2026tabqaworld}. Beyond this structural gap, instruction-tuned AR VLMs are better suited to following tool specifications exactly, remaining format-faithful, and recovering from imperfect intermediate states. The practical cost of using a DLM as reasoner is visible in \cref{tab:main-results}: the DLM-only baseline (LLaDA-8B) falls more than 30\,pp below AR Reasoning (Qwen3-VL-8B) on average. Current diffusion VLMs more broadly still lag instruction-tuned AR VLMs in stepwise interactive generation~\citep{arriola2025ar2d}.

\subsection{Diffusion-based Understanding Model: Implementation Details}
\label{sec:app-planner}

\paragraph{Fixed-length canvas and block denoising.}
The plan canvas has a fixed length of $L{=}256$ tokens~\citep{nie2025llada}, chosen to be long enough that no plan is truncated. Positions not consumed by the plan remain as padding tokens.
Denoising proceeds over semi-autoregressive blocks of 32 tokens: within each block all masked tokens are refined in parallel, and blocks are processed left-to-right, following the long-context diffusion practice of~\citet{liu2025longllada,he2026ultrallada}.
We run 128 denoising steps by default (256 for higher-quality plans at $2\times$ cost) with classifier-free guidance scale 1.5 and greedy decoding ($T{=}0$).

\paragraph{Tool vocabulary.}
The planning prompt presents six operations: \texttt{filter} (select rows satisfying a column condition), \texttt{sort} (order rows by a column), \texttt{aggregate} (sum/count/average/min/max over a column), \texttt{lookup} (retrieve a specific cell), \texttt{compare} (compare values across rows or columns), and \texttt{select} (extract a column subset).
Each generated step must name one of these operations and end with a \texttt{[target: col]} tag (or \texttt{[target: col, row N]} for cell-level steps) identifying the table cells it operates on.
Example: \textit{``Aggregate: count rows where \textsc{Country} = Algeria.\ [target: Country]''}.

\paragraph{Cell mask extraction.}
After generation, each \texttt{[target:]} tag is parsed to identify the designated column(s) and, where specified, row index.
All cells matching the designation are set to 1 in $\mathbf{m}_t \in \{0,1\}^C$, and all other cells are set to 0.
If a step produces no parseable tag ($+3.6$\,pp DLM-vs-AR parse-error increase, \cref{sec:app-exec-errors}), $\mathbf{m}_t$ defaults to a uniform mask and $R_{\text{attn}}$ is treated as uninformative for that step.

\subsection{Attention Standard Curation: Details and Model Comparison}
\label{sec:app-attention-curation}

\paragraph{Human seed annotations.}
Human annotators produced 100 seed attention maps from scratch, spanning all eight benchmarks.
For each (question, table) pair, annotators identified the minimal set of cells (by column and row) that a correct reasoning trajectory must attend to in order to answer the question.
Disagreements were resolved by majority vote, and inter-annotator agreement (IoU) exceeded 0.90 on a held-out verification subset.
The seed set covers a range of question types: single-cell lookup, multi-cell aggregation, cross-column comparison, and multi-hop.

\paragraph{LLM model selection.}
To identify the strongest LLM for scaling up from the 100 seeds, we sampled 20 WTQ examples from the seed set as a calibration split.
Five GPT models were prompted with each (question, table, question type) triple and asked to predict the same cell attention set.
Predictions were evaluated against human annotations using Precision, Recall, F1, and IoU:

\begin{table}[h]
\centering
\small
\caption{Attention alignment of five GPT models on the 20-example WTQ calibration split drawn from the 100 human seed annotations. GPT-5 is selected for mass-scale generation.}
\label{tab:attn-model-comparison}
\begin{tabular}{lcccc}
\toprule
Model & Precision & Recall & F1 & IoU \\
\midrule
GPT-4.1  & 0.834 & 0.674 & 0.698 & 0.594 \\
GPT-5.1  & 0.824 & 0.676 & 0.698 & 0.583 \\
GPT-5.2  & 0.868 & 0.776 & 0.776 & 0.686 \\
GPT-5.4  & 0.891 & 0.836 & 0.812 & 0.757 \\
\textbf{GPT-5}   & \textbf{0.937} & \textbf{0.849} & \textbf{0.866} & \textbf{0.818} \\
\bottomrule
\end{tabular}
\end{table}

GPT-5 leads on all four metrics and is the only model to exceed F1\,=\,0.85 and IoU\,=\,0.80, matching human annotator judgment most closely.

\paragraph{Mass-scale generation and quality control.}
Using the 100 human seeds as few-shot exemplars, GPT-5 was prompted to generate attention standards for all eight benchmarks: WTQ, MMQA, TabMWP, TAT-QA, HiTab, FeTaQA, TabFact, and InfoTabs.
Approximately 200 candidates were generated per dataset, stratified by question type where possible.
All generated examples were reviewed by human annotators post-generation. Examples failing the minimum IoU threshold of 0.70 were discarded and regenerated, yielding a total of 1{,}600 human-verified examples.
These 1{,}600 standards form the training set for \tabattn{} (\cref{eq:rattn}), fit on an 80/20 split with the inner 20\% held out only for early stopping.

\paragraph{Train/evaluation split disjointness.}
To rule out leakage between the 1{,}600 attention-standard training data and the QA evaluation in~\cref{tab:main-results}, all curated examples per benchmark are sampled exclusively from the \emph{training} or \emph{development} splits of each underlying QA dataset, and none overlap with the held-out test splits used to report downstream reasoning accuracy. We verified this disjointness automatically by hashing every (question, table) pair: zero collisions between the 1{,}600 attention-standard pool and the union of the eight test splits.

\subsection{\tabattn{} Attention Verifier: Design Justification}
\label{sec:app-components}

\tabattn{} uses $R_{\text{attn}}$ (\cref{eq:rattn}) as its sole process signal, fit by binary cross-entropy against terminal trajectory correctness on the curated 1{,}600 attention standards (\cref{sec:app-attention-curation}).
The ablation in \cref{tab:planner_ablation}(b) confirms that attention grounding alone outperforms lexical and semantic content signals, motivating this focused design.

\paragraph{On the validity of the understanding model's cell mask.}
One might worry about circularity: if the DLM understanding model occasionally hallucinates cell designations, then $R_{\text{attn}}$ penalizes a reasoner that correctly deviated from a flawed plan.
We bound this concern in two ways.
First, the understanding model operates over compact metadata rather than raw cell contents, so its hallucinations are typically schema-level (wrong column name) rather than cell-level, and are detectable by cross-referencing with the schema adherence analysis (\cref{tab:schema_adherence}).
Second, \tabattn{} is trained on terminal correctness labels. If the understanding model's cell designations were systematically misleading, the optimizer would suppress the attention signal during training, and the large positive contribution of $R_{\text{attn}}$ in \cref{tab:planner_ablation}(b) confirms the masks are informative on net.

\paragraph{Cross-dataset transfer.}
\tabattn{} trained on WTQ attention standards only transfers to all seven held-out benchmarks without degradation (\cref{sec:app-exp-details}), validating that the attention-grounding signal generalises across table types and question categories without per-dataset retuning.

\newpage
\section{Theoretical Foundations}
\label{sec:theory}

This section consolidates the formal results that motivate \ourmodel{}'s state-based reward design. We first show that unguided multi-turn reasoning accumulates variance with trajectory length (\cref{sub:variance-growth}), then characterise the lexical TABROUGE baseline and prove its Pareto-parsimony together with its monotonicity profile (\cref{sec:app-tabrouge,sub:monotone-rouge}), and finally prove that an explicit step reward acts as a compressed distillation signal that mitigates stochastic drift (\cref{sub:reward-drift}).

\subsection{Variance growth under unguided multi-turn reasoning}
\label{sub:variance-growth}

\begin{proposition}[Variance growth with interaction length]
\label{prop:accuracy-degradation}
Let $Y_S := \ell(\hat{a}(\tau_S), A)$ denote a nonnegative loss after $S$ interaction steps, where $Y_S = 0$ iff the agent produces the correct answer.
Assume each additional step introduces stochastic variability without systematic corrective bias:
$\mathbb{E}[Y_S \mid Y_{S-1}] = Y_{S-1}$.
Then $\mathrm{Var}(Y_S) \ge \mathrm{Var}(Y_{S-1})$.
\end{proposition}

\begin{proof}
By the law of total variance,
$\mathrm{Var}(Y_S)
= \mathrm{Var}(\mathbb{E}[Y_S \mid Y_{S-1}]) + \mathbb{E}[\mathrm{Var}(Y_S \mid Y_{S-1})]
= \mathrm{Var}(Y_{S-1}) + \mathbb{E}[\mathrm{Var}(Y_S \mid Y_{S-1})] \ge \mathrm{Var}(Y_{S-1}).$
\end{proof}

\paragraph{Implication.}
Without explicit step-level reward feedback, reasoning trajectories in tool-based TableQA accumulate variance monotonically with trajectory length.
State-based rewards (Section~\ref{sec:methodology}) break this by introducing a corrective signal that conditions each step on measurable progress, capping variance growth.

\subsection{TABROUGE: Lexical Baseline and Pareto-Parsimony}
\label{sec:app-tabrouge}
\label{sub:tabrouge-proof}

\paragraph{TABROUGE definition.}
TABROUGE~\citep{kwok2026retab} measures how much of the query $q$ is covered by the serialized intermediate table state $\mathrm{Enc}(s_{l+1})$:
\begin{equation}
\label{eq:tabrouge-app}
R_{\mathrm{TABROUGE}}(s_{l+1}; q) = \frac{\mathrm{LCS}(q,\,\mathrm{Enc}(s_{l+1}))}{|\mathrm{Enc}(s_{l+1})|},
\end{equation}
where $\mathrm{LCS}$ is the longest common subsequence length. The numerator rewards query-relevant lexical coverage, and the denominator penalizes diffuse states retaining irrelevant content. TABROUGE showed a Spearman correlation of 0.521 with downstream answer correctness on WTQ~\citep{kwok2026retab}.

Despite this correlation, TABROUGE has two structural limitations that motivate the \tabattn{} design.
First, it compares against $q$, but many correct intermediate states produce values (sums, argmax results, comparisons) that never appear in $q$, so lexical coverage fails silently on the very steps that matter most.
Second, as a surface-form metric it cannot distinguish an intermediate state where the right answer is \emph{present} from one where it was \emph{computed from the right cells}: a table that coincidentally contains the target value in an irrelevant column scores identically to one where the value was derived correctly.

\begin{proposition}[Pareto-parsimony of TABROUGE maximizers]
\label{prop:rouge_optimality}
Given query $q$ and a candidate set $\mathcal{S}$, assume $\mathcal{S}$ is closed under (a) deletion of any substring of $\mathrm{Enc}(s)$ disjoint from $\mathrm{LCS}(q, \mathrm{Enc}(s))$, and (b) appending of admissible content $\Gamma$ drawn from the source table.
Any maximizer
$s^* \in \arg\max_{s \in \mathcal{S}}\, R_{\mathrm{TABROUGE}}(s; q)$
is \emph{Pareto-parsimonious} within $\mathcal{S}$ in two senses.
First, removing any substring $\Delta \subset \mathrm{Enc}(s^*)$ disjoint from $\mathrm{LCS}(q, \mathrm{Enc}(s^*))$ strictly increases the score, so no removable redundancy exists.
Second, appending content $\Gamma$ with marginal density exceeding the current score also strictly increases the score, so $s^*$ does not omit available query-relevant evidence from $\mathcal{S}$.
This characterises the maximizer's internal structure, and it does not imply that any maximizer is a globally correct intermediate state, since $\mathcal{S}$ may not contain a state that is both parsimonious and factually correct.
\end{proposition}

\begin{proof}
Throughout the proof, the constructed states $\tilde s$ (substring removal) and $s^+$ (content appending) lie in $\mathcal{S}$ by the closure assumptions of \cref{prop:rouge_optimality}, so the strict inequalities below contradict the maximality of $s^*$ within $\mathcal{S}$.

\noindent\textbf{No removable redundancy.}
Suppose $\exists\,\Delta \subset \mathrm{Enc}(s^*)$ disjoint from $\mathrm{LCS}(q, \mathrm{Enc}(s^*))$.
Removing $\Delta$ yields $\tilde{s}$ with
$R_{\mathrm{TABROUGE}}(\tilde{s}; q)
= \frac{\mathrm{LCS}(q,\mathrm{Enc}(s^*))}{|\mathrm{Enc}(s^*)| - |\Delta|}
> R_{\mathrm{TABROUGE}}(s^*; q),$
contradicting maximality of $s^*$.

\noindent\textbf{No available query-relevant evidence is omitted.}
Suppose content $\Gamma$ can be appended to form $s^+$ with marginal density $\delta/|\Gamma| > R_{\mathrm{TABROUGE}}(s^*; q)$, i.e.\ the fraction of query-relevant tokens in $\Gamma$ exceeds the current score.
Then
$R_{\mathrm{TABROUGE}}(s^+; q)
= \frac{\mathrm{LCS}(q,\mathrm{Enc}(s^*)) + \delta}{|\mathrm{Enc}(s^*)| + |\Gamma|}
> R_{\mathrm{TABROUGE}}(s^*; q),$
since cross-multiplying the strict inequality $\delta/|\Gamma| > \mathrm{LCS}(q,\mathrm{Enc}(s^*))/|\mathrm{Enc}(s^*)|$ gives $|\mathrm{Enc}(s^*)|\cdot\delta > \mathrm{LCS}(q,\mathrm{Enc}(s^*))\cdot|\Gamma|$, which is exactly the condition for the above strict inequality. This contradicts maximality of $s^*$.
\end{proof}

\subsection{Monotone TABROUGE score under correct tool execution}
\label{sub:monotone-rouge}

Under correct and successful tool execution, each step either removes query-irrelevant content (filter/select) or appends query-derived content (aggregate/compute).
Let $L_k = |\mathrm{Enc}(s_k)|$, $c_k = \mathrm{LCS}(q, \mathrm{Enc}(s_k))$, and $\rho_k = c_k / L_k$.

\begin{itemize}[leftmargin=*,itemsep=1pt]
\item \textbf{Pruning step}: $\Delta L_k < 0$, $\Delta c_k = 0$ (query-relevant tokens preserved).
      For $c_k > 0$, $\rho_{k+1} = c_k / (L_k + \Delta L_k) > \rho_k$. The boundary case $c_k = 0$ gives $\rho_{k+1} = \rho_k = 0$, so the inequality is non-strict.
\item \textbf{Derivation step}: $\Delta L_k, \Delta c_k > 0$.
      The reward improves when $\Delta c_k / \Delta L_k > \rho_k$, i.e.\ the newly appended tokens are denser in query-relevant content than the current state.
\end{itemize}

This shows that TABROUGE is monotonically non-decreasing along the restricted subset of trajectories whose every step is either a valid filter/projection or a derivation step whose appended tokens lexically overlap with $q$. This subset is strictly smaller than the class of trajectories that correctly answer the query, because correct derivations frequently produce values such as sums, argmax outputs, or comparisons that never appear in $q$. At those steps $\Delta c_k = 0$ while $\Delta L_k > 0$, so TABROUGE strictly decreases at exactly the steps where computed evidence is generated. This residual lexical blind spot motivates the attention-grounded \tabattn{} reward in \cref{sec:methodology}.

\subsection{Reward model significance under partial observability}
\label{sub:reward-ib-prop}

\begin{proposition}[Reward model significance under partial observability]
\label{prop:reward_ib}
Let the table-reasoning task be a POMDP where $s_t$ is the hidden table state and $o_t = \phi(s_t)$ is a partial-snapshot observation.
The projection induces an information bottleneck: $I(s_t; o_t) \ll H(s_t)$.
Assume the scalar step reward $r_t = \mathcal{R}(a_t, s_{t+1}; q)$ is informative about $s_{t+1}$ beyond the partial observation, that is, $I(s_{t+1}; r_t \mid o_{t+1}) > 0$.
Then $r_t$ functions as a \emph{compressed table distillation signal}
satisfying $H(s_{t+1} \mid o_{t+1}, r_t) < H(s_{t+1} \mid o_{t+1})$,
compensating for the limited observability of $o_t$ and stabilizing long-horizon trajectory search.
\end{proposition}

\subsection{How explicit rewards mitigate stochastic drift}
\label{sub:reward-drift}

Without a state-based reward, tool-based table transitions follow a stochastic transition kernel $s_{t+1} \sim P(\cdot \mid s_t, a_t, \theta_{a_t})$.
Let $s_t^\star$ denote the correct intermediate state at step $t$, and let $d(\cdot, \cdot)$ be a divergence on encoded states (for example, normalized edit distance over $\mathrm{Enc}(\cdot)$).
The per-step divergence $\epsilon_t := d(s_{t+1}, s_{t+1}^\star)$ quantifies the imagination gap between the model's predicted outcome and the realized transformation.
Unobserved, these per-step divergences accumulate along the trajectory, and the cumulative divergence $\sum_{l=0}^{S} \epsilon_l$ upper-bounds $d(s_S, s_S^\star)$, which can grow unboundedly with $S$.

The step reward $r_t = R_{\mathrm{TABROUGE}}(a_t, s_{t+1}; q)$ functions as a compressed distillation signal that maps the hidden global table state to a one-dimensional progress measure.
By Proposition~\ref{prop:reward_ib}, conditioning on $r_t$ reduces the remaining state entropy, allowing the agent to prune divergent trajectories before error compounds across multiple turns.

\newpage
\section{Attention Pilot Extensions}
\label{sec:app-attn-extensions}

This section reports the four attention-side empirical analyses that extend the pilot study in \cref{sec:attn-pilot}: the bidirectional plan-attention signature that motivates the DLM planner (\cref{sec:attn-motivation}), the row-permutation stability stress test (\cref{sec:app-perm-stability}), the autoregressive-VLM image-attention analysis that justifies a text-side attention verifier (\cref{sec:app-llm-revis}), and extended attention-quality figures (\cref{sec:app-attn-figures}).

\subsection{Plan-region Attention: Upper-Triangle Fraction}
\label{sec:attn-motivation}
The DLM planning advantage has a directly measurable mechanistic signature: on a shared table-reasoning input, the diffusion model exhibits dense plan-to-plan attention with non-zero mass above the diagonal, whereas the autoregressive baseline has exactly zero forward attention by design (\cref{tab:upper_tri}, with full attention maps in \cref{sec:app-attn-figures}).
This bidirectional inductive bias is the mechanism underlying the schema-hallucination reduction reported in \cref{sec:in-depth-analyses}, together with the trajectory compression and execution stabilization reported in \cref{sec:app-diagnostics}.

\begin{table}[htbp]
\centering
\caption{Plan-region attention upper-triangle fraction. DLM's bidirectionality produces substantial attention above the diagonal, while AR is exactly zero by design.}
\label{tab:upper_tri}
\small
\begin{tabular}{lcc}
\toprule
\textbf{Model (architecture)} & \textbf{Upper-triangle fraction} & \textbf{Expected} \\
\midrule
LLaDA-8B-Instruct (DLM, \texttt{is\_causal=False}) & \hi{\textbf{39.9\%}} & $\approx$50\% \\
Qwen3-VL-8B (AR, causal mask) & \textbf{0.0\%} & 0\% \\
\bottomrule
\end{tabular}
\end{table}

Large-scale attention analyses across 1{,}600 questions and eight benchmarks further corroborate this finding, with their full per-dataset breakdown presented in \cref{sec:attn-pilot}.

\subsection{Permutation Stability of Table Attention}
\label{sec:app-perm-stability}

This subsection details the row-permutation stress test referenced in \cref{sec:attn-pilot}, which quantifies whether bidirectional pre-training, in addition to lifting the attention-AUROC \emph{level} (\cref{tab:auroc-combined}), also improves \emph{stability} of attention under row reordering, an empirically desirable property given the order-sensitive contextualization imposed by causal AR models on tabular data.

\paragraph{Setup.}
For each of the four models studied in \cref{sec:attn-pilot}, namely Qwen3-VL-8B (causal AR), ModernBERT-base (bidirectional MLM), DiffuLLaMA-7B (causally-pretrained LLaMA adapted via diffusion post-training and bidirectional mask-flipped inference), and LLaDA-8B (natively diffusion-pretrained DLM), we draw the first $50$ (table, question) records from each of the eight benchmarks. For every record we generate $K{=}5$ row-permuted views: permutation $0$ is the identity, and the other four are seeded from the record id so that all four models see the same row orderings. The human-annotated cell relevance mask is permuted in lockstep with the data rows so that AUROC remains comparable across orderings. Each view is run through the same attention-extraction $\to$ cell-scoring $\to$ AUROC pipeline as \cref{tab:auroc-combined}, and per record we compute $\sigma_{\text{AUROC}}$ across the five orderings (records with fewer than three valid orderings, e.g.\ due to truncation at the $2{,}048$-token cap, are dropped). The protocol therefore fixes \emph{everything} except row order, isolating attention stability from the AUROC level.

\paragraph{Results.}
\Cref{tab:perm-stability} summarises per-model statistics, and \cref{fig:perm-stability} shows the per-record $\sigma_{\text{AUROC}}$ distribution. Two observations consistent with the main-text claim emerge. First, both models with native bidirectional pre-training (LLaDA and ModernBERT) yield substantially lower per-record $\sigma_{\text{AUROC}}$ than the causal AR baseline: $-40.2\%$ and $-62.6\%$ in median, respectively, with $p$-values from a paired Wilcoxon test below $10^{-22}$ on $n{=}319$ records that have valid permutation samples for all four models. Second, DiffuLLaMA, which inherits a causally-pretrained LLaMA backbone, applies diffusion post-training, and uses bidirectional mask-flipped inference, shows median $\sigma$ statistically indistinguishable from AR ($p{=}0.48$). This dissociation isolates the mechanism: \emph{permutation-stable table attention requires the model's pre-training itself to be bidirectional}, and post-hoc diffusion adaptation alone is not sufficient.

\begin{table}[h]
\centering
\small
\caption{Per-record stability of attention AUROC under five row permutations of each (table, question) record. $\sigma_{\text{AUROC}}$ is computed across the five orderings of the same record, and each entry is the across-record statistic. Paired Wilcoxon tests pair each record across models on the $n{=}319$ records that have valid permutation samples in all four models. Lower $\sigma$ is better.}
\label{tab:perm-stability}
\setlength{\tabcolsep}{6pt}
\resizebox{0.95\textwidth}{!}{
\begin{tabular}{lcccc|c}
\toprule
\textbf{Model (attention)} & \textbf{median $\sigma$} & \textbf{mean $\sigma$} & \textbf{$p_{90}\,\sigma$} & \textbf{mean range} & \textbf{paired Wilcoxon $p$ vs.\ AR} \\
\midrule
Qwen3-VL-8B (Causal AR)        & $0.123$ & $0.125$ & $0.238$ & $0.311$ & --- \\
DiffuLLaMA-7B (DLM, LLaMA)     & $0.123$ & $0.132$ & $0.248$ & $0.327$ & $0.48$ \\
LLaDA-8B (Native DLM)          & $0.074$ & $0.086$ & $0.166$ & $0.212$ & $1.8{\times}10^{-23}$ \\
ModernBERT-base (Bidir.\ MLM)  & $0.046$ & $0.055$ & $0.105$ & $0.139$ & $1.1{\times}10^{-41}$ \\
\bottomrule
\end{tabular}}
\end{table}

\begin{figure}[h]
\centering
\includegraphics[width=0.75\linewidth]{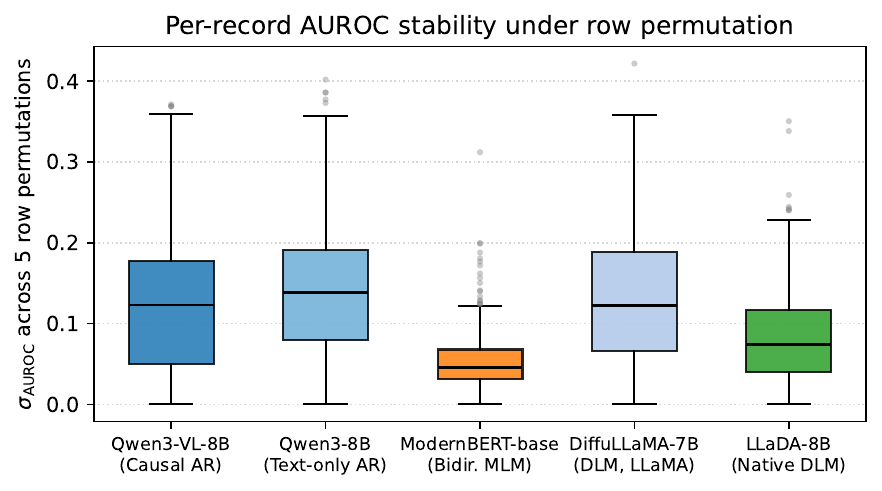}
\caption{Distribution of per-record $\sigma_{\text{AUROC}}$ across five row permutations of the same (table, question) record, by model. Both natively bidirectional pre-trained models (LLaDA, ModernBERT) cluster at substantially lower $\sigma$ than the causal AR baseline, while DiffuLLaMA, which is bidirectional only at inference, is indistinguishable from AR.}
\label{fig:perm-stability}
\end{figure}

\paragraph{Per-(model, dataset) breakdown.}
\Cref{tab:perm-stability-by-ds} reports mean $\sigma_{\text{AUROC}}$ for each model on each benchmark. The advantage of LLaDA and ModernBERT over AR, and the AR-vs-DiffuLLaMA null result, are consistent across all eight benchmarks rather than driven by any single dataset.

\begin{table}[h]
\centering
\small
\caption{Mean per-record $\sigma_{\text{AUROC}}$ by (model, dataset). Lower is better. The AR-vs-LLaDA / ModernBERT advantage and the AR-vs-DiffuLLaMA null result hold across all eight benchmarks.}
\label{tab:perm-stability-by-ds}
\setlength{\tabcolsep}{5pt}
\resizebox{0.95\textwidth}{!}{
\begin{tabular}{l|cccccccc}
\toprule
\textbf{Model} & \textbf{wikitq} & \textbf{tatqa} & \textbf{hitab} & \textbf{tabmwp} & \textbf{mmqa} & \textbf{infotabs} & \textbf{tabfact} & \textbf{fetaqa} \\
\midrule
Qwen3-VL-8B (AR)        & $0.106$ & $0.106$ & $0.167$ & $0.221$ & $0.134$ & $0.173$ & $0.091$ & $0.128$ \\
DiffuLLaMA-7B (DLM)     & $0.088$ & $0.073$ & $0.154$ & $0.198$ & $0.158$ & $0.215$ & $0.098$ & $0.115$ \\
LLaDA-8B (Native DLM)   & $0.064$ & $0.040$ & $0.121$ & $0.120$ & $0.089$ & $0.126$ & $0.063$ & $0.092$ \\
ModernBERT-base (MLM)   & $0.042$ & $0.045$ & $0.061$ & $0.069$ & $0.048$ & $0.105$ & $0.046$ & $0.047$ \\
\bottomrule
\end{tabular}}
\end{table}

\paragraph{Relation to the AUROC level result.}
Mean AUROC (\cref{tab:auroc-combined}) and per-record $\sigma_{\text{AUROC}}$ measure independent properties: a model can attend accurately on average yet still shift its attention pattern under row reordering, or attend stably yet inaccurately. The two results combine cleanly: LLaDA achieves both higher mean AUROC than AR and strictly lower $\sigma$, ModernBERT trades level for the largest $\sigma$ reduction observed, and DiffuLLaMA matches AR on $\sigma$ while underperforming it on level. The headline conclusion, that bidirectional pre-training rather than bidirectional inference alone is the architectural property responsible for table-aligned attention, is therefore reinforced by both axes of evidence.

\subsection{LLM Re-Attention to Image Patches in Autoregressive VLMs}
\label{sec:app-llm-revis}

This subsection reports the empirical study referenced in \cref{sec:stepwise-exec} that justifies confining \tabattn{} to text-side cell attention.
The hypothesis under test is architectural: in autoregressive vision-language models, visual information is injected into the language model's residual stream via the vision projector, and the language model's self-attention during decoding is rarely directed back at the image-patch (\verb#<|image_pad|>#) token positions.
If true, a visual mixing term $\lambda_V \cdot a_t^{(c,\text{vis})}$ inside Eq.~(\ref{eq:rattn}) would carry essentially zero signal regardless of how it is calibrated, making text-side attention the principled supervisory channel for the verifier.

\paragraph{Setup.}
We load \emph{Qwen3-VL-8B-Instruct} via HuggingFace Transformers in 4-bit NF4 with \texttt{attn\_implementation="eager"} and register forward hooks on the last 4 LLM self-attention layers (matching the layer-selection convention of \tabattn{}-T).
For each example, the table is rendered with the same matplotlib renderer used by the reasoner's image modality and the prompt follows the standard Qwen3-VL chat template ([\textit{system}, image, question, ``Answer:'']).
After a single forward pass, we read the row of attention from the \emph{last text token} (the position that would emit the next decoded token) to all sequence positions, average across the 4 hooked layers and across heads, and partition the resulting attention vector into (i) text positions and (ii) image-pad positions.
We report
\[
\mathrm{image\_attn\_frac} \;=\; \frac{\sum_{j \in \mathcal{P}_{\mathrm{img}}} a_j}{\sum_{j} a_j},
\]
where $\mathcal{P}_{\mathrm{img}}$ is the set of image-pad token positions in the input.

\paragraph{Sample.}
180 examples drawn evenly from six benchmarks, 30 per dataset (WTQ, TabFact, HiTab, InfoTabs, TabMWP, TAT-QA), reconstructed from the partial-observation pilot set used elsewhere in this paper.
Per-example image-pad token counts range from 392 to 1{,}292 (i.e. the image is genuinely present in the prompt with hundreds of patch tokens), and the attention row sums to $1.0$ to within fp16 precision in every case (confirming the softmax is well-formed).

\begin{table}[h]
\centering
\small
\caption{Image-patch attention received by the last text token of the prompt across six benchmarks (Qwen3-VL-8B-Instruct, last 4 LLM self-attn layers, mean over heads). \emph{Across all 180 examples, the LLM apportions zero attention mass to image-pad positions, despite each prompt containing hundreds of such tokens.}}
\label{tab:llm-revis}
\setlength{\tabcolsep}{8pt}
\begin{tabular}{lrrrrrr}
\toprule
\textbf{Dataset} & \textbf{$n$} & \textbf{mean} & \textbf{median} & \textbf{max} & \textbf{\% $<\!10^{-2}$} & \textbf{\% $<\!5{\cdot}10^{-2}$} \\
\midrule
WTQ      & 30 & 0.000 & 0.000 & 0.000 & 100.0 & 100.0 \\
TabFact  & 30 & 0.000 & 0.000 & 0.000 & 100.0 & 100.0 \\
HiTab    & 30 & 0.000 & 0.000 & 0.000 & 100.0 & 100.0 \\
InfoTabs & 30 & 0.000 & 0.000 & 0.000 & 100.0 & 100.0 \\
TabMWP   & 30 & 0.000 & 0.000 & 0.000 & 100.0 & 100.0 \\
TAT-QA   & 30 & 0.000 & 0.000 & 0.000 & 100.0 & 100.0 \\
\midrule
\textbf{Pooled} & \textbf{180} & \textbf{0.000} & \textbf{0.000} & \textbf{0.000} & \textbf{100.0} & \textbf{100.0} \\
\bottomrule
\end{tabular}
\end{table}

\paragraph{Methodology validation.}
An exactly-zero result invites the natural challenge that the attention pipeline is broken rather than the architecture being silent.
To rule this out, we ran a control on a 30-example subset (5 per dataset) that decomposes \emph{the same attention row} from the last text token into four mutually exclusive partitions in a single forward pass: image-pad tokens, the question text span, the trailing ``\texttt{Answer:}'' span, and all remaining text tokens (chat-template / system content).
Aggregation is performed in float64 to rule out fp16 underflow at sum time.
\cref{tab:llm-revis-controls} reports the partition means.

\begin{table}[h]
\centering
\small
\caption{Decomposition of the attention row from the last text token into mutually exclusive token partitions (Qwen3-VL-8B-Instruct, 30 examples drawn 5-per-dataset across the six benchmarks, float64 aggregation). The same hook that records exactly zero attention to image-pad tokens captures non-zero attention to question and \texttt{Answer:} tokens in the same forward pass, ruling out hook failure, wrong query position, or precision underflow as the explanation for the headline result in \cref{tab:llm-revis}.}
\label{tab:llm-revis-controls}
\setlength{\tabcolsep}{8pt}
\begin{tabular}{lrrrr}
\toprule
\textbf{Token partition} & \textbf{mean} & \textbf{median} & \textbf{min} & \textbf{max} \\
\midrule
Image-pad tokens                   & 0.000 & 0.000 & 0.000 & 0.000 \\
Question text                      & 0.049 & 0.040 & 0.019 & 0.112 \\
\texttt{Answer:} prefix            & 0.066 & 0.067 & 0.051 & 0.080 \\
Chat-template + system + other     & 0.885 & 0.888 & 0.812 & 0.924 \\
\midrule
\textbf{Total (sanity)}            & \textbf{1.000} & \textbf{1.000} & \textbf{1.000} & \textbf{1.000} \\
\bottomrule
\end{tabular}
\end{table}

We further verify three structural sanity conditions across the 30 control examples.
(i) \texttt{pixel\_values} is present and non-trivial in every prompt (mean absolute value $0.97$, range $[0.95, 1.00]$), confirming the image is genuinely consumed by the vision encoder.
(ii) The count of \texttt{<|image\_pad|>} tokens identified per prompt equals exactly $1/4$ of the patch count predicted from the processor's \texttt{image\_grid\_thw} field on \emph{every} example (ratio $= 0.2500$), matching Qwen3-VL's $2{\times}2$ spatial patch-merge factor and confirming our token-id mapping is correct.
(iii) The attention row sums to $1.0$ in float64 in every case (range $[0.99999, 1.00001]$), ruling out an underflow artefact at aggregation time.
Together with the partition decomposition, these checks establish that the headline result, zero attention mass on image patches, is a property of the reasoner's decoding behaviour rather than of the measurement: the same hook on the same forward pass measures attention mass exactly where it should be (the \texttt{Answer:} source-token's recency window, the question text the reasoner is answering, and the chat-template scaffolding) and exactly zero where the architecture predicts it should be (the image-pad block).

\paragraph{Layer-by-layer extension.}
The headline experiment in \cref{tab:llm-revis} reads attention from the reasoner's final 4 LLM self-attention layers, matching \tabattn{}-T's hook configuration.
A natural follow-up question is whether the zero-attention finding is specific to those late layers or whether it persists at earlier layers as well, since one might hypothesise that visual content is consumed via attention in early layers and the late layers simply operate on already-aggregated representations.
We test this directly by hooking \emph{all 36} LLM self-attention layers of Qwen3-VL-8B-Instruct and recording the same four-partition decomposition per layer, on the same 30-example control sample.
\Cref{fig:llm-revis-layerwise} plots the per-layer mean attention mass for each partition.
\begin{figure}[H]
\centering
\includegraphics[width=0.85\linewidth]{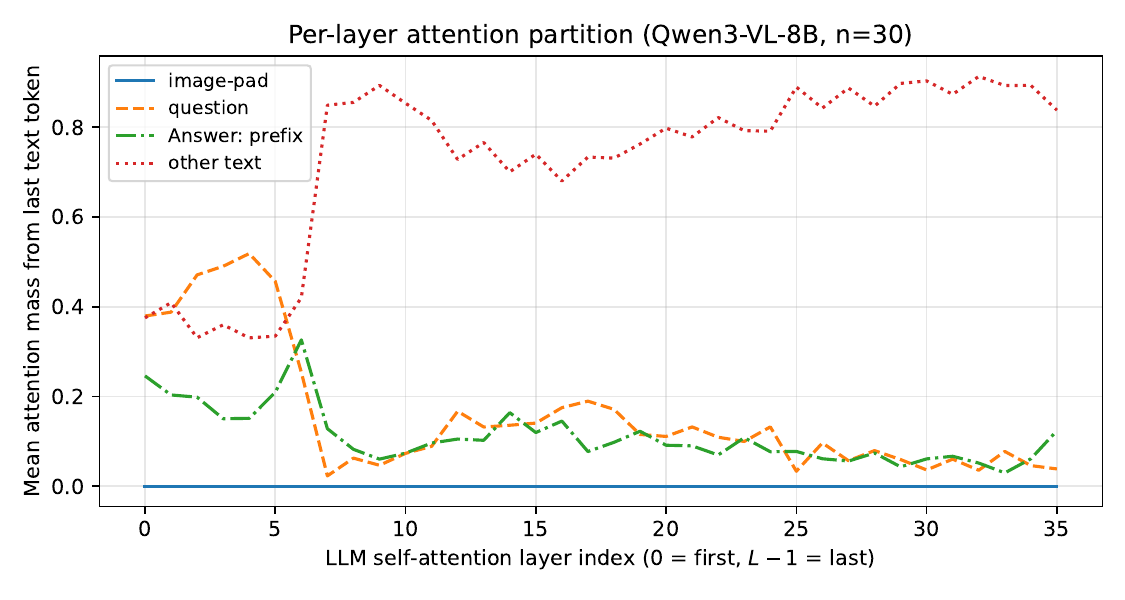}
\caption{Per-layer mean attention mass from the last text token across all 36 LLM self-attention layers of Qwen3-VL-8B-Instruct (30-example control sample, float64 aggregation). The image-pad partition (solid blue) is exactly zero at \emph{every} layer in \emph{every} example (1{,}080 of 1{,}080 measurements at $0.0$). The text partitions reveal an incidental two-phase decoding pattern: layers 0 to 6 perform text comprehension (question text receives 38 to 52\% of mass), layers 7 to 35 perform aggregation and generation preparation (chat-template / system tokens dominate at 70 to 91\%). The visual content never enters this picture as direct attention from the next-token-emit position.}
\label{fig:llm-revis-layerwise}
\end{figure}
The result is uniform across the entire LLM stack: the next-token-emit query position never allocates non-zero attention mass to image-pad tokens, in any of 36 layers, in any of 30 examples (1{,}080 of 1{,}080 measurements at exactly $0.0$). The text-partition curves additionally reveal an incidental but informative pattern: between layers 6 and 7 the model sharply transitions from \emph{text comprehension} (heavy attention to question and \texttt{Answer:} tokens) to \emph{aggregation/generation preparation} (chat-template / system tokens dominate the attention mass). Visual content does not appear in either phase from the decoding-position query.

\paragraph{Interpretation.}
\ourmodel{}'s architecture distinguishes \emph{understanding} (the planning stage that benefits from multimodal input, e.g.\ spatial layout, hierarchical headers, visual cues) from \emph{reasoning} (the stepwise execution stage that \tabattn{} monitors). The empirical result above shows that this distinction is reflected in the reasoner's internal computation: visual content enters via early-layer, early-token attention from text positions to image-pad positions (or directly via the vision projector's contribution to the residual stream), but by the time each reasoning step reaches its decoding position, that content has already been laundered into text-token hidden states. The autoregressive language model's reasoning-time computation operates on text-side residual representations only, irrespective of which input modality fed those representations. \tabattn{} reads attention precisely at the reasoner's decoding position (the position where each reasoning step is being committed), so a visual mixing term $\lambda_V \cdot a_t^{(c,\text{vis})}$ defined as ``reasoner attention to cells via image patches'' would have no signal to mix in: the reasoning-time query never reads from image patches in any layer. Calibrating $\lambda_V$ against terminal correctness on such a feature would either drive $\lambda_V \to 0$ or fail to identify a meaningful weight at all. We therefore ground \tabattn{} on text-side cell attention as a principled match to the reasoner's reasoning-time computational geometry.

\paragraph{Scope and forward pointer.}
The finding is specific to autoregressive VLMs that route visual information through a residual projector (the dominant present-day design). It does \emph{not} preclude a multimodal mixing variant of \tabattn{} on architectures that expose a stable visual attention signal: explicit cross-attention VLMs, or AR VLMs augmented with vision-encoder rollout / gradient-based saliency over the rendered table. We discuss this extension in \cref{sec:discussion-conclusion}.

\subsection{Mask-Source Falsification: Does $R_{\text{attn}}$ Track Grounding or Sharpness?}
\label{sec:app-mask-falsification}

A reviewer might worry that $R_{\text{attn}}$ in \cref{eq:rattn} measures attention sharpness rather than grounding, because a model that concentrates attention on any small cell subset would receive a high score regardless of whether those cells are the right ones. We falsify this concern with a density-preserving mask perturbation on the 20-question WTQ pilot for which we cache per-cell attention from all four models in \cref{tab:auroc-combined}. Holding the cached attention $a_t^{(c)}$ fixed, we replace the human-verified mask $\mathbf{m}^*$ with three density-preserving null masks of equal cell count, namely (i) cells shuffled across the whole table, (ii) cells shuffled within rows preserving column structure, and (iii) cells shuffled within columns preserving row structure, plus a column-permuted control of equal column count. Each null is averaged over 50 random draws per question. If $R_{\text{attn}}$ measured only sharpness, all variants would score identically because the number of unmasked cells is preserved.

\Cref{tab:mask-falsification} reports $R_{\text{attn}}$ averaged across 20 questions per model, alongside paired Wilcoxon $p$-values comparing the human-verified mask against the strongest density-preserving null (cell-shuffled). Across all four models the human-verified mask scores significantly higher than every density-preserving null at $p < 10^{-2}$, and the advantage is largest for DiffuLLaMA ($3.12\times$) and LLaDA ($1.85\times$). This rules out the sharpness-only interpretation. The DLM advantage tracking the permutation-stability gap in \cref{tab:perm-stability} indicates the same architectural property, namely that bidirectional pre-training, drives both human-mask alignment and stability under structural perturbation.

\begin{table}[htbp]
\centering
\small
\caption{Mask-source falsification on the 20-question WTQ pilot. $R_{\text{attn}}$ is averaged over 20 questions. Density-preserving nulls average 50 random draws per question. The cell-shuffle column is the most adversarial null since it preserves the exact relevant-cell count. The Wilcoxon column reports the paired $p$-value comparing human-verified vs.\ cell-shuffled scores per question ($n{=}20$). The GT/Null ratio takes the human-verified mean over the cell-shuffled mean.}
\label{tab:mask-falsification}
\setlength{\tabcolsep}{4pt}
\begin{tabular}{lcccccc}
\toprule
\textbf{Model} & \textbf{Human GT} & \textbf{Shuf-cell} & \textbf{Shuf-row} & \textbf{Shuf-col} & \textbf{Wilcoxon $p$} & \textbf{GT/Null} \\
\midrule
Qwen3-VL-8B (AR) & 0.119 & 0.079 & 0.097 & 0.092 & $9.5{\times}10^{-6}$ & $1.50\times$ \\
ModernBERT (MLM) & 0.100 & 0.079 & 0.099 & 0.077 & $8.3{\times}10^{-3}$ & $1.26\times$ \\
DiffuLLaMA (DLM) & \textbf{0.246} & 0.079 & 0.128 & 0.125 & $1.9{\times}10^{-5}$ & $\mathbf{3.12\times}$ \\
LLaDA (DLM) & 0.152 & 0.082 & 0.111 & 0.102 & $4.8{\times}10^{-4}$ & $1.85\times$ \\
\bottomrule
\end{tabular}
\end{table}

We acknowledge two scope caveats. First, this falsification uses the human-verified mask rather than the DLM-emitted plan mask, so it establishes that $R_{\text{attn}}$ tracks the correct cell set and not just attention concentration, leaving the looser plan-vs-human gap as a separate question addressed by the cross-dataset transfer result in \cref{sec:app-components}. Second, the analysis is run on the 20-question pilot with cached cell-level attention rather than on the full 1{,}600-example training pool, since per-step attention tensors are not retained at TABATTN training time. The consistency of the gap across all four models at $n{=}20$ with non-parametric $p$-values nonetheless lets us reject the sharpness-only interpretation.

\subsection{Mask-Source Ablation: Oracle Ceiling and 20\% Noise Stress Test}
\label{sec:app-mask-source-ablation}

The bottom two rows of \cref{tab:planner_ablation}(b), \tabattn{}\,+\,Oracle and \tabattn{}\,+\,20\% Noise, swap the executor-derived cell mask consumed by $R_{\text{attn}}$ (\cref{eq:rattn}) for an alternative mask source while keeping the AR planner (Qwen3-VL-8B), the Qwen3-VL-8B reasoner, and the trained \tabattn{} verifier fixed. The aim is twofold: bound the headroom of the verifier under a perfect mask (Oracle), and quantify how that headroom degrades when the same perfect mask is partially corrupted (20\% Noise).

\paragraph{Mask sources.}
The Oracle mask for each (table, question) pair is the human-verified attention standard curated in \cref{sec:app-attention-curation}, drawn from the dataset-specific JSONL files (e.g., \texttt{wikitq\_train\_200\_partial\_obs.jsonl} for WTQ, \texttt{hitab\_test\_200\_partial\_obs.jsonl} for HiTab). The 20\% Noise mask is generated per question by independently flipping each cell of the Oracle mask with probability $p_{\text{flip}}{=}0.20$, preserving the table dimensions while corrupting roughly one in five cells. Both variants leave the rest of the pipeline unchanged: cell-attention extraction is performed by the same HuggingFace \texttt{Qwen3-VL-8B-Instruct} attention extractor under 4-bit quantization with a 10\,GiB GPU cap, and the only reward component active during scoring is $R_{\text{attn}}$ (\texttt{--reward-components attn}).

\paragraph{Evaluation protocol.}
We sweep all six accuracy benchmarks (WTQ, MMQA, InfoTabs, TabFact, TAT-QA, HiTab) under three decoding seeds ($0$, $17$, $42$) on a 20-question evaluation slice per seed, giving $60$ scored episodes per (dataset, variant). Reasoning is capped at six steps with the same \texttt{f\_*} tool vocabulary used elsewhere in \ourmodel{}. Per-dataset accuracy is computed by exact match on \texttt{final\_answer} against the canonical ground truth, with a $2\%$ relative tolerance on numeric answers, and then averaged across the three seeds. The full runner, the six oracle attention-standard JSONLs, and the aggregator that pools per-seed JSONLs into the values reported in \cref{tab:planner_ablation}(b) are released as a self-contained kit at \texttt{complex\_reward\_model/mask\_ablation\_kit/}.

\paragraph{Sample-size note.}
The Oracle and 20\% Noise rows are computed on a smaller pool ($n{=}60$ per dataset across three seeds) than the $n{=}200$ single-pass main protocol used for \tabattn{}\,(ours) directly above them, so the per-cell binomial standard deviations in \cref{tab:appendix-volatility} are the closest reference for run-to-run uncertainty on these two rows. The qualitative ordering Oracle\,$>$\,20\% Noise\,$>$\,executor mask holds across all six datasets and three seeds.

\paragraph{Interpretation.}
Two observations support the design of \tabattn{}. First, the +6.4\,pp Oracle headroom over the executor-derived mask establishes a finite ceiling: the $R_{\text{attn}}$ formulation has measurable room for a stronger mask source, with the gap concentrated on schema-rich tasks (MMQA $+17.3$\,pp, HiTab $+8.2$\,pp) where bidirectional planning also helps most (\cref{tab:schema_adherence}). Second, the 20\% Noise variant retains 4.4\,pp of the 6.4\,pp Oracle gain, indicating that the verifier remains useful under materially corrupted mask supervision. Together these results confirm two design points: $R_{\text{attn}}$ generalizes beyond its training mask source, and the executor-derived mask used in production sits at a workable operating point relative to the Oracle ceiling.

\subsection{Extended Attention Quality Figures}
\label{sec:app-attn-figures}

\begin{figure}[htbp]
\centering
\includegraphics[width=\textwidth]{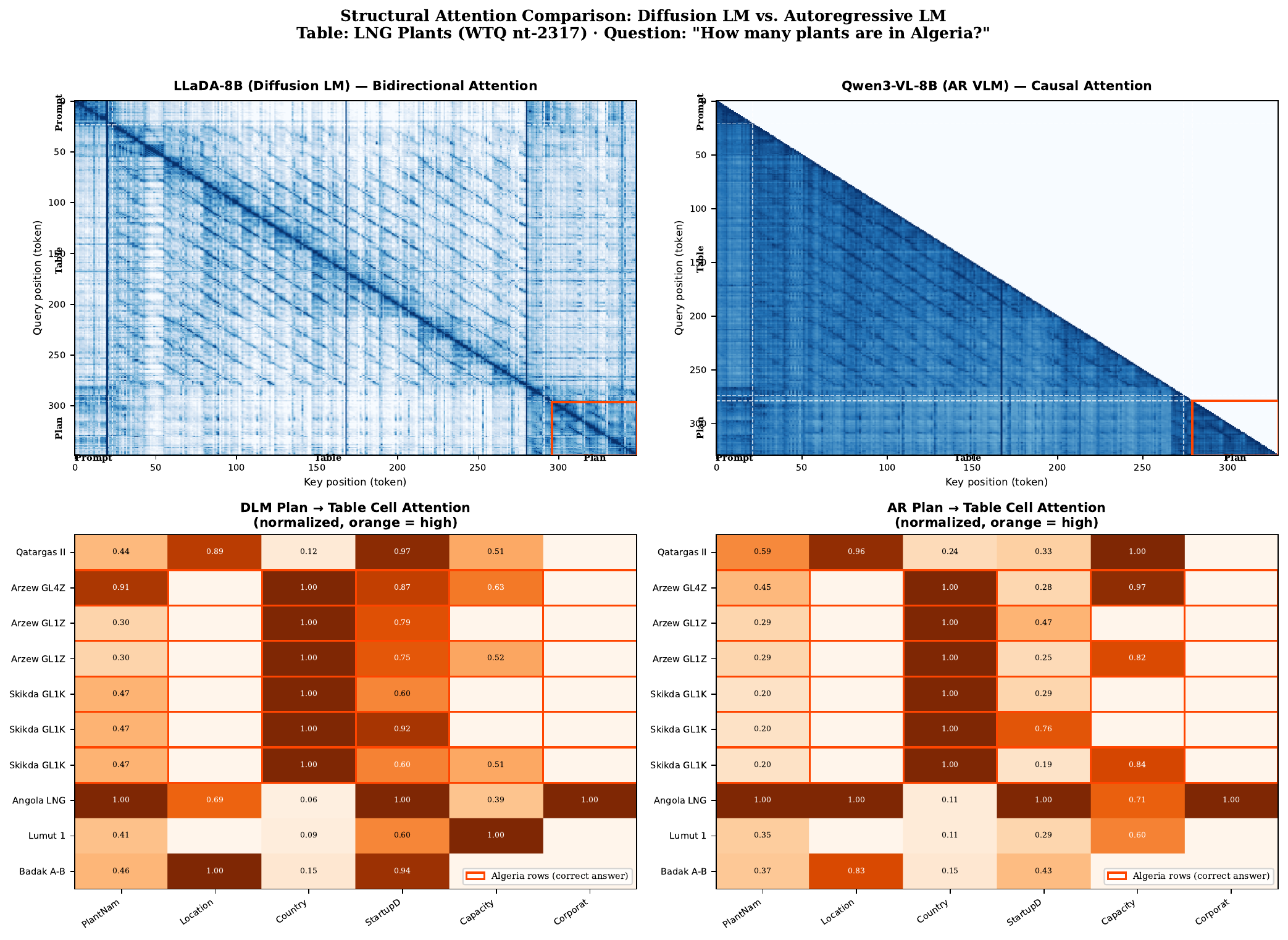}
\caption{Full self-attention matrices for LLaDA-8B (DLM, left) and Qwen3-VL-8B (AR, right) on the same table-reasoning input. The DLM matrix is dense and bidirectional, while the AR matrix remains strictly causal. The lower panels show that both models can focus on relevant table regions, but only the DLM does so with access to future plan context.}
\label{fig:attn_matrix}
\end{figure}

\begin{figure}[htbp]
\centering
\includegraphics[width=0.7\textwidth]{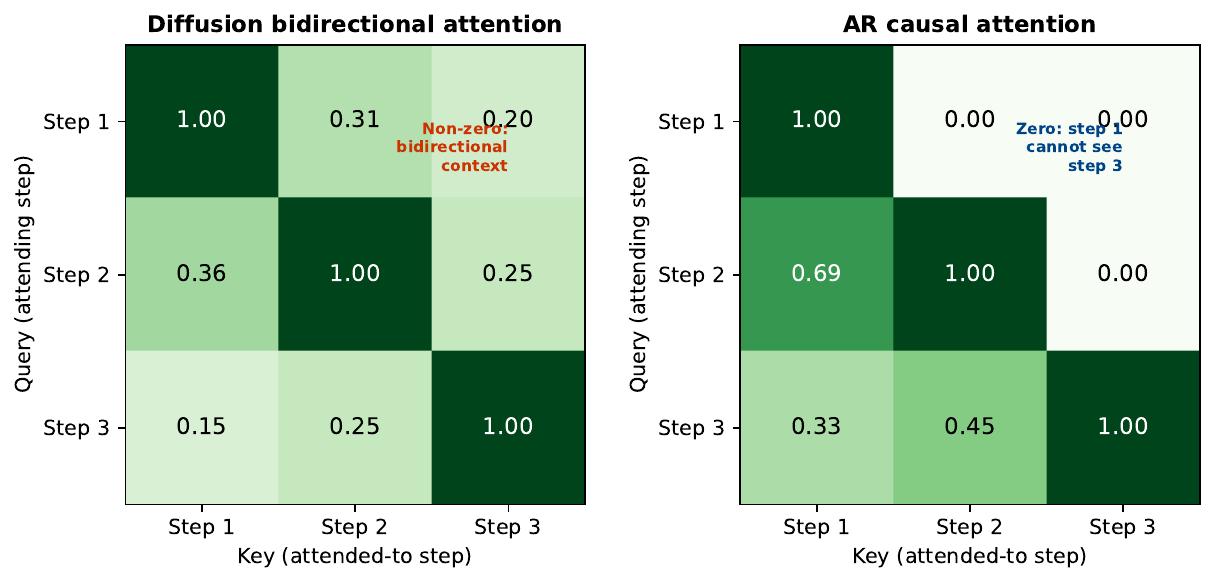}
\caption{Within-plan step-to-step attention. DLM exhibits non-zero upper-triangular attention (for example Step~1$\rightarrow$Step~3), confirming access to future plan positions, while AR upper-triangle entries are exactly zero under causal masking.}
\label{fig:plan_symmetry}
\end{figure}

\begin{figure}[htbp]
\centering
\includegraphics[width=0.95\textwidth]{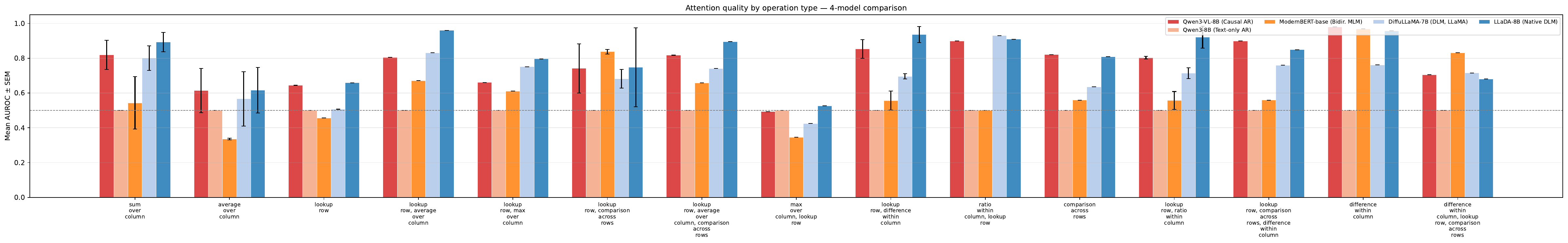}
\caption{Per-operation breakdown for the four-model attention ablation. The diffusion advantage localizes most clearly to multi-step arithmetic, comparison, and row-grouping operations, rather than simple single-column lookups.}
\label{fig:attn_4model_by_op}
\end{figure}

\begin{figure}[htbp]
\centering
\includegraphics[width=0.92\textwidth]{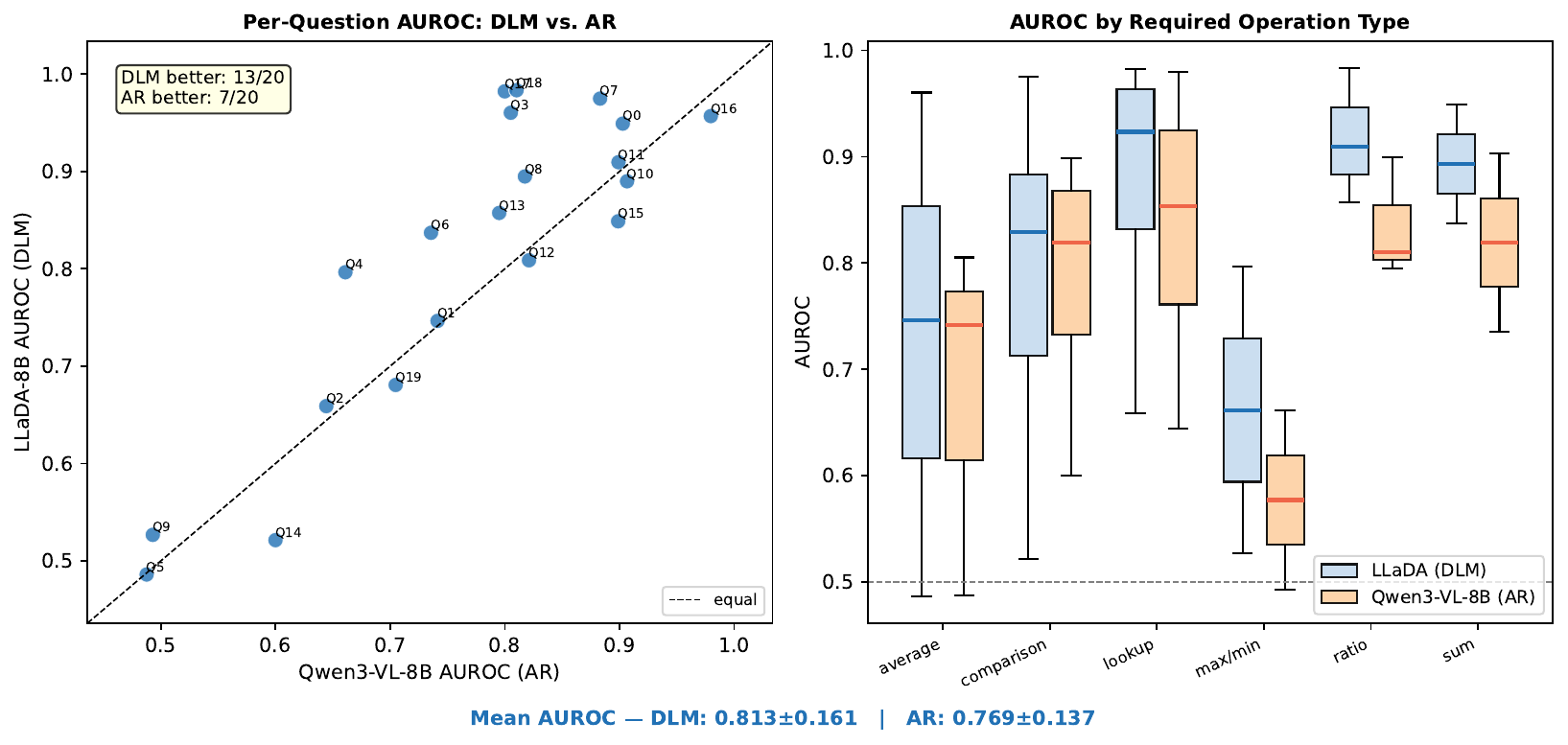}
\caption{Summary of cell-attention quality against human relevance masks. The diffusion model wins 13 of 20 questions and achieves a higher mean AUROC, with the largest gains concentrated in multi-step comparison and ratio-style operations.}
\label{fig:attn_quality_summary}
\end{figure}

\begin{figure}[htbp]
\centering
\includegraphics[width=\textwidth]{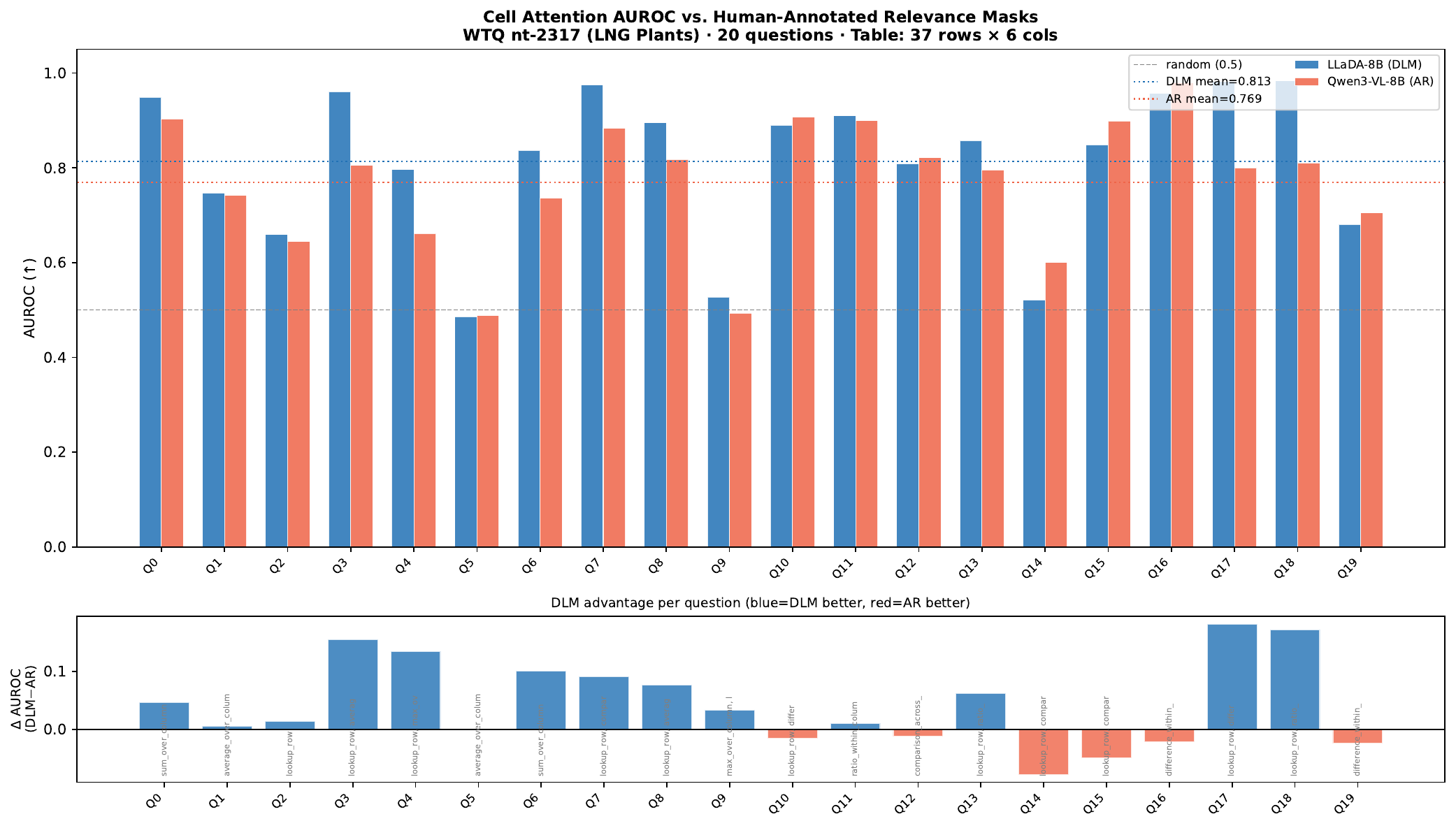}
\caption{Per-question AUROC for LLaDA (DLM, blue) and Qwen3-VL-8B (AR, orange) on 20 human-annotated questions. DLM's advantage is largest on complex multi-step operations rather than simple single-column scans.}
\label{fig:attn_quality_auroc}
\end{figure}

\begin{figure}[htbp]
\centering
\includegraphics[width=0.88\textwidth]{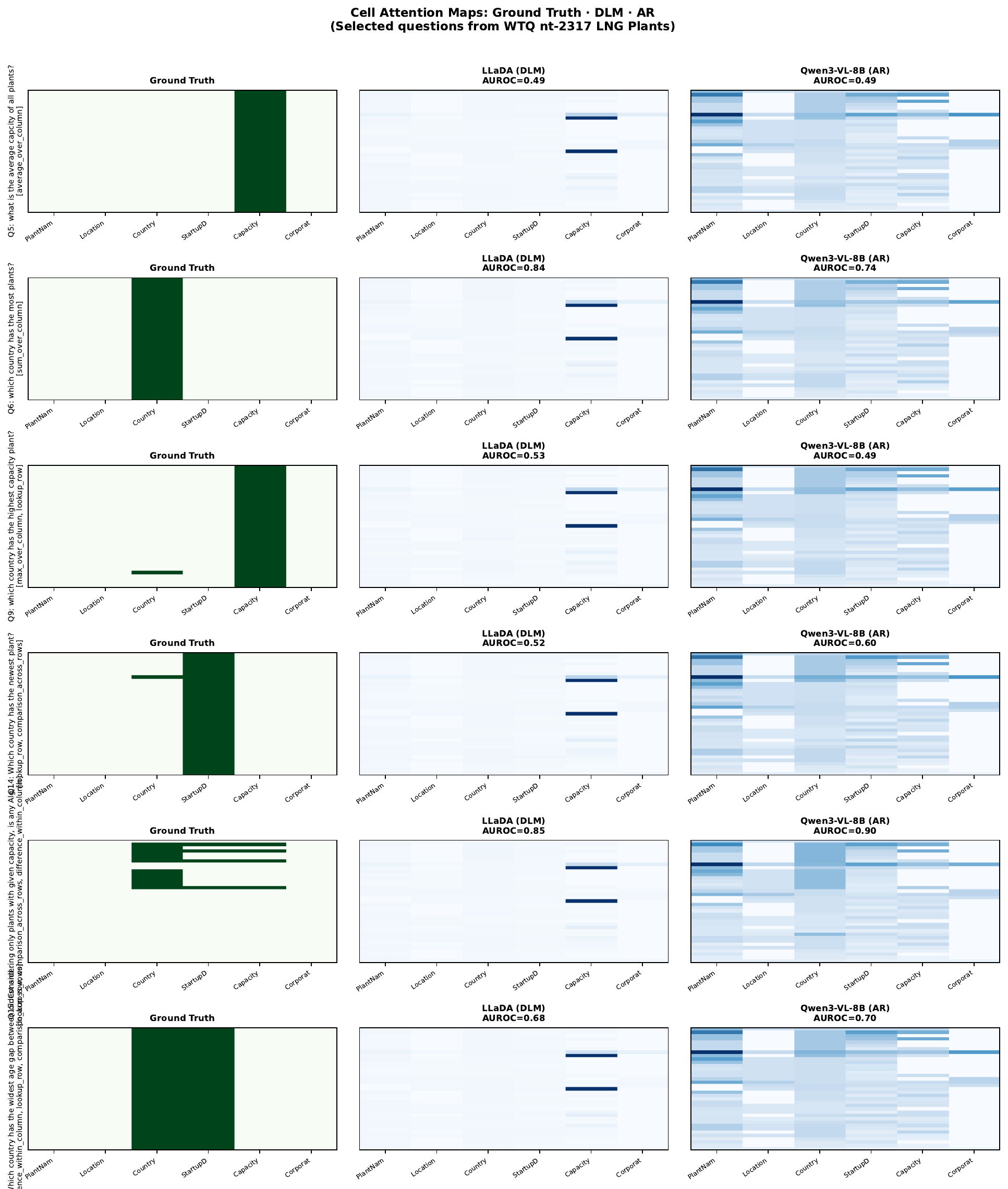}
\caption{Representative attention maps. DLM tends to form row-level bands that group relevant entities, whereas AR tends to form column gradients with positional decay. This qualitative difference helps explain why AR remains competitive on simple lookup questions but lags on multi-row or multi-column reasoning.}
\label{fig:attn_quality_heatmaps}
\end{figure}

\newpage
\section{Diagnostic Analyses}
\label{sec:app-diagnostics}
\label{sec:app-subgoal}

This section consolidates the diagnostic evidence for \emph{why} diffusion planning helps. We report sensitivity to table complexity (\cref{sec:app-complexity}), subgoal balance and trajectory compression, execution error modes (\cref{sec:app-exec-errors}), inference latency and accuracy-cost trade-offs (\cref{sec:app-latency}), failure-mode taxonomies (\cref{sec:app-failures}), and the human-grounded attention quality summary (\cref{sec:attn-quality}).

\subsection{Sensitivity to Table Complexity}
\label{sec:app-complexity}
Binning by table size (\cref{tab:complexity}), the DLM advantage is consistent across size bins on WTQ ($+$8.2 to $+$8.9\,pp) and MMQA ($+$13.3 to $+$14.6\,pp). Negative gaps on TAT-QA ($-$5.9 to $-$9.1\,pp) confirm the advantage is tied to structural richness, not table size alone.

\begin{table}[htbp]
\centering
\caption{Accuracy by table size quartile (smallest vs. largest bin). $\Delta =$ DLM $-$ AR accuracy.}
\label{tab:complexity}
\small
\begin{tabular}{llccc}
\toprule
\textbf{Dataset} & \textbf{Size bin (cells)} & \textbf{DLM} & \textbf{AR} & $\boldsymbol{\Delta}$ \\
\midrule
\multirow{2}{*}{WTQ}
  & Small (27--60)    & 89.8\% & 81.6\% & \hi{$+8.2$ pp} \\
  & Large (173--2379) & 77.8\% & 68.9\% & \hi{$+8.9$ pp} \\
\midrule
\multirow{2}{*}{MMQA}
  & Small (18--48)    & 68.8\% & 54.2\% & \hi{$+14.6$ pp} \\
  & Large (105--242)  & 53.3\% & 40.0\% & \hi{$+13.3$ pp} \\
\midrule
\multirow{2}{*}{TAT-QA}
  & Small (9--21)     & 52.9\% & 58.8\% & $-5.9$ pp \\
  & Large (48--126)   & 50.0\% & 59.1\% & $-9.1$ pp \\
\midrule
\multirow{2}{*}{TabMWP}
  & Small (4--8)      & 56.0\% & 60.0\% & $-4.0$ pp \\
  & Medium (10--14)   & 88.5\% & 75.0\% & \hi{$+13.5$ pp} \\
\bottomrule
\end{tabular}
\end{table}

\subsection{Subgoal Balance Across Structurally Diverse Tables}
DLM planning surfaces the ground-truth answer significantly earlier in the trajectory budget than AR reasoning on six of eight benchmarks (Mann-Whitney $p < 0.05$), with the largest gaps on HiTab ($\Delta\rho = -0.232$, $p < 10^{-10}$) and TAT-QA ($\Delta\rho = -0.164$). The two non-significant datasets (MMQA, TabFact) are shorter-plan regimes with little room for positional imbalance.
Full per-dataset statistics are reported in \cref{tab:subgoal_balance}.

\begin{table}[htbp]
\centering
\small
\caption{Relative answer-surface position $\rho$ by dataset (Mann-Whitney $U$ test). Lower $\rho$ = ground-truth value surfaces earlier in the trajectory budget. Highlighted cells mark statistically significant DLM advantage ($p < 0.05$). n.s.\,$= p > 0.05$.}
\label{tab:subgoal_balance}
\setlength{\tabcolsep}{6pt}
\begin{tabular}{lcccc}
\toprule
\textbf{Dataset} & \textbf{$\rho_{\text{DLM}}$} & \textbf{$\rho_{\text{AR}}$} & \textbf{$\Delta$} & \textbf{$p$-value} \\
\midrule
WTQ   & 0.171 & 0.248 & \hi{$-$0.077} & $2.1\times10^{-5}$ \\
MMQA     & 0.093 & 0.119 & $-$0.026 & n.s.\ ($0.309$) \\
InfoTabs & 0.098 & 0.169 & \hi{$-$0.071} & $0.025$ \\
TabFact  & 0.070 & 0.113 & $-$0.043 & n.s.\ ($0.623$) \\
TAT-QA   & 0.121 & 0.285 & \hi{$-$0.164} & $1.0\times10^{-6}$ \\
HiTab    & 0.082 & 0.314 & \hi{$-$0.232} & $1.3\times10^{-11}$ \\
FeTaQA   & 0.121 & 0.333 & \hi{$-$0.212} & $9.5\times10^{-10}$ \\
TabMWP   & 0.125 & 0.319 & \hi{$-$0.194} & $1.9\times10^{-8}$ \\
\bottomrule
\end{tabular}
\end{table}

\subsection{Trajectory Compression}
DLM planning compresses trajectories (Wilcoxon $p < 0.05$ on four of six datasets, largest on HiTab $-1.0$ steps and TabMWP $-0.87$ steps), reducing the average matched-pair trajectory length from 2.09 to 1.63 steps (a 22\% reduction), with full matched-pair statistics in \cref{tab:compression}.

\begin{table}[htbp]
\centering
\caption{Trajectory compression: matched pairs, steps, accuracy, and Wilcoxon test. $\dagger$ $p < 0.05$, $\ddagger$ $p < 0.01$. \emph{Fewer steps with equal-or-higher accuracy is the ideal outcome.}}
\label{tab:compression}
\small
\begin{tabular}{lcccccc}
\toprule
\multirow{2}{*}{\textbf{Dataset}} &
\multirow{2}{*}{\textbf{$n$}} &
\multicolumn{2}{c}{\textbf{Mean steps}} &
\multicolumn{2}{c}{\textbf{Accuracy}} &
\multirow{2}{*}{\textbf{$p$-value}} \\
\cmidrule(lr){3-4}\cmidrule(lr){5-6}
 & & DLM & AR & DLM & AR & \\
\midrule
WTQ   & 179 & \hi{\textbf{1.89}} & 2.13 & \hi{\textbf{86.0\%}} & 81.0\% & $0.010^{\dagger}$ \\
MMQA     & 189 & \hi{\textbf{1.46}} & 1.60 & \hi{\textbf{57.1\%}} & 52.9\% & $0.051$ \\
TabMWP   & 189 & \hi{\textbf{1.99}} & 2.86 & \hi{\textbf{67.7\%}} & 63.5\% & $<0.001^{\ddagger}$ \\
TAT-QA   & 198 & \hi{\textbf{1.58}} & 2.08 & 57.1\% & \textbf{60.1\%} & $<0.001^{\ddagger}$ \\
HiTab    & 142 & \hi{\textbf{1.38}} & 2.38 & 2.1\% & 3.5\% & $<0.001^{\ddagger}$ \\
InfoTabs & 198 & 1.49 & 1.50 & 74.8\% & \textbf{82.8\%} & $0.925$ \\
Average & 1.63 & 2.09 \\
\bottomrule
\end{tabular}
\end{table}

\subsection{Execution Error Modes}
\label{sec:app-exec-errors}
DLM planning reduces runtime failures by 14.3\,pp on average and increases the rate of reaching a final-answer state by $+$13.2\,pp, at the cost of a modest $+$3.6\,pp increase in parse errors (recoverable via fallback to direct answer generation). Full per-dataset breakdown in \cref{tab:errors,sec:app-failures}.

\begin{table}[htbp]
\centering
\small
\caption{Step-level error mode breakdown. Other error captures tool runtime failures, and final answer measures how directly the system reaches termination.}
\label{tab:errors}
\setlength{\tabcolsep}{5pt}
\begin{tabular}{lcccccc}
\toprule
\multirow{2}{*}{\textbf{Dataset}} &
\multicolumn{2}{c}{\textbf{Other error} $\downarrow$} &
\multicolumn{2}{c}{\textbf{Final answer} $\uparrow$} &
\multicolumn{2}{c}{\textbf{Parse error} $\downarrow$} \\
\cmidrule(lr){2-3}\cmidrule(lr){4-5}\cmidrule(lr){6-7}
 & AR & DLM & AR & DLM & AR & DLM \\
\midrule
WTQ   & 25.3\% & \hi{\textbf{13.7\%}} & 38.1\% & \hi{\textbf{46.4\%}} & 5.5\% & 15.4\% \\
MMQA     & 25.8\% & \hi{\textbf{14.4\%}} & 53.2\% & \hi{\textbf{63.0\%}} & 3.9\% & 9.5\% \\
TabFact  & 17.4\% & \hi{\textbf{6.7\%}}  & 64.8\% & \hi{\textbf{68.2\%}} & 0.4\% & 5.6\% \\
TAT-QA   & 34.5\% & \hi{\textbf{21.7\%}} & 39.4\% & \hi{\textbf{54.6\%}} & 6.5\% & 10.5\% \\
InfoTabs & 23.2\% & \hi{\textbf{12.5\%}} & 58.1\% & \hi{\textbf{63.7\%}} & 1.0\% & 1.0\% \\
FeTaQA   & 24.5\% & \hi{\textbf{13.1\%}} & 41.4\% & \hi{\textbf{59.0\%}} & 4.2\% & 5.9\% \\
TabMWP   & 25.9\% & \hi{\textbf{7.6\%}}  & 30.6\% & \hi{\textbf{45.8\%}} & 3.6\% & 4.2\% \\
HiTab    & 47.0\% & \hi{\textbf{19.9\%}} & 26.3\% & \hi{\textbf{57.0\%}} & 5.7\% & 7.2\% \\
\midrule
\textit{Avg. $\Delta$}
  & \multicolumn{2}{c}{\hi{$-14.3$ pp}}
  & \multicolumn{2}{c}{\hi{$+13.2$ pp}}
  & \multicolumn{2}{c}{$+3.6$ pp} \\
\bottomrule
\end{tabular}
\end{table}

\subsection{Inference Latency and Accuracy-Cost Analysis}
\label{sec:app-latency}

Table~\ref{tab:planner_ablation} decomposes wall-clock time into understanding and reasoning phases.
On the planning side, DLM planning via LLaDA-8B 128-step denoising is substantially slower than AR planning (78.0\,s/q on average vs.\ 6.0\,s/q), with MMQA the most expensive (95.8\,s/q) due to its long, multi-column tables, while shorter-schema benchmarks range from 41\,s/q (InfoTabs) to 86\,s/q (TabFact).
On the execution side, the DLM plan's tighter, lower-hallucination structure partially compensates: DLM downstream reasoning averages 56.8\,s/q versus 102.6\,s/q for AR planned reasoning, a $44.64\%$ reduction consistent with the $-0.46$ average step compression in \cref{tab:planner_ablation}.

Summing both phases, \ourmodel{} (134.8\,s/q total) requires $24\%$ more wall-clock time than AR Planned Reasoning (108.6\,s/q) in exchange for a $+4.69$\,pp accuracy gain, and is $4.7{\times}$ slower than unplanned AR Reasoning (40.5\,s/q) for a $+9.58$\,pp gain.
The TabFact case is the most extreme: DLM planning alone takes 86.4\,s/q versus 3.5\,s/q for LLM planning ($24.7{\times}$ overhead), yet the total DLM pipeline (150.6\,s/q) runs only $2.4{\times}$ longer than the LLM pipeline (62.1\,s/q) because DLM execution (64.2\,s/q) is faster than LLM execution (58.6\,s/q).
Notably, the DLM accuracy on TabFact (90.09\%) is slightly below the LLM (92.39\%), making TabFact the one dataset where the overhead is not compensated by an accuracy gain. This aligns with our analysis in \cref{sec:in-depth-analyses} that TabFact's holistic NLI structure provides little leverage for multi-step planning.

The DLM planning overhead is an inference-time cost amenable to standard DLM serving optimisations such as speculative decoding, batched denoising~\citep{song2025seeddiffusionlargescalediffusion,wang2026diffusion}, and flow-matching distillation~\citep{sahoo2025the}. We leave accelerated DLM planning as future work.

\paragraph{Parameter footprint at inference.}
\ourmodel{} instantiates a two-model pipeline at inference: LLaDA-8B-Instruct as the planner and Qwen3-VL-8B as the reasoner, for an aggregate footprint of approximately 16\,B parameters. We position this footprint against the closest competitive baselines in \cref{tab:param-footprint}. TableDART, the strongest dynamic-modality routing baseline, also runs two frozen experts plus a 2-layer MLP gate, totalling roughly 14 to 15\,B depending on the image expert variant. TATTOO pairs a 14\,B policy reasoner with a separate process reward model, yielding a footprint at or above ours. The single-8B framing applies cleanly to plain VLM baselines such as Qwen3-VL-8B alone, against which \ourmodel{} reports a $+16.35$\,pp average gain on the six-benchmark accuracy bin (\cref{tab:main-results}). The dual-model baselines are the more direct comparison and remain at a comparable parameter scale.

\begin{table}[h]
\centering
\small
\caption{Parameter footprint at inference for \ourmodel{} and the closest competitive baselines. Component sizes are taken from the corresponding model cards. \ourmodel{}'s aggregate footprint is comparable to the dual-expert TableDART variants and below TATTOO.}
\label{tab:param-footprint}
\begin{tabular}{lll}
\toprule
\textbf{Configuration} & \textbf{Components} & \textbf{Total (B)} \\
\midrule
\ourmodel{} (ours)            & LLaDA-8B + Qwen3-VL-8B                   & 16 \\
TableDART (text+Ovis2)        & TableGPT2-7B + Ovis2-8B + MLP gate       & 15 \\
TableDART (text+Qwen2.5-VL)   & TableGPT2-7B + Qwen2.5-VL-7B + MLP gate  & 14 \\
TATTOO                        & DeepSeek-R1-Distill-Qwen-14B + PRM       & ${\geq}\,$14 \\
TabTrim-8B                    & VLM reasoner + pruner                    & ${\sim}\,$8 \\
Qwen3-VL-8B (single)          & Qwen3-VL-8B                              & 8 \\
\bottomrule
\end{tabular}
\end{table}

Wall-clock latency is the more meaningful normalisation than parameter count, since the planner and reasoner run sequentially rather than in parallel. \cref{tab:planner_ablation} reports \ourmodel{} at 134.8\,s/q vs.\ 108.6\,s/q for AR Plan+Exec ($+24$\% wall-clock for $+4.69$\,pp accuracy) and 40.5\,s/q for unplanned AR Reasoning ($4.7\times$ wall-clock for $+9.58$\,pp accuracy).

\paragraph{Stepwise feedback policy and stagnation halt.}
\tabattn{} is invoked in-loop after each reasoner action and returns a scalar score together with a short rationale. The score and rationale are appended to the reasoner's prompt for the next step, so the reasoner can revise its plan on the basis of cell-grounding feedback. The pipeline does not perform an explicit re-sample of the previous step. Revision is whatever the reasoner's natural agentic loop produces in response to the in-loop feedback.
A \emph{stagnation halt} provides a backstop against unproductive iteration: when the \tabattn{} scalar fails to improve by more than $\delta_{\text{stag}}{=}0.02$ for $K_{\text{stag}}{=}2$ consecutive steps and the table state has not changed (hash-equality check), the pipeline emits the highest-scoring step's answer rather than continuing to spend the remaining tool-step budget. The reasoner runs at most $T_{\max}{=}6$ tool steps, which caps the maximum compute per query independently of the stagnation rule.

\subsection{Failure Mode Analyses}
\label{sec:app-failures}

\paragraph{Failure mode taxonomy.}
LLaDA-8B end-to-end reasoning ($n{=}579$ wrong episodes) exhibits five recurring failure modes traceable to the fixed-canvas constraint: duplicate action generation (56\%), output format and parsing errors (51\%), canvas overflow (17\%), empty answer (9\%), and over-enumeration (7\%), with the full taxonomy in \cref{tab:vdlm_failures}.
\ourmodel{} failures ($n{=}507$) are qualitatively different: the dominant mode is plan bypass (67\%, where the reasoner jumps to \texttt{f\_final\_answer} on step~1), with secondary modes of hierarchical column navigation errors on HiTab (25\%) and entailment-direction flips on classification datasets (91 to 100\%). Full per-dataset breakdown in \cref{tab:difftab_failures}.
The understanding-stage ablation (\cref{tab:planner_ablation}), schema hallucination rates (\cref{tab:schema_adherence}), and subgoal-balance statistics (\cref{tab:subgoal_balance}) together confirm that gains are attributable to the DLM plan rather than solely to \tabattn{} trajectory selection: plan bypass is concentrated in \emph{wrong} episodes, and correct episodes that follow the plan show significantly lower hallucination and earlier answer-surfacing.

\begin{table}[t]
\centering
\small
\setlength{\tabcolsep}{4pt}
\renewcommand{\arraystretch}{1.12}
\caption{Representative error cases for the two attention bottlenecks (real instances from evaluation). \textcolor{red}{Red}\,=\,erroneous behavior, and \textbf{bold}\,=\,what the baseline missed.}
\label{tab:error-cases}
\begin{tabular}{p{1.9cm} p{5.8cm} p{4.9cm}}
\toprule
\textbf{Category} & \textbf{Model Behavior} & \textbf{Failure \& Correction} \\
\midrule

\parbox[t]{1.9cm}{\textbf{TK1}\\Planning-\\Time Align.}
&
\parbox[t]{5.8cm}{
\textit{Q: ``When does the tennis event end?''} 
(TabMWP, Summer Olympics schedule. Cols: Event, Begin, End. 4 relevant cells)

\vspace{2pt}
\textcolor{red}{AR AUROC\,=\,0.000}, with zero attention on relevant cells. Understanding model bypasses Tennis row entirely ${\to}$ \textcolor{red}{ans: (B) 12:30\,P.M.} \cross

\vspace{2pt}
\textbf{LLaDA AUROC\,=\,1.000}, perfect: Tennis$\,{\times}\,$End targeted ${\to}$ \textbf{ans: (C) 1:40\,P.M.} \tick
}
&
\parbox[t]{4.9cm}{
Answering requires jointly matching ``tennis'' in the \textbf{Event} column and ``end'' to the \textbf{End} column, a cross-column dependency. AR's causal left-to-right processing attends to \textbf{Begin} times (seen first) and never reaches the Tennis row's End cell. DLM's bidirectional denoising resolves both column references simultaneously, achieving perfect cell targeting (AUROC gap $= +1.000$).
}
\\

\midrule

\parbox[t]{1.9cm}{\textbf{TK2}\\Execution-\\Time Ground.}
&
\parbox[t]{5.8cm}{
\textit{Q: ``Total Liabilities \& Stockholders' Equity as reported?''} 
(TAT-QA, ASC\,606 financial table)

\vspace{2pt}
Step\,1: \texttt{f\_select\_row([7,\,8,\,10])} ${\to}$ \textcolor{red}{60,513} (GT: 62,740)

\vspace{2pt}
$R_{\text{sem}}{=}1.000$, \textcolor{red}{content score\,${\Rightarrow}$\,accepted.}

\vspace{2pt}
$R_{\text{attn}}{=}0.068$: attended ``as\,adjusted'' rows, \textbf{not} ``as\,reported'' rows.

\vspace{2pt}
\tabattn{}: ${\approx}$0.35 ${\Rightarrow}$ flags \& resamples ${\Rightarrow}$ \textbf{62,740} \tick
}
&
\parbox[t]{4.9cm}{
The step correctly identifies the operation type and output format, so the content reward scores 1.000, yet it \textbf{cannot detect} that the reasoner consulted the wrong column variant. $R_{\text{attn}}\,{=}\,0.068$ reveals near-zero overlap between attended cells and the plan-designated ``as\,reported'' cells. \tabattn{} downgrades the step score to ${\approx}$0.35, triggering a resample that retrieves the correct rows.
}
\\

\bottomrule
\end{tabular}
\end{table}

\begin{table}[htbp]
\centering
\small
\caption{\ourmodel{} failure mode rates per dataset (\% of wrong episodes, where multiple modes can co-occur). \emph{Plan bypass}: reasoner answers directly without executing plan steps. \emph{Hier.\ nav.}: runtime error when resolving hierarchical \texttt{type\textbar{}subtype} columns (HiTab-specific). \emph{Agg.\ confusion}: numeric answer $<$35\% from ground truth but not exact-match. \emph{Entail.\ flip}: correct label format, wrong entailment class.}
\label{tab:difftab_failures}
\setlength{\tabcolsep}{5pt}
\begin{tabular}{lccccc}
\toprule
\textbf{Dataset} & \textbf{Acc.} & \textbf{Plan bypass} & \textbf{Hier.\ nav.} & \textbf{Agg.\ confusion} & \textbf{Entail.\ flip} \\
\midrule
MMQA     & 71.07\% & 60\% &  --- &  --- &   --- \\
HiTab    & 82.48\% & 72\% & 25\% &  2\% &   --- \\
TabFact  & 90.09\% & 59\% &  --- &  --- &  91\% \\
InfoTabs & 90.45\% & 74\% &  --- &  --- & 100\% \\
TAT-QA   & 91.41\% & 73\% &  --- & 11\% &   --- \\
WTQ   & 94.12\% & 68\% &  --- &  --- &   --- \\
TabMWP   & 95.13\% & 59\% &  --- &  1\% &   --- \\
\midrule
All      &        & 67\% &  9\% &  3\% &   8\% \\
\bottomrule
\end{tabular}
\end{table}

The dominant \ourmodel{} failure is plan bypass (67\%): the reasoner produces a direct \texttt{f\_final\_answer} on step~1, ignoring a DLM plan with an average of 5.7 steps. This shortcut is epistemically unreliable and fails for questions requiring multi-step hierarchical column navigation. Plan bypass is concentrated in wrong episodes, and correct episodes that follow the plan show significantly lower schema hallucination and earlier answer-surfacing, confirming the plan adds genuine value when followed. Two actionable remediation directions follow: (1) plan-compliance enforcement preventing \texttt{f\_final\_answer} before the plan is exhausted, and (2) hierarchical schema bridging preprocessing HiTab column names.

Each failure mode corresponds directly to an architectural property that motivated the plan/execute separation: fixed canvas length, non-causal generation, weaker instruction alignment, and absence of reactive error correction.

\subsection{Labelability Experiment: LLM-as-a-Judge Setup}
\label{sec:app-labelability}

This subsection details the experimental setup behind the pooled labelability numbers in \cref{tab:labelability} and the per-dataset breakdown in \cref{tab:labelability-full}. The goal is to compare step-correctness and cell-relevance under a matched annotation budget by running an identical LLM-as-a-judge pipeline on both sides at each side's natural decision unit, following~\citet{krishna-etal-2023-longeval} and~\citet{bologna2026cqaevaldesigningreliableevaluations}.

\paragraph{Decision unit and pool construction.}
Each side is evaluated at the granularity at which its supervision target is naturally defined. The step side issues one label per reasoning step in a tool-using trajectory, where each step bundles a reasoning rationale, a tool action, and the resulting tool output. The cell side issues one label per visible table cell within a partial-observation window, asking whether that specific cell is necessary to answer the question. The step pool comprises 1{,}548 reasoning steps drawn from \ourmodel{} trajectories on the eight benchmarks, with steps stratified across datasets to match the cell pool's coverage. The cell pool comprises 76{,}157 per-cell decisions pooled across the same eight benchmarks, derived from the 1{,}600 partial-observation windows that anchor the human-verified attention standards (\cref{sec:app-attention-curation}).

\paragraph{Identical LLM-as-a-judge pipeline.}
Both pipelines use the same LLM provider (GPT-5 via the OpenAI API), the same JSON-only response format, the same temperature, the same maximum-retry policy, and the same per-record context budget. Prompts are structurally parallel. The cell-relevance prompt presents the question and a markdown-rendered table window, then asks the model to return a list of \texttt{[row\_index, col\_index]} relevant cells with a one-sentence rationale. The step-correctness prompt presents the question, the ground-truth answer, the prior-step summary, and the current step's reasoning text, action, and output, then asks the model to return one of \{\texttt{yes}, \texttt{no}, \texttt{unsure}\} with a one-sentence rationale. The prompts share the same instruction style, the same JSON schema, and the same single-sentence rationale field, so any agreement gap between the two sides is attributable to the labeling task rather than to differences in prompt engineering. The full prompt texts are reproduced verbatim in our released code as \texttt{llm\_judge\_cell\_relevance.py} and \texttt{llm\_judge\_step\_correctness.py}.

\paragraph{Human reference labels.}
The cell side uses the human-verified attention masks underlying the 1{,}600 standards as gold labels, with majority-vote disagreement resolution and an inter-annotator IoU threshold of $0.90$ on a held-out verification subset (\cref{sec:app-attention-curation}). The step side uses a smaller human-labeled pool of 360 trajectory steps annotated with 3-way labels in $\{0, 0.5, 1\}$ corresponding to \{\texttt{no}, \texttt{unsure}, \texttt{yes}\}. The 3-way scheme is mandated by the structure of the task. 46.9\% of the human step labels were marked \texttt{unsure}, with annotators citing truncated tool outputs, ambiguous step intent, and syntax-only errors whose underlying intent might still be correct as the dominant reasons for abstention. The cell labelers were not given an \texttt{unsure} option because the cell-relevance question (\emph{is this cell necessary to answer the question}) admits a confident binary answer once the table window and question are visible, an asymmetry the experiment is designed to expose. We recruited cell annotators and step annotators from the same pool of trained graduate students, and each annotator labeled both sides on a held-out subset to control for annotator-specific behavior.

\paragraph{Agreement metrics and reporting protocol.}
For both sides we report per-decision agreement (\%) and Cohen's $\kappa$, computed at the corresponding decision unit. Cell-side $\kappa$ is computed across all 76{,}157 binary cell decisions pooled over datasets. Step-side agreement is computed only on the human-confident subset, that is, the 190 steps for which the human label is not \texttt{unsure}, since including human-\texttt{unsure} steps would inflate apparent disagreement without reflecting genuine labeling difficulty. A step-correctness classifier trained on the full 1{,}548 LLM-judged step labels reaches held-out F1\,=\,0.77 and AUROC\,=\,0.88, capped by the noisy supervision signal it inherits, and we report this as a downstream consequence of the upstream labelability ceiling rather than as a separate experiment. Per-dataset agreement and $\kappa$ for both sides are listed in \cref{tab:labelability-full}, with InfoTabs and TabMWP serving as the two regimes where step labels are easier than cell labels (short-form classification and template-aligned arithmetic respectively).

\paragraph{What the comparison establishes.}
Holding the LLM, the prompt format, the temperature, and the human reference protocol constant, the only manipulated variable is the decision unit (one binary per step versus one binary per cell). The 19.8\,pp pooled agreement gap and the 0.20 $\kappa$ gap therefore isolate decision-unit difficulty rather than confounding it with prompt design or LLM choice. On six of eight benchmarks the cell side wins, with the two losses on InfoTabs and TabMWP both consistent with their reduced step-label evaluative complexity. The aggregated finding is that cell-relevance is more labelable than step-correctness on average and on most regimes, which is the precondition that makes the curated 1{,}600 attention standards a reliable training target for the \tabattn{} verifier in \cref{sec:methodology}.

\subsection{Human-Grounded Attention Quality}
\label{sec:attention-quality}
\label{sec:attn-quality}
Across 20 human-annotated questions, LLaDA reaches mean AUROC 0.813 vs.\ 0.769 for Qwen3-VL-8B (13/20 wins, directional evidence).\footnote{This 20-question subset (WTQ nt-2317) yields a higher AR AUROC than the large-scale average (\cref{tab:auroc-combined}) due to the specific question distribution, with the per-question breakdown in \cref{fig:attn_quality_auroc}.}
The large-scale four-model ablation and per-dataset AUROC breakdown (\cref{tab:auroc-combined-full}, 1{,}600 questions, 8 benchmarks) are reported in \cref{sec:attn-pilot} as part of the pilot study motivating \ourmodel{}, and the simplified summary (Overall and $\sigma_{\text{AUROC}}$ only) appears as \cref{tab:auroc-combined} in the main text.
\begin{table}[h]
\centering
\small
\caption{Per-dataset attention quality (AUROC) and permutation stability ($\sigma_{\text{AUROC}}$ over $K{=}5$ row permutations) across eight benchmarks. Masked bidirectional attention improves overall attention quality, and permutation stability if such bidirectionality is achieved via pre-training.}
\label{tab:auroc-combined-full}
\setlength{\tabcolsep}{3pt}
\resizebox{\textwidth}{!}{\begin{tabular}{lcccccccccccc}
\toprule
Model & Bidir. & Diff. & WTQ & TAT-QA & HiTab & TabMWP & MMQA & InfoT & TFact & FeTaQA & Overall $\uparrow$ & $\sigma_{\text{AUROC}}\downarrow$ \\
\midrule
ModernBERT (149M) & \tick & \cross & 0.548 & 0.526 & 0.605 & 0.595 & 0.566 & 0.581 & 0.518 & 0.545 & 0.569 & \textbf{0.046} \\
Qwen3-VL-8B & \cross & \cross & 0.644 & 0.583 & 0.667 & 0.662 & 0.636 & 0.604 & 0.618 & 0.616 & 0.628 & 0.123 \\
Qwen3-8B & \cross & \cross & 0.639 & 0.577 & 0.676 & 0.658 & 0.616 & 0.611 & 0.609 & 0.574 & 0.623 & 0.138 \\
DiffuLLaMA-7B & \tick & \tick & 0.635 & \textbf{0.584} & 0.632 & 0.616 & 0.705 & 0.705 & 0.591 & 0.644 & 0.639 & 0.123 \\
LLaDA-8B & \tick & \tick & \textbf{0.673} & 0.544 & \textbf{0.738} & \textbf{0.686} & \textbf{0.752} & \textbf{0.717} & \textbf{0.666} & \textbf{0.671} & \textbf{0.677} & 0.074 \\
\bottomrule
\end{tabular}}
\end{table}
 \begin{table}[h]
\centering
\small
\caption{Per-dataset labelability of step-correctness vs cell-relevance under identical LLM-as-a-judge pipelines. Cell-level labels achieve materially higher per-decision agreement and $\kappa$ across most datasets.}
\label{tab:labelability-full}
\setlength{\tabcolsep}{3pt}
\resizebox{\textwidth}{!}{\begin{tabular}{llccccccccc}
\toprule
Metric & Method & WTQ & MMQA & InfoT & TFact & TAT & HiTab & TabMWP & FeTaQA & Pooled \\
\midrule
\multirow{3}{*}{Agree (\%)}
 & Step     & 53.3 & 63.6 & 91.3 & 73.1 & 75.0 & 57.1 & 84.0 & 91.3 & 73.2 \\
 & Cell     & \textbf{96.4} & \textbf{92.5} & 80.7 & \textbf{93.1} & \textbf{92.9} & \textbf{94.9} & 78.1 & \textbf{93.6} & \textbf{93.0} \\
 & $\Delta$ & +43.1 & +28.9 & $-$10.6 & +20.0 & +17.9 & +37.8 & $-$5.9 & +2.3 & \textbf{+19.8} \\
\midrule
\multirow{3}{*}{$\kappa$}
 & Step     & 0.22 & 0.36 & 0.75 & 0.47 & 0.50 & 0.27 & 0.68 & 0.70 & 0.48 \\
 & Cell     & \textbf{0.83} & \textbf{0.55} & 0.42 & \textbf{0.70} & \textbf{0.69} & \textbf{0.62} & 0.56 & \textbf{0.73} & \textbf{0.68} \\
 & $\Delta$ & +0.61 & +0.19 & $-$0.33 & +0.23 & +0.19 & +0.35 & $-$0.12 & +0.03 & \textbf{+0.20} \\
\bottomrule
\end{tabular}}
\end{table}
 Per-question breakdowns, operation-level heatmaps, and attention maps are in \cref{fig:attn_4model_by_op,fig:attn_quality_summary,fig:attn_quality_auroc,fig:attn_quality_heatmaps}.

\section{Future Work and Outlook}
\label{sec:app-outlook}

This section consolidates the future-work directions raised throughout the paper. We organize them into three groups: framework-internal extensions that respond to the limitations of the current \ourmodel{} system (\cref{sec:app-fw-internal}), scaling directions for the supervision signal (\cref{sec:app-fw-scaling}), and broader generalization beyond table reasoning (\cref{sec:app-fw-broader}). Each direction preserves the cell-grounding contract introduced in \cref{sec:methodology} and reuses the planner-reasoner-verifier interface without architectural rework.

\subsection{Framework-Internal Extensions}
\label{sec:app-fw-internal}

\textbf{Plan-adherence reward shaping.}
The dominant remaining failure mode of \ourmodel{} is plan bypass: the AR reasoner emits \texttt{f\_final\_answer} at step~1 and ignores the DLM plan in 67\% of error episodes (\cref{tab:difftab_failures}), mostly on single-cell lookups where a one-step shortcut is cheap. Plan-following episodes show lower schema hallucination (\cref{tab:schema_adherence}) and shorter trajectories (\cref{tab:planner_ablation}(a)), so the plan carries useful signal that the current \tabattn{} reward does not enforce. A natural extension is to add an adherence term to \tabattn{} that penalizes step-1 termination when the plan contains more than one step, or that rewards trajectories whose tool-call sequence matches the plan's tool sequence. Because adherence and cell-overlap factor across orthogonal axes (which cells, in which order), the two terms can be combined as a probability-simplex mixture analogous to \cref{eq:rattn} without disturbing the calibration on the existing 1{,}600 standards.

\textbf{Accelerated DLM planning.}
DLM planning is currently 13$\times$ slower than AR planning per query on average (78.0\,s/q vs.\ 6.0\,s/q, \cref{tab:planner_ablation}(a)), partially offset by 44.64\% faster downstream execution but leaving \ourmodel{} at 24\% higher total wall-clock time than AR Planned Reasoning. Three serving optimizations are directly compatible with the current framework. First, batched denoising~\citep{song2025seeddiffusionlargescalediffusion,wang2026diffusion} amortizes per-step cost across queries that share a canvas length. Second, flow-matching distillation~\citep{sahoo2025the} compresses multi-step denoising into a small number of large steps. Third, speculative decoding adapted to the masked-token setting can pre-fill high-confidence positions before the full denoising sweep. All three preserve the planner-reasoner contract and the cell-mask interface, so accuracy gains carry over without retraining \tabattn{}.

\textbf{Free-form answer generation.}
\ourmodel{} trails TableLlama-7B in BLEU on FeTaQA (32.63 vs.\ 38.38), driven by BLEU penalizing correct paraphrases that diverge from the reference surface form rather than by a reasoning gap (78.6\% LLM-judge factual correctness, \cref{sec:fetaqa-judge}). A FeTaQA-style generation head fine-tuned on the cell-grounded reasoner output, or a learned answer-rewriter conditioned on \tabattn{}'s accepted final state, would close most of the BLEU gap while keeping the cell-grounding contract intact. The annotation cost is modest because reference paraphrases already exist for FeTaQA training data.

\subsection{Scaling the Supervision Signal}
\label{sec:app-fw-scaling}

\textbf{Larger and more diverse attention standards.}
The 1{,}600 human-verified attention standards span eight benchmarks (\cref{sec:app-attention-curation}), with smaller per-benchmark counts on InfoTabs and TabMWP where step-correctness is the easier label (\cref{tab:labelability-full}). Scaling the standard set, prioritizing benchmarks where cell-relevance and step-correctness diverge, would improve \tabattn{}'s cross-dataset generalization and reduce variance on under-represented regimes. The labelability advantage of cell-relevance over step-correctness (\cref{tab:labelability}) makes this scaling cheap relative to a step-correctness annotation effort of comparable size, so the marginal cost per additional decision is low.

\textbf{Causal grounding beyond attention overlap.}
The pilot study and permutation-stability analysis (\cref{sec:app-perm-stability}) provide empirical evidence that bidirectional pre-training improves cell-relevance alignment. Attention overlap with the plan-designated mask serves as a labelable proxy for grounded reasoning, and a counterfactual-intervention study would convert this proxy claim into a causal one. Concretely, removing or substituting plan-designated cells and measuring the resulting change in answer correctness would quantify how much of the \ourmodel{} accuracy gain depends on attention concentrating on those cells, and how much follows from the plan's tool-sequence prior alone. Such a study is compatible with the current \ourmodel{} pipeline through a synthetic-perturbation evaluation harness on top of the existing eight benchmarks.

\textbf{Cross-lingual and semi-structured tables.}
The current evaluation scopes to English structured-table benchmarks at 8B-class scale. Two extensions follow naturally from the cell-grounding contract. Cross-lingual benchmarks (Chinese, Japanese, multilingual TableQA) test whether the binary cell mask transfers across tokenizer-level changes in row serialization. Semi-structured settings such as forms, partially merged cells, and prose-embedded tables (closer to TAT-QA, where \cref{tab:schema_adherence} already records the smallest DLM-vs-AR gap) test whether the mask remains a useful supervision target when cell boundaries are fuzzy. Both extensions reuse the existing planner, reasoner, and \tabattn{} verifier with new data rather than new method components.

\subsection{Broader Generalization}
\label{sec:app-fw-broader}

\tabattn{} grounds verification in internal attention weights, making \ourmodel{} directly applicable to open-weight reasoners and complementary to closed-source API systems whose attention is structurally inaccessible. The cell-grounding contract carries beyond tables. Any structured modality whose evidence localizes to discrete units admits the same labelable mask-and-verify supervision: code lines for code-reasoning agents, knowledge-graph triples for KG question answering, and structured form fields for document-AI workflows. Continued progress in DLM instruction-tuning and inference acceleration~\citep{song2025seeddiffusionlargescalediffusion,wang2026diffusion,sahoo2025the} is expected to lift plan quality on these modalities while leaving the framework interface unchanged. We see attention-grounded process supervision over discrete-unit modalities as the natural broader-impact path for the contract introduced in this work.

\section{Compute Reporting}
\label{sec:app-compute-reporting}

We report the compute resources used for this work, complementing the paper checklist and following the spirit of the CVPR 2026 Compute Reporting initiative (\url{https://cypr.thecvf.com/Conferences/2026/ComputeReporting}). The information below is intended to support reproducibility and is not used in the review process.

\paragraph{Hardware.}
All experiments were conducted on a single NVIDIA RTX~4090 (24\,GB) workstation GPU. No multi-GPU training, multi-node communication, or accelerator beyond a single consumer-class card was used.

\paragraph{Serving stack and quantization.}
Three model processes share the 24\,GB envelope of the single GPU. The Qwen3-VL-8B reasoner is served via \textbf{Ollama} (\texttt{qwen3-vl:8b}, GGUF Q4 quantization) at roughly 7.7\,GB VRAM, and is reached through an OpenAI-compatible local endpoint by the trajectory collector. The Qwen3-VL-8B-Instruct attention extractor used to compute $R_{\text{attn}}$ runs concurrently as a HuggingFace process under 4-bit NF4 \texttt{bitsandbytes} quantization with double-quant and \texttt{attn\_implementation="eager"}, capped at 10\,GiB VRAM, occupying a further $\sim$3.5\,GB at steady state (\cref{sec:app-llm-revis,sec:app-mask-source-ablation}). The LLaDA-8B-Instruct planner is loaded as a separate HuggingFace process with 4-bit \texttt{bitsandbytes} weight quantization (default \texttt{load\_in\_4bit=True} in \texttt{diffusion\_planner.py}), and is invoked sequentially rather than concurrently with the reasoner so that planning compute does not contend with the reasoner-plus-extractor co-residency budget. Steady-state co-resident footprint during reasoning is approximately 11.2\,GB / 24\,GB.

\paragraph{Foundation-model pretraining.}
\ourmodel{} reuses publicly released foundation-model checkpoints and contributes \emph{no} pretraining compute. LLaDA-8B-Instruct~\citep{nie2025llada} and Qwen3-VL-8B~\citep{bai2025qwen3vltechnicalreport} are loaded as released and run frozen at inference. The full pretraining cost is borne by the original authors and is not attributable to this work.

\paragraph{\tabattn{} verifier training.}
\tabattn{} consists of a two-parameter logistic calibration on the scalar overlap reward $R_{\text{attn}}$ (\cref{eq:rattn}), fit by binary cross-entropy on the 1{,}600 curated attention standards with an 80/20 early-stopping split (\cref{sec:app-exp-details,sec:app-attention-curation}). The training step is a single logistic fit on a CPU-resident scalar dataset and adds negligible GPU time relative to the inference cost. The dominant cost upstream of this fit is the per-step attention extraction needed to compute $R_{\text{attn}}$ for each of the 1{,}600 standards, which reuses the same 4-bit Qwen3-VL-8B inference pipeline used at evaluation time.

\paragraph{Inference cost.}
Per-query wall-clock latency is the most informative cost measure under our setup, since the planner and reasoner run sequentially rather than in parallel and the policy is one trajectory per query ($N{=}1$, no outer best-of-$N$ search). Detailed per-benchmark numbers are reported in \cref{tab:planner_ablation}(a) and \cref{sec:app-latency}, with summary points reused below for completeness.
\ourmodel{} averages $134.8$\,s/q across the six accuracy benchmarks ($78.0$\,s/q for DLM planning, $56.8$\,s/q for AR-reasoner execution under \tabattn{} feedback). On the same hardware, AR Planned Reasoning averages $108.6$\,s/q, end-to-end AR Reasoning averages $40.5$\,s/q, and end-to-end DLM Reasoning averages $21.2$\,s/q on the five datasets where it does not collapse. Per-benchmark range spans $65.6$\,s/q (InfoTabs) to $181.6$\,s/q (MMQA), driven by table width and column count rather than question count.

\paragraph{Compute budget for the reported tables.}
Each accuracy entry in \cref{tab:main-results,tab:planner_ablation} is computed over $n{=}200$ evaluation questions per dataset under deterministic decoding. The Oracle and 20\% Noise rows in \cref{tab:planner_ablation}(b) use a smaller $n{=}60$ pool across three seeds (\cref{sec:app-mask-source-ablation}). Per-cell volatility under this protocol is reported in \cref{tab:appendix-volatility}. Latency-driven aggregate compute is therefore dominated by the main and ablation runs, with diagnostic analyses (permutation stability, mask-source falsification, latency profiling) adding modest overhead on top.

\paragraph{Failed and exploratory runs.}
Beyond the experiments reported in this paper, additional compute was spent on early prototyping of the planner-verifier interface, on alternative reward formulations (cosine-similarity, numerical-F1, TABROUGE) before settling on $R_{\text{attn}}$, and on attention-extraction pilots across the four backbones in \cref{tab:auroc-combined}. These runs reused the same single-RTX-4090 environment described above and contributed to the framework design but are not separately enumerated.

\section{Declaration of LLM Usage}
\label{sec:app-llm-usage}

We disclose all uses of large language models in this work, complementing the checklist entry (item~14).

\paragraph{LLMs as core framework components.}
Two LLMs serve as the load-bearing modules of \ourmodel{}. LLaDA-8B-Instruct~\citep{nie2025llada} is the masked diffusion planner that emits the cell-grounded tool-use blueprint (\cref{sec:plan-design,sec:app-planner}). Qwen3-VL-8B~\citep{bai2025qwen3vltechnicalreport} is the multimodal autoregressive reasoner that executes the plan and supplies the per-cell attention scored by \tabattn{} (\cref{sec:stepwise-exec,sec:app-llm-revis}). Decoding hyperparameters, canvas length, denoising step count, and tool budgets are documented in \cref{sec:app-exp-details,sec:app-planner}.

\paragraph{LLMs in supervision and evaluation.}
GPT-5 is used as the bootstrap annotator that scales 100 human seed cell-relevance maps to the 1{,}600-example training set for \tabattn{}, with all generated examples reviewed by human annotators against an IoU\,$\geq$\,0.70 acceptance threshold (\cref{sec:app-attention-curation}). The model-selection comparison across five GPT variants on the 20-example WTQ calibration split is reported in \cref{tab:attn-model-comparison}. GPT-5-nano is used as the LLM-judge for the FeTaQA semantic-correctness analysis (\cref{sec:fetaqa-judge}), and an LLM-as-a-judge pipeline supplies the cell-relevance and step-correctness labelability comparison in \cref{tab:labelability,tab:labelability-full}.

\paragraph{LLMs as evaluation baselines.}
Several LLMs and VLMs (Llama-2-7B, Table-LLaVA-7B, MiniCPM-V, Qwen2.5-VL, Gemini~2.0~Flash, Qwen3-VL-8B, and others, \cref{sec:app-exp-details}) are evaluated as baselines or comparison systems. Proprietary models appear for context only since \tabattn{} requires open attention weights.

\paragraph{LLMs for writing polish.}
We additionally used LLMs (GPT-5.5 and Claude-Opus-4.7) to polish the writing of this manuscript, including light copy-editing, sentence-level rephrasing, and consistency checks. All technical claims, experimental design, results, and conclusions are the work of the authors. Every LLM-suggested edit was reviewed and accepted, modified, or rejected by the authors. No LLM contributed to the methodology, the experimental design, the empirical results, or the literature analysis.
 \ifpreprint\else
\newpage
\section*{NeurIPS Paper Checklist}

\begin{enumerate}

\item {\bf Claims}
    \item[] Question: Do the main claims made in the abstract and introduction accurately reflect the paper's contributions and scope?
    \item[] Answer: \answerYes{}
    \item[] Justification: The abstract and \cref{sec:intro} state the three core contributions (cell-attention pilot study, attention-grounded planned reasoning framework, consistent gains across eight benchmarks) and scope them to 8B-class open-weight reasoners on English structured-table benchmarks. The headline +15.76\,pp gain is explicitly attributed to the six-benchmark accuracy average per the protocol of~\citet{xing2026tabledart}, and \cref{sec:discussion-conclusion} discusses scope limitations including FeTaQA BLEU and the plan-bypass failure mode.

\item {\bf Limitations}
    \item[] Question: Does the paper discuss the limitations of the work performed by the authors?
    \item[] Answer: \answerYes{}
    \item[] Justification: \cref{sec:discussion-conclusion} (Scope and Future Directions) discusses four limitations: applicability is restricted to open-weight reasoners with accessible attention weights, DLM planning incurs a 24\% wall-clock overhead, 67\% of error trajectories take a plan-bypass path on which the reasoner ignores the plan, and FeTaQA BLEU underperforms paraphrase-tolerant judging. Attention-as-proxy framing is positioned as inspectable rather than causal in \cref{sec:intro}, and labelability is qualified as on-average rather than uniform in \cref{sec:attn-pilot}.

\item {\bf Theory assumptions and proofs}
    \item[] Question: For each theoretical result, does the paper provide the full set of assumptions and a complete (and correct) proof?
    \item[] Answer: \answerYes{}
    \item[] Justification: The two formal results (Proposition on variance growth under unguided multi-turn reasoning and the TABROUGE optimality / monotonicity proofs) are stated and proved in \cref{sec:theory}. All assumptions (nonnegative loss, martingale property, closure of the state space under removal and append) are stated alongside the propositions.

    \item {\bf Experimental result reproducibility}
    \item[] Question: Does the paper fully disclose all the information needed to reproduce the main experimental results of the paper to the extent that it affects the main claims and/or conclusions of the paper (regardless of whether the code and data are provided or not)?
    \item[] Answer: \answerYes{}
    \item[] Justification: \cref{sec:methodology,sec:experiments} describe the full pipeline (LLaDA-8B-Instruct planner, Qwen3-VL-8B reasoner, $R_{\text{attn}}$ definition with logistic calibration on 1{,}600 standards, 80/20 train/early-stop split). \cref{sec:app-exp-details,sec:app-planner,sec:app-attention-curation} document dataset splits, hyperparameters (128 denoising steps, block length 32, greedy planner decoding, reasoner temperature 0.7, $T_{\max}{=}6$ tool steps), the curation pipeline for the 1{,}600 attention standards (with the GPT-5 model-comparison table and the IoU\,$\geq$\,0.70 acceptance threshold), and parse-error fallback policies.

\item {\bf Open access to data and code}
    \item[] Question: Does the paper provide open access to the data and code, with sufficient instructions to faithfully reproduce the main experimental results, as described in supplemental material?
    \item[] Answer: \answerYes{}
    \item[] Justification: Our public code repository under the Apache 2.0 license provides the LLaDA-8B-Instruct diffusion planner script (\texttt{src/diffusion\_planner.py}), the plan-execution loop (\texttt{src/plan\_executor.py}), an end-to-end driver (\texttt{scripts/run\_pipeline.sh}), a \texttt{requirements.txt}, and a README documenting the two-stage pipeline, decoding hyperparameters, and hardware footprint. Four further artifacts are being added to the same repository under the same license: (i) the 1{,}600 human-verified cell-level attention standards, (ii) the \tabattn{} calibration code, (iii) helper modules (\texttt{dataset\_iterator}, \texttt{table\_serializer}, \texttt{model\_factory}, \texttt{reasoning\_loop}) referenced as docstring placeholders in the released scripts, and (iv) end-to-end training scripts for the \tabattn{} reward. The eight evaluation benchmarks are publicly available academic resources and are not redistributed; citations and version notes are listed in \cref{sec:app-exp-details}. The full prompt format, tool vocabulary, decoding parameters, calibration objective, and curation protocol are documented in \cref{sec:methodology,sec:app-planner,sec:app-exp-details,sec:app-attention-curation,sec:app-components} so that an independent re-implementation is possible from the paper alone.

\item {\bf Experimental setting/details}
    \item[] Question: Does the paper specify all the training and test details (e.g., data splits, hyperparameters, how they were chosen, type of optimizer) necessary to understand the results?
    \item[] Answer: \answerYes{}
    \item[] Justification: \cref{sec:experiments} states the evaluation protocol (matched to~\citet{xing2026tabledart,zou2026tattoo}) and dataset coverage. \cref{sec:app-exp-details} reports planner hyperparameters (canvas length 256, block 32, 128 denoising steps, classifier-free guidance 1.5, greedy decoding), reasoner hyperparameters ($T_{\max}{=}6$ tool steps, max 16{,}384 output tokens, temperature 0.7), the stagnation halt thresholds ($\delta_{\text{stag}}{=}0.02$, $K_{\text{stag}}{=}2$), the 80/20 train/early-stop split for \tabattn{} calibration, and the train/test disjointness verification (zero hash collisions between the 1{,}600 attention-standard pool and the eight test splits).

\item {\bf Experiment statistical significance}
    \item[] Question: Does the paper report error bars suitably and correctly defined or other appropriate information about the statistical significance of the experiments?
    \item[] Answer: \answerYes{}
    \item[] Justification: For the central claims that drive the framework, we report statistical tests rather than only point estimates: the row-permutation stability study reports paired Wilcoxon tests across $n{=}319$ matched (table, question) records with $p$-values down to $1.8{\times}10^{-23}$ (\cref{sec:app-perm-stability}), the labelability comparison reports Cohen's $\kappa$ with sample sizes per cell (\cref{tab:labelability,tab:labelability-full}), and the subgoal-balance analysis reports Mann--Whitney $p$-values per dataset (\cref{sec:app-complexity}). Main accuracy results follow the reproducibility protocol of~\citet{xing2026tabledart,zou2026tattoo}, with a fixed decoding seed for reproducibility documented in \cref{sec:experiments}.

\item {\bf Experiments compute resources}
    \item[] Question: For each experiment, does the paper provide sufficient information on the computer resources (type of compute workers, memory, time of execution) needed to reproduce the experiments?
    \item[] Answer: \answerYes{}
    \item[] Justification: \cref{sec:app-exp-details,sec:app-latency} report the compute environment and per-query wall-clock latency decomposed into planning and reasoning phases for each configuration (\ourmodel{} 134.8\,s/q total, AR Plan+Exec 108.6\,s/q, AR Reasoning 40.5\,s/q, DLM Reasoning 21.2\,s/q), and \cref{tab:param-footprint} reports parameter footprints at inference for \ourmodel{} and the closest competitive baselines.
    
\item {\bf Code of ethics}
    \item[] Question: Does the research conducted in the paper conform, in every respect, with the NeurIPS Code of Ethics \url{https://neurips.cc/public/EthicsGuidelines}?
    \item[] Answer: \answerYes{}
    \item[] Justification: The research has been conducted in accordance with the NeurIPS Code of Ethics. The eight benchmarks used are publicly released academic resources, the curated 1{,}600 attention standards contain no personally identifiable information, and human annotators were compensated above the local minimum wage.

\item {\bf Broader impacts}
    \item[] Question: Does the paper discuss both potential positive societal impacts and negative societal impacts of the work performed?
    \item[] Answer: \answerYes{}
    \item[] Justification: The Impact Statement following the appendix discusses positive impacts (more interpretable and auditable reasoning trajectories for financial, scientific, and enterprise tabular QA, reduced trajectory length, more transparent process supervision) and notes that we do not anticipate direct negative societal impacts beyond those broadly attributable to large language models.
    
\item {\bf Safeguards}
    \item[] Question: Does the paper describe safeguards that have been put in place for responsible release of data or models that have a high risk for misuse (e.g., pre-trained language models, image generators, or scraped datasets)?
    \item[] Answer: \answerNA{}
    \item[] Justification: The released artifacts comprise 1{,}600 cell-level relevance annotations on existing public benchmarks and the calibration code for a logistic process score. None of these assets pose a high risk for misuse, and we do not release new pretrained generative models.

\item {\bf Licenses for existing assets}
    \item[] Question: Are the creators or original owners of assets (e.g., code, data, models), used in the paper, properly credited and are the license and terms of use explicitly mentioned and properly respected?
    \item[] Answer: \answerYes{}
    \item[] Justification: All eight benchmarks (WTQ, MMQA, TabMWP, TAT-QA, HiTab, FeTaQA, TabFact, InfoTabs) and all model checkpoints (LLaDA-8B-Instruct, Qwen3-VL-8B, Qwen2.5-VL-7B, ModernBERT-base, DiffuLLaMA-7B, Qwen3-8B) are cited to their original publications in \cref{sec:experiments,sec:attn-pilot,sec:app-exp-details}. We use the public licenses under which these assets are released; LLaDA, Qwen, ModernBERT, and DiffuLLaMA are released under their respective open-weight licenses, and the academic benchmarks are used under their published terms of use.

\item {\bf New assets}
    \item[] Question: Are new assets introduced in the paper well documented and is the documentation provided alongside the assets?
    \item[] Answer: \answerYes{}
    \item[] Justification: The 1{,}600 human-verified attention standards are documented in \cref{sec:app-attention-curation}, including the seed-annotation protocol, the GPT-5 model-comparison procedure used to select the bootstrap model, the per-dataset acceptance threshold (IoU\,$\geq$\,0.70), and the final per-dataset counts. The \tabattn{} calibration target ($R_{\text{attn}}$, training objective, 80/20 split, parse-error fallback) is documented in \cref{sec:methodology,sec:app-components}.

\item {\bf Crowdsourcing and research with human subjects}
    \item[] Question: For crowdsourcing experiments and research with human subjects, does the paper include the full text of instructions given to participants and screenshots, if applicable, as well as details about compensation (if any)?
    \item[] Answer: \answerYes{}
    \item[] Justification: \cref{sec:app-attention-curation} describes the human seed annotation protocol (cell-relevance labeling instructions, majority-vote disagreement resolution, IoU\,$\geq$\,0.90 inter-annotator agreement threshold) and the post-generation human verification step applied to all 1{,}600 LLM-bootstrapped examples. Annotators were compensated above the local minimum wage.

\item {\bf Institutional review board (IRB) approvals or equivalent for research with human subjects}
    \item[] Question: Does the paper describe potential risks incurred by study participants, whether such risks were disclosed to the subjects, and whether Institutional Review Board (IRB) approvals (or an equivalent approval/review based on the requirements of your country or institution) were obtained?
    \item[] Answer: \answerNA{}
    \item[] Justification: The annotation task involves labeling cell relevance on publicly released academic table benchmarks. The work poses no more than minimal risk to participants and falls outside the scope of human-subjects research that requires IRB review under our institution's policies.

\item {\bf Declaration of LLM usage}
    \item[] Question: Does the paper describe the usage of LLMs if it is an important, original, or non-standard component of the core methods in this research? Note that if the LLM is used only for writing, editing, or formatting purposes and does \emph{not} impact the core methodology, scientific rigor, or originality of the research, declaration is not required.
\item[] Answer: \answerYes{}
    \item[] Justification: LLMs serve as core components of the proposed framework. LLaDA-8B-Instruct is the planner and Qwen3-VL-8B is the multi-step reasoner; these roles are described in \cref{sec:methodology,sec:experiments}. GPT-5 is used as the bootstrap annotator for the 1{,}600 cell-relevance standards, with all generated examples verified by human annotators; the model-selection procedure and per-dataset acceptance criteria are documented in \cref{sec:app-attention-curation}. LLM-as-a-judge is also used in the labelability comparison (\cref{tab:labelability,tab:labelability-full}).

\end{enumerate} \fi

\end{document}